\documentclass{article}

\usepackage{arxiv}

\usepackage[utf8]{inputenc}
\usepackage[T1]{fontenc}
\usepackage{url}
\usepackage{booktabs}
\usepackage{amsfonts}
\usepackage{amsmath}
\usepackage{amssymb}
\usepackage{amsthm}
\usepackage{mathtools}
\usepackage{nicefrac}
\usepackage{microtype}
\usepackage{graphicx}
\usepackage{natbib}
\usepackage{algorithm}
\usepackage{algpseudocode}

\usepackage{array}
\usepackage{xcolor}
\usepackage{enumitem}
\usepackage{tikz}
\usepackage{hyperref}
\usepackage{doi}
\usepackage[nameinlink,noabbrev]{cleveref}

\title{Acceleration and its Limits \\
in Gaussian Variational Inference}

\author{%
Jaesung R. Park\textsuperscript{1}\thanks{Equal contribution.} \qquad
Uijeong Jang\textsuperscript{1}\footnotemark[1] \qquad
Yifan Chen\textsuperscript{1}\\
\textsuperscript{1}Department of Mathematics\\
University of California, Los Angeles\\
\texttt{\{ryanpark7,uijeongjang,yifanchen\}@math.ucla.edu}
}

\date{}

\providecommand{\shorttitle}{}
\renewcommand{\shorttitle}{Acceleration and its Limits in Gaussian Variational Inference}

\hypersetup{
    pdftitle={Acceleration and its Limits in Gaussian Variational Inference},
    pdfsubject={Gaussian variational inference and accelerated optimization},
    pdfauthor={Jaesung R. Park, Uijeong Jang, Yifan Chen},
    pdfkeywords={Gaussian variational inference, acceleration, query complexity, variance reduction, log-concavity}
}

\newtheorem{theorem}{Theorem}[section]
\newtheorem{proposition}[theorem]{Proposition}
\newtheorem{lemma}[theorem]{Lemma}
\newtheorem{corollary}[theorem]{Corollary}
\newtheorem{assumption}[theorem]{Assumption}
\newtheorem{remark}[theorem]{Remark}

\definecolor{darkred}{rgb}{.6,0,0}
\definecolor{darkblue}{rgb}{0,0,.7}
\definecolor{darkgreen}{rgb}{0,.7,0}
\definecolor{darkbrown}{rgb}{0.8,0.4,0.4}

\newcommand{\prox}{\operatorname{prox}}

\newcommand{\proj}{\Pi}

\begin{document}

\maketitle

\begin{abstract}
We study whether the classical momentum-based acceleration that improves the condition-number dependence of gradient descent from $\kappa$ to $\sqrt{\kappa}$ for strongly convex, smooth objectives extends to Gaussian variational inference (GVI) for strongly log-concave, log-smooth target distributions. With exact Gaussian expectations, a convex composite formulation of GVI yields the classical accelerated $\sqrt{\kappa}$ convergence rate. However, when only pointwise derivatives are available, we construct a family of strongly log-concave targets for which any randomized algorithm requires $\kappa^{1-o(1)}$ queries to achieve a fixed objective accuracy, even when the initial objective gap is bounded independently of $\kappa$. Thus curvature bounds alone cannot guarantee uniform accelerated total-query complexity. We then ask which additional structure permits acceleration and give positive examples: under a Hessian-Lipschitz assumption, variance-reduced stochastic methods yield accelerated guarantees when the Hessian variation is sufficiently small. Finally, experiments on real datasets show substantial empirical speedups from acceleration across a broad range of methods. 

\end{abstract}


\section{Introduction}\label{sec:intro}

Variational inference (VI) approaches probabilistic inference through optimization: it fits a tractable probability law to a target distribution, providing an approximation that is widely used in Bayesian inference~\citep{blei2017variational}. This viewpoint naturally motivates the use of momentum-based acceleration to develop faster inference algorithms. For an $\alpha$-strongly convex, $\beta$-smooth objective, momentum improves gradient descent's complexity from $O(\kappa\log(1/\varepsilon))$ to $O(\sqrt\kappa\log(1/\varepsilon))$, where $\kappa=\beta/\alpha$ \citep{nesterov2004introductory}. With stochastic gradients, however, acceleration must also account for a variance-dependent cost, which can dominate the gains from momentum.

The parallel question has driven substantial progress in Markov chain Monte Carlo (MCMC), the other pillar of probabilistic inference alongside VI. Its development has also benefited from an optimization viewpoint: the law of overdamped Langevin diffusion is a Wasserstein gradient flow of KL divergence \citep{jordan1998variational}, while underdamped Langevin connects sampling with momentum methods~\citep{cheng2018underdamped,ma2021analog}; see \citet{chewi2026sampling} for an overview. Acceleration in sampling is delicate and remains an active area of research.
Recent developments begin to achieve sublinear condition-number dependence in algorithm complexity for strongly log-concave targets \citep{altschuler2025ballistic,chen2026exact,lu2026windowed}, while a non-acceleration result holds for standard splitting schemes \citep{bourabee2026nonacceleration}.

In VI, momentum variants of black-box methods have been related to Langevin dynamics \citep{hoffman2020blackbox}, and accelerated gradient methods have been analyzed for mean-field VI with exact expectations \citep{jiang2024polyhedral}. This paper takes a first step toward the theoretical foundations of momentum-based stochastic acceleration in VI, focusing on full-covariance Gaussian variational inference (GVI). 

For strongly log-concave, log-smooth targets, GVI admits a convex composite formulation whose expected potential inherits the target's curvature bounds \citep{challis2013gaussian,domke2020provable}. Exact Gaussian integration therefore makes classical acceleration available. With pointwise target access, however, the required expectations must be estimated, bringing the variance-dependent cost into play. For unbiased stochastic gradients with variance bounded by $\sigma^2$, the classical accelerated complexity is
\begin{equation}\label{eq:intro-stochastic-benchmark}
O\!\left(\sqrt\kappa\log\frac{2\Delta_0}{\varepsilon}+\frac{\sigma^2}{\alpha\varepsilon}\right)
\end{equation}
\citep{ghadimi2013optimal}, where $\Delta_0$ bounds the initial objective gap. 
In GVI, the variance bound depends on the target, the variational parameters, and the estimator. Establishing acceleration in the total number of target queries requires accounting for this dependence. This distinction motivates our study of both the possibilities and limitations of acceleration in GVI.

Our contributions are as follows:

\begin{itemize}[leftmargin=2em,itemsep=1pt,topsep=3pt]
\item \textbf{An accelerated benchmark.} With exact Gaussian expectations, the convex composite formulation of GVI gives the classical $O(\sqrt\kappa\log(\Delta_0/\varepsilon))$ rate using square-root parameterizations of the covariance matrix (Theorem~\ref{thm:exact-apg}).  

\item \textbf{A query lower bound.} We construct a family of strongly log-concave, log-smooth targets in dimension  $\Theta((\log\kappa)^2)$, with uniformly bounded initial gaps, for which any randomized  algorithm using derivatives of fixed finite order requires $\kappa^{1-o(1)}$ queries to attain a fixed expected accuracy uniformly over the family (Theorem~\ref{thm:lb-main}). The lower bound therefore extends beyond any particular estimator or momentum scheme.

\item \textbf{Acceleration with controlled Hessian variation.} 
Under a Hessian-Lipschitz assumption, we derive variance bounds for stochastic gradient estimators and use them to establish convergence guarantees for accelerated proximal and projected methods. All retain the accelerated $O(\sqrt\kappa\log(2\Delta_0/\varepsilon))$ iteration count, and their total query complexity remains accelerated when the Hessian variation is sufficiently small (Section~\ref{sec:stochastic}).
\end{itemize}

Numerical experiments illustrate the predicted condition-number scaling on Gaussian targets and the practical benefits of acceleration on graph-probit posteriors. Compared with recent progress on accelerated sampling, our results suggest that stochastic acceleration in VI requires assumptions beyond log-concavity and effective variance reduction.

\subsection{Related work}\label{sec:related}
In statistical inference, means and covariances are often the most used statistics, making Gaussian approximations and GVI useful. GVI also provides a tractable setting for studying the convergence and complexity of optimization algorithms for variational inference. It admits a convex formulation in mean and covariance-factor parameters for log-concave targets \citep{challis2013gaussian,domke2020provable}, and its
expected-potential term inherits smoothness from the potential $V$.
These properties underpin convergence analyses of projected and proximal stochastic gradient methods \citep{domke2023provable,kim2023convergence}.

Geometry has played an important role in the study of GVI. 
Under strong log-concavity and log-smoothness, geodesic strong convexity of the GVI objective in Bures--Wasserstein geometry yields algorithms with convergence guarantees that are linear or quadratic in the condition number \citep{lambert2022variational,diao2023forward, kim2026stochastic}. Fisher--Rao geometry underlies natural gradient methods
\citep{amari1998natural}.
For GVI, natural gradient algorithms and their convergence
have been studied in both deterministic and stochastic settings
\citep{chen2026sampling,wu2024understanding,kumar2025square,
sun2025natural,guilmeau2026projected}.
Momentum acceleration on manifolds requires controlling geometric distortion when comparing quantities in different tangent spaces \citep{ahn2020riemannian}. Silver step-size schedules offer another approach to acceleration in Riemannian optimization \citep{park2025silver}, with existing guarantees based on exact gradients. In contrast to these geometry-based approaches, we study momentum acceleration in Euclidean mean and covariance-factor coordinates, considering both exact Gaussian expectations and pointwise derivative access.

For stochastic VI, quadratic control variates
\citep{miller2017reducing,geffner2020approximation} and sticking-the-landing (STL) corrections \citep{roeder2017sticking} reduce reparameterization gradient variance. Recent work has established convergence guarantees for variance-reduced GVI methods \citep{kim2024linear,luu2025stochastic}. We derive variance bounds in terms of Hessian variation and combine them with accelerated stochastic optimization
\citep{ghadimi2013optimal}, obtaining accelerated total-query guarantees when Hessian variation is sufficiently small.

Existing theoretical work on GVI has largely focused on establishing convergence guarantees for specific algorithms. To the best of our knowledge, no query-complexity lower bounds have been established for GVI under pointwise derivative access. Related lower bounds have been studied in log-concave sampling~\citep{chewi2022query,chewi2024query}, but they measure accuracy in total variation distance rather than excess reverse KL.

\subsection{Problem setup}\label{sec:setup}\label{subsec:setup-objective}

Let $\pi_V(x)=Z_V^{-1}e^{-V(x)}$ with $Z_V:=\int_{\mathbb{R}^d} e^{-V(x)}dx<\infty$ be a target density on $\mathbb{R}^d$, where $V\in C^2(\mathbb{R}^d)$ is twice continuously differentiable, and let $\mathcal{G}_d$ denote the family of Gaussian distributions on $\mathbb{R}^d$:
\[
\mathcal G_d:=\left\{\mathcal{N}(m,\Sigma):
m\in\mathbb R^d,\, \Sigma\in\mathbb S_{++}^d \right\}.
\]
For $q\in\mathcal{G}_d$, define
\[
F_V(q):=\mathbb{E}_{X\sim q}[V(X)]-H(q),\qquad H(q):=-\mathbb{E}_{X\sim q}[\log q(X)].
\]
where $H$ is the entropy of $q$. The Gaussian variational inference (GVI) problem is to minimize $F_V$ over $\mathcal G_d$. We suppress the subscript and write $F:=F_V$ and $F^\star:=\inf_{q\in\mathcal{G}_d}F(q)$. Because $F(q)=\mathrm{KL}(q\,\|\,\pi_V)-\log Z_V$, the quantity $F(q)-F^\star$ is exactly the excess reverse KL divergence relative to the best Gaussian approximation to $\pi_V$. We make the following curvature assumptions on $V$:
\begin{assumption}\label{assump:curvature}
There exist constants $0<\alpha\le\beta<\infty$ such that $\alpha I\preceq\nabla^2V(x)\preceq\beta I$ for every $x\in\mathbb{R}^d$.

\end{assumption}
This is a standard assumption for establishing convergence rates and complexity bounds in log-concave sampling and in GVI~\citep{chewi2026sampling,lambert2022variational,diao2023forward,domke2023provable}. Throughout the paper, we write $\kappa:=\beta/\alpha$. Under Assumption~\ref{assump:curvature}, $F$ has a unique minimizer $q_\star=\mathcal{N}(m_\star,\Sigma_\star)$. The first-order optimality conditions give
\[
\mathbb{E}_{q_\star}[\nabla V(X)]=0,\qquad \mathbb{E}_{q_\star}[\nabla^2V(X)]=\Sigma_\star^{-1},\qquad \beta^{-1}I\preceq\Sigma_\star\preceq\alpha^{-1}I.
\]
However, in Section~\ref{sec:lb}, we show that these assumptions alone are insufficient to obtain $O(\sqrt{\kappa})$ acceleration. The stochastic acceleration results later in the paper additionally use the following regularity condition:

\begin{assumption}\label{assump:hessian-lipschitz}
The potential $V$ has a Lipschitz continuous Hessian. That is,
\[
L_H:=\sup_{x,y\in\mathbb{R}^d,\ x\neq y}\frac{\|\nabla^2V(x)-\nabla^2V(y)\|_{\mathrm{op}}}{\|x-y\|_2}<\infty.
\]
\end{assumption}

The constant $L_H$ quantifies the nonquadraticity of $V$. 
In particular, if $L_H=0$, then $V$ is quadratic and $\pi_V$ is Gaussian.

We represent covariance $\Sigma$ either by its Cholesky factor parametrization $\Sigma=LL^\top$, where $L\in\mathbb{T}_{++}^d$ is lower triangular with strictly positive diagonal, or by its symmetric square root parametrization $\Sigma=C^2$, where $C\in\mathbb{S}_{++}^d$ is positive definite. We write $\mathcal A=\mathbb T^d$ or $\mathbb S^d$, respectively, and let $\mathcal A_{++}$ denote the corresponding positive domain. We write $A\in\mathcal A_{++}$ for either factor.

For $A\in\mathcal{A}_{++}$, let $X=m+A\xi$ with $\xi\sim\mathcal{N}(0,I_d)$. Up to an additive constant independent of $(m,A)$,
\[
F(m,A)=\mathbb{E}_{\xi}[V(m+A\xi)]-\log\det A.
\]
We use $x=(m,A)$ when the factor coordinate is clear from context and equip the product space with the norm
$\|(u,U)\|^2:=\|u\|^2+\|U\|_F^2$. The gradients with respect to $C$ are taken in $\mathbb{S}^d$ under the Frobenius inner product. For a differentiable function $\Phi$ on $\mathbb{S}_{++}^d$, $\nabla_C\Phi(C)$ denotes the unique symmetric matrix satisfying $D\Phi(C)[U]=\langle\nabla_C\Phi(C),U\rangle_F$ for every $U\in\mathbb{S}^d$. Unless otherwise specified, $\|\cdot\|$ denotes the
Euclidean norm on vectors and the Frobenius norm on matrices.

Under these assumptions, a key observation is that the curvature of $V$ is inherited by $\mathbb{E}_{\xi}[V(m+A\xi)]$, which is formalized by the following lemma \citep{challis2013gaussian,domke2020provable}.
\begin{lemma}\label{lem:curvature-inheritance}
Under Assumption~\ref{assump:curvature}, the energy function
\[
E(m,A):=\mathbb{E}_{\xi}[V(m+A\xi)]
\]
is twice continuously differentiable on $\mathbb{R}^d\times\mathcal{A}$ and, for every $\theta=(m,A)$ and every direction $(u,U)$,
\[
\alpha\|(u,U)\|^2\le\nabla^2E(\theta)[(u,U),(u,U)]\le\beta\|(u,U)\|^2.
\]
Consequently, $E$ is $\alpha$-strongly convex and $\beta$-smooth.
\end{lemma}

The proof is in Appendix~\ref{app:intro}, where we use the isotropy identity: $\mathbb{E}\|u+U\xi\|^2=\|u\|^2+\|U\|_F^2.$

\paragraph{Access model and cost.}
Unless stated otherwise, algorithms access $V$ through pointwise derivative evaluations. We count pointwise derivative evaluations and Hessian--vector products (HVPs) separately as $N_{\nabla V}$ and $N_{\mathrm{HVP}}$, forming $\nabla^2V(X)A$ costs $d$ HVPs. An output $q$ attains accuracy $\varepsilon$ if $\mathbb{E}[F(q)-F^\star]\le\varepsilon$, and $T_\varepsilon$ denotes the number of updates after which this holds. Section~\ref{sec:exact} alone uses exact Gaussian integration as an idealized benchmark.

\section{Acceleration with deterministic gradients}\label{sec:exact}

Suppose the Gaussian expectations can be evaluated exactly. One exact-expectation call at $(m,A)$ returns
\[
\nabla E(m,A)=\bigl(\mathbb{E}[\nabla V(m+A\xi)],\, P_{\mathcal{A}}\mathbb{E}[\nabla V(m+A\xi)\xi^\top]\bigr),\qquad P_{\mathbb{T}^d}=\operatorname{low}(\cdot),\quad P_{\mathbb{S}^d}=\operatorname{sym}(\cdot)
\]
where $\operatorname{sym}(\cdot)$ and $\operatorname{low}(\cdot)$ denote the symmetric and lower-triangular projections, respectively.  GVI is then a deterministic convex composite problem with the extended objective $F = E + h$, where $h(A) = - \log \det A$ on $\mathcal{A}_{++}$ and $h(A) = + \infty$ otherwise. 

By Lemma~\ref{lem:curvature-inheritance}, $E$ is $\alpha$-strongly convex and $\beta$-smooth on the entire ambient space, while $h$ is proper, closed, and convex. Moreover, for any stepsize $\eta>0$, the proximal map $\prox_{\eta h}(\cdot)$ is explicit: for $\mathcal A=\mathbb T^d$, it applies $y\mapsto \phi_\eta(y):=(y+\sqrt{y^2+4\eta})/2$ to the diagonal entries and the identity map to the strictly lower-triangular entries, while for $\mathcal A=\mathbb S^d$, it applies $\phi_\eta$ to the eigenvalues. This formula also appears in \citet{domke2020provable}. The accelerated proximal gradient (APG) method therefore applies directly as specified in Algorithm~\ref{alg:exact-apg} with its convergence guarantee:

\begin{algorithm}[htbp]
\caption{Deterministic accelerated proximal gradient for GVI}
\label{alg:exact-apg}
\begin{algorithmic}[1]
\Require $x_0=(m_0,A_0)\in\mathbb R^d\times\mathcal A_{++}$, horizon length $T$
\State $\eta\gets1/\beta$, $\rho\gets(1-\sqrt{\alpha/\beta})/(1+\sqrt{\alpha/\beta})$
\State $y_0\gets x_0$
\For{$k=0,\ldots,T-1$}
\State $g_k\gets\nabla E(y_k)$
\State $(m_{k+1},\widetilde A_{k+1})\gets y_k-\eta g_k$
\State $A_{k+1}\gets\prox_{\eta h}(\widetilde A_{k+1})$
\State $x_{k+1}\gets(m_{k+1},A_{k+1})$
\State $y_{k+1}\gets x_{k+1}+\rho(x_{k+1}-x_k)$
\EndFor
\State \Return $q_T=\mathcal N(m_T,A_TA_T^\top)$
\end{algorithmic}
\end{algorithm}

\begin{theorem}\label{thm:exact-apg}\label{thm:exact-apg-full}
Let Assumption~\ref{assump:curvature} hold, and let
$x_\star=(m_\star,A_\star)$ be the unique minimizer.
For the iterates $x_k=(m_k,A_k)$ of Algorithm~\ref{alg:exact-apg},
write $q_k:=\mathcal N(m_k,A_kA_k^\top)$ and define
\[
\Delta_0:=F(q_0)-F^\star+\frac{\alpha}{2}\|x_0-x_\star\|^2.
\]
For either factor representation
$\mathcal A\in\{\mathbb T^d,\mathbb S^d\}$, the iterates satisfy
\[
F(q_k)-F^\star\le(1-\kappa^{-1/2})^k\Delta_0
\]
for every $k\ge0$. Consequently, for every
$0<\varepsilon\le\Delta_0$, if the algorithm is run for
\[
T_\varepsilon:=
\left\lceil\sqrt\kappa\log\frac{\Delta_0}{\varepsilon}\right\rceil
\]
steps, then $F(q_{T_\varepsilon})-F^\star\le\varepsilon$.
Each iteration uses one exact energy-gradient call, so the
iteration count and exact energy-gradient oracle count both
equal $T_\varepsilon$.
\end{theorem}

The full proof is given in Appendix~\ref{app:intro}.

\section{A query-complexity lower bound for GVI}\label{sec:lb}
 
Section~\ref{sec:exact} attains the accelerated $O(\sqrt\kappa)$ dependence with exact Gaussian expectations. With pointwise access, where $V$ and its derivatives are evaluated only at individual query points, these expectations must instead be estimated, for example by Monte Carlo. Then, a direct variance bound loses the acceleration: since $\nabla V$ is $\beta$-Lipschitz, the mean component of the standard reparameterization estimator has variance of order $\beta^2\operatorname{tr}(\Sigma)$, so the resulting bound on the variance term in \eqref{eq:intro-stochastic-benchmark} scales as $\kappa^2$. Better estimators may reduce this cost, but we show that under Assumption~\ref{assump:curvature} alone, no algorithm with pointwise access can avoid a nearly linear dependence on $\kappa$ in the total number of target evaluations.
 
For the lower bound, we allow the algorithm to evaluate $V$ and its derivatives of any fixed finite order at each query point. We count each evaluation as one query, including evaluations within a Monte Carlo batch. Each point may be chosen based on previous observations and the algorithm's internal randomness. This includes any Monte Carlo method built from gradient or HVP evaluations. By rescaling the variables $y=\sqrt{\alpha}\,x$, we may normalize $\alpha=1$ and consider the normalized class
\[
\mathcal V_d(\kappa):=\left\{V\in C^\infty(\mathbb R^d):I_d\preceq\nabla^2V(x)\preceq\kappa I_d\text{ for all }x\in\mathbb R^d\right\}.
\]
All asymptotic notation in this section is as $\kappa\to\infty$.
 
\begin{theorem}\label{thm:lb-main}
There exist constants $\varepsilon_0,C_0,C>0$ and $\kappa_0>1$ such that, for every $\kappa\ge\kappa_0$, there exist a dimension $d=\Theta((\log\kappa)^2)$ and a family $\{V_v:v\in\mathbb S^{d-1}\}\subseteq\mathcal V_d(\kappa)$ satisfying
\[
\sup_{v\in\mathbb S^{d-1}}\left(F_{V_v}(\mathcal N(0,I_d))-F_{V_v}^\star\right)\le C_0.
\]
Furthermore, for every finite order $r\ge1$, any randomized algorithm that makes at most $N$ queries, evaluating $V_v$ and its derivatives up to order $r$ at each query point, and returns $q\in\mathcal G_d$ satisfying
\[
\sup_{v\in\mathbb S^{d-1}}\mathbb E\left[F_{V_v}(q)-F_{V_v}^\star\right]\le\varepsilon_0
\]
must satisfy
\[
N\ge\kappa\exp\left[-C(\log\kappa)^{3/4}\right].
\]
\end{theorem}
 
Since $d=\Theta((\log\kappa)^2)$, any fixed polynomial in $d$ is $\kappa^{o(1)}$. At the fixed accuracy $\varepsilon_0$, Theorem~\ref{thm:lb-main} therefore rules out every uniform bound of order $\widetilde O(\operatorname{poly}(d)\kappa^p)$ with $p<1$, and in particular every accelerated bound of the form $O(\operatorname{poly}(d)\sqrt\kappa\,(\log(\Delta_0/\varepsilon)+1/\varepsilon))$. The construction also provides a common initialization $\mathcal N(0,I_d)$ whose objective gap is $1+o(1)$ uniformly over all targets in the family. Thus, the lower bound does not rely on an initial objective gap that grows with $\kappa$.

We now outline the proof of the lower bound. First, we construct a family of targets indexed by a hidden
direction $v$, with uniformly bounded initial objective gaps. Second, we show that pointwise queries are unlikely to reveal the hidden direction. Third, we show that a Gaussian chosen independently of $v$
cannot achieve a small objective gap for most targets in the family.
We develop these arguments in the following paragraphs
and defer the technical proofs to Appendix~\ref{app:lb}.

\paragraph{Construction of the target family.}
Let $G\sim\mathcal N(0,1)$, and choose $R>0$ so that
\[
\frac{\kappa-1}{2}\mathbb E[(G-R)_+^2]=1.
\]
The Gaussian tail expansion gives
$R^2=2\log\kappa-3\log\log\kappa+O(1)$.
Set $d=\lceil R^4\rceil$ and $a=(1+R^{-1/2})\sqrt d$.
Let $B$ be the ball of radius $a$ centered at the origin,
and for each $v\in\mathbb S^{d-1}$, define
$H_v:=\{x:v^\top x\le R\}$ and $K_v:=B\cap H_v$.
Thus $K_v$ is obtained by removing a small cap from $B$
in the hidden direction $v$ (lightly shaded region in Figure~\ref{fig:lb-construction}).

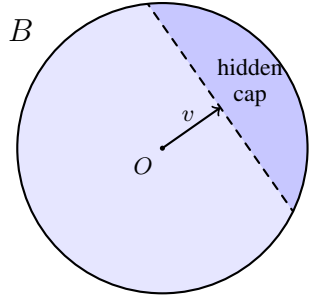
\begin{figure}[t]
\centering
\begin{tikzpicture}[scale=0.48, line cap=round, line join=round]
\def\Rad{4}      
\def\rho{1.95}   
\def\th{35}      
 
\coordinate (O) at (0,0);
\coordinate (P) at ({\rho*cos(\th)},{\rho*sin(\th)}); 
 
\fill[blue!10] (O) circle (\Rad);
 
\begin{scope}
  \clip (O) circle (\Rad);
  \begin{scope}[rotate=\th]
    \fill[blue!22] (\rho,-6) rectangle (6,6);
    \draw[dashed,thick] (\rho,-4.5) -- (\rho,4.5);
  \end{scope}
\end{scope}
 
\draw[thick] (O) circle (\Rad);
 
\node[font=\large] at (-3.9,3.25) {$B$};
 
\fill (O) circle (0.07);
\node[below left,font=\footnotesize] at (O) {$O$};

\draw[->,thick] (O) -- (P);
 
\node[font=\footnotesize]
  at ({0.56*\rho*cos(\th)-0.32*sin(\th)},
      {0.56*\rho*sin(\th)+0.32*cos(\th)}) {$v$};
\begin{scope}[rotate=\th]
  \node[font=\footnotesize, align=center] at (3.02,0.10) {hidden\\cap};
\end{scope}
 
\end{tikzpicture}
\caption{A hidden cap in direction $v$ inside the ball $B$.}
\label{fig:lb-construction}
\end{figure}

For a closed convex set $K$, write
$e_K(x):=\frac12\operatorname{dist}(x,K)^2$.
Let $e_K^\tau:=e_K*\rho_\tau$, where $\rho_\tau$ is a smooth,
even probability density supported on the ball of radius
$\tau=(\kappa-1)^{-2}$. We construct
\begin{equation}\label{eq:construction}
V_v(x):=\frac12\|x\|^2+(\kappa-1)e_{K_v}^\tau(x),
\qquad
V_B(x):=\frac12\|x\|^2+(\kappa-1)e_B^\tau(x).
\end{equation}
These potentials are infinitely differentiable and satisfy
$I_d\preceq\nabla^2V_v,\nabla^2V_B\preceq\kappa I_d$. The parameters are chosen so that the initial Gaussian $\gamma:=\mathcal N(0,I_d)$ satisfies
\[
F_{V_v}(\gamma)-F_{V_v}^\star\le 1+o(1)
\]
uniformly over $v$. See Proposition~\ref{prop:lb-initial-gap} and Lemma~\ref{lem:lb-full-jet} for details.

\paragraph{Difficulty of detecting the hidden region.}
The key property of the construction~\eqref{eq:construction} is that $V_v$ and
$V_B$ agree outside the hidden region. In particular,
Lemma~\ref{lem:lb-full-jet} shows that
\[
v^\top\Pi_Bx\le R-\tau
\quad\Longrightarrow\quad
D^jV_v(x)=D^jV_B(x)
\quad\text{for all }j\ge0.
\]
Thus a pointwise oracle, even with derivatives of any
fixed finite order, cannot distinguish the two potentials
unless its query reaches the hidden region. For a uniformly random direction $v$,  Lemma~\ref{lem:lb-spherical-cap} gives
\[
\sup_x\mathbb P_v(v^\top\Pi_Bx>R-\tau)
\le
\frac{C}{\kappa}
\exp\left[C(\log\kappa)^{3/4}\right].
\]
Hence any fixed query is informative with probability
at most $\kappa^{-1+o(1)}$. Now fix the
algorithm's randomness and compare its transcripts
under $V_v$ and the reference potential $V_B$.
The reference queries do not depend on $v$, and the
two transcripts coincide until a query reaches the
hidden region. A union bound therefore shows that
the probability of distinguishing the two potentials
within $N$ queries is at most
\[
C\frac{N}{\kappa}
\exp\left[C(\log\kappa)^{3/4}\right].
\]
The complete proof is given in
Lemma~\ref{lem:lb-common-path}.

\paragraph{Completion of the proof.}
It remains to show that a Gaussian chosen independently
of $v$ cannot achieve a small objective gap for most
targets in the family.

\begin{lemma}\label{lem:lb-compromise}
There exist constants $\varepsilon_0,c>0$ such that,
for all sufficiently large $\kappa$ and every fixed
Gaussian $q\in\mathcal G_d$,
\[
\mathbb P_v\left(
F_{V_v}(q)-F_{V_v}^\star\le2\varepsilon_0
\right)\le e^{-cR^2}.
\]
\end{lemma}

The GVI objective for the constructed potentials can be written, up to an additive constant, as
\[
F_{V_v}(q)+\frac d2\log(2\pi)
=\mathrm{KL}(q\|\gamma)
+(\kappa-1)\mathbb E_q[e_{K_v}^{\tau}(X)],
\]
where $\gamma=\mathcal N(0,I_d)$. A small objective gap requires $q=\mathcal N(m,\Sigma)$ to remain close to $\gamma=\mathcal N(0,I_d)$ in KL divergence, while also reducing its Gaussian tail in the hidden direction $v$. For a fixed Gaussian $q$, these two conditions can hold simultaneously for only an exponentially small fraction of directions $v$. Appendix~\ref{app:lb} provides a detailed proof using concentration inequalities.

To complete the proof, consider running the algorithm on the reference potential $V_B$ using the same internal randomness, and denote its output by $q^0$. If none of the queries made under $V_B$ reaches the hidden region, the oracle responses under $V_v$ are identical, and the algorithm produces the same output $q=q^0$. Therefore, achieving an objective gap of at most $2\varepsilon_0$ requires either that a query reaches the hidden region or that $q^0$ already achieves this accuracy for $V_v$. The probability of the first event is bounded by Lemma~\ref{lem:lb-common-path}, while the second is bounded by Lemma~\ref{lem:lb-compromise}, since $q^0$ is independent of $v$. Taking a union bound gives
\[
\mathbb P_{v,\mathrm{alg}}\left(
F_{V_v}(q)-F_{V_v}^\star\le2\varepsilon_0
\right)
\le
C\frac{N}{\kappa}
\exp\left[C(\log\kappa)^{3/4}\right]+e^{-cR^2},
\]
where $\mathbb P_{v,\mathrm{alg}}$ denotes probability over a uniformly random direction $v$ and the algorithm's internal randomness.

If the worst-case expected gap is at most $\varepsilon_0$,
Markov's inequality makes the left-hand side at least
$1/2$. Since $R\to\infty$, rearranging yields
\[
N\ge\kappa\exp\left[-C(\log\kappa)^{3/4}\right],
\]
after adjusting $C$, completing the proof sketch of
Theorem~\ref{thm:lb-main}.

\begin{remark}
 The construction has large Hessian variation, so it does not contradict the accelerated guarantees of the next section, which require this variation to be small.
\end{remark}

\section{Recovering stochastic acceleration with small Hessian variation}\label{sec:stochastic}

Section~\ref{sec:lb} shows that Assumption~\ref{assump:curvature} alone is insufficient for uniform stochastic acceleration. We now give positive examples under Assumption~\ref{assump:hessian-lipschitz}: accelerated dependence on $\kappa$ can be recovered when the Hessian variation is sufficiently small. We give two complementary constructions. The first uses a gradient-only finite-difference (FD) control variate in Cholesky coordinates, and the second uses score corrections in symmetric-root coordinates.
For an initial point $x_0=(m_0,A_0)$, let $q_0:=\mathcal N(m_0,A_0A_0^\top)$ and let $x_\star$ denote the minimizer. Throughout this section, write
\[
\Delta_0:=F(q_0)-F^\star+\frac{\alpha}{2}\|x_0-x_\star\|^2,\qquad \lambda_H:=\frac{L_H}{\alpha^{3/2}}.
\]
Note that after normalizing $\alpha=1$, $\lambda_H$ is simply the Hessian-Lipschitz constant.

\subsection{Finite-difference control variates in Cholesky coordinates}
For $\nu\ne0$, define
\[
D_\nu(m;e_j):=\frac{\nabla V(m+\nu e_j)-\nabla V(m)}{\nu},\qquad D_\nu(m):=\bigl[D_\nu(m;e_1)\ \cdots\ D_\nu(m;e_d)\bigr].
\]
The matrix $D_\nu(m)$ is a first-order approximation of $\nabla^2V(m)$ and is exact whenever $V$ is quadratic. For $\xi\sim\mathcal N(0,I_d)$, let
\[
r_\nu(m,L;\xi):=\nabla V(m+L\xi)-\nabla V(m)-D_\nu(m)L\xi
\]
and define
\[
G^{\mathrm{FD}}(m,L;\xi):=\left(\nabla V(m)+r_\nu(m,L;\xi),\operatorname{low}\left(D_\nu(m)L+r_\nu(m,L;\xi)\xi^\top\right)\right).
\]
For a batch of size $b$, $G_b^{\mathrm{FD}}$ denotes the average of $b$ independent samples using a common $D_\nu(m)$. Thus, $d+1$ gradients are shared across the batch, and one update uses $d+1+b$ first-order oracle calls. The FD estimator $G^{\mathrm{FD}}$ is unbiased for $\nabla E$, and its variance admits a bound proportional to $L_H^2$ (Proposition~\ref{prop:fd-estimator-full}).

To control the variance along the iterates, we restrict the covariance factor to $\|L\|_F\le\sqrt{d/\alpha}$ and define 
\[
\bar h(L):=\begin{cases}-\log\det L,&L\in\mathbb T_{++}^d\text{ and }\|L\|_F\le\sqrt{d/\alpha},\\+\infty,&\text{otherwise}.\end{cases}
\]
This restriction does not change the optimum, since $\|L_\star\|_F^2=\operatorname{tr}(\Sigma_\star)\le d/\alpha$. We then apply APG to $E+\bar h$, using $G_b^{\mathrm{FD}}$ in place of the exact energy gradient. Proposition~\ref{prop:fd-prox-map} (in Appendix~\ref{subsec:fd-prox}) shows that the resulting proximal step reduces to solving a univariate monotone equation. 

\begin{algorithm}[hbtp]
\caption{Cholesky proximal acceleration}
\label{alg:fd-prox}
\begin{algorithmic}[1]
\Require $x_0=(m_0,L_0)$ with $L_0\in\operatorname{dom}\bar h$, batch size $b$, FD scale $\nu$, horizon length $T$
\State  $y_0\gets x_0$
\For{$k=0,\ldots,T-1$}
\State $g_k\gets G_b^{\mathrm{FD}}(y_k)$ 
\State $(m_{k+1},\widetilde L_{k+1})\gets y_k-\eta g_k$
\State $x_{k+1}\gets(m_{k+1},\prox_{\eta\bar h}(\widetilde L_{k+1}))$, $y_{k+1}\gets x_{k+1}+\rho(x_{k+1}-x_k)$
\EndFor
\State \Return $q_T=\mathcal N(m_T,L_TL_T^\top)$
\end{algorithmic}
\end{algorithm}

The fixed stepsize $\eta$ and momentum $\rho$ are specified in Theorem~\ref{thm:fd-prox-full}. The resulting accelerated convergence guarantee is as follows.

\begin{theorem}\label{thm:fd-prox}
Suppose Assumptions~\ref{assump:curvature} and~\ref{assump:hessian-lipschitz} hold, $x_0\in\operatorname{dom}\bar h$, $0<|\nu|\le\alpha^{-1/2}$, and $0<\varepsilon\le\Delta_0$. There is a fixed batch size $b$ such that Algorithm~\ref{alg:fd-prox} satisfies
\[
\mathbb E[F(q_{T_\varepsilon})-F^\star]\le\varepsilon,
\qquad
T_\varepsilon=O\left(\sqrt\kappa\log\frac{2\Delta_0}{\varepsilon}\right),
\qquad 
N_{\nabla V}
=O\left(\left(
d\sqrt\kappa+\frac{\lambda_H^2d^3}{\varepsilon}
\right)\log\frac{2\Delta_0}{\varepsilon}
\right).
\]
\end{theorem}

The full constants are given in Theorem~\ref{thm:fd-prox-full}. The same update also accepts the second-order estimator $G^{\mathrm{paired}}$, which evaluates the gradient and the Hessian at independent Gaussian samples (Proposition~\ref{prop:hessian-assisted-final} and Corollary~\ref{cor:hessian-assisted-fixed-batch}). The finite-difference estimator can also be combined with a Euclidean projection, with the same asymptotic first-order complexity; see Appendix~\ref{subsec:fd-projected}.

\subsection{Score corrections in symmetric-root coordinates}

The finite-difference method uses only first-order derivatives but requires $d+1$ additional calibration queries per update. We next use Gaussian score corrections to obtain relative-noise bounds without this additional cost.
For $q = \mathcal{N}(m, C^2)$ with $C\in\mathbb S_{++}^d$ and $X := m + C \xi$, $\xi \sim \mathcal{N}(0, I_d)$, we have  $\nabla_x\log q(X)=-C^{-1}\xi$, so the sticking-the-landing correction \citep{roeder2017sticking} gives
\[
G^{\mathrm{STL}}(m,C;\xi):=\left(\nabla V(X)-C^{-1}\xi,\operatorname{sym}\left((\nabla V(X)-C^{-1}\xi)\xi^\top\right)\right).
\]
Using Stein's identity \citep{stein1981estimation,price1958useful} 
\[
\mathbb E[\nabla V(X)\xi^\top]=\mathbb E[\nabla^2V(X)C],
\]
we define the Price--STL estimator~\citep{luu2025stochastic, kim2026stochastic}
\[
G^{\mathrm{Price\text{-}STL}}(m,C;\xi):=\left(\nabla V(X)-C^{-1}\xi,\operatorname{sym}(\nabla^2V(X)C)-C^{-1}\right).
\]
Both estimators are unbiased for the full gradient $\nabla F(m,C)$. To control the inverse factor and bound the estimator moments, we restrict the iterates to
\[
K_C:=\left\{(m,C):m\in\mathbb R^d,\ C\in\mathbb S^d,\ \frac{1}{2\sqrt\beta}I\preceq C\preceq\frac{2}{\sqrt\alpha}I\right\}.
\]
The optimal symmetric factor lies in the interior of $K_C$. For a batch of size $b$, let $G_b$ denote the average of $b$ independent samples of either score estimator $G^{\mathrm{STL}}$ or $G^{\mathrm{Price\text{-}STL}}$. On $K_C$, both estimators have second moments bounded by a multiple of $\|\nabla F\|^2$ plus an additive term proportional to $L_H^2/b$ (Lemma~\ref{lem:score-growth-full}). We apply the projected acceleration scheme below with $s=\sqrt{15\kappa}$ and $\eta=1/(15\beta)$. (Theorem~\ref{thm:score-recovery-full}). See Appendix~\ref{subsec:score} for details.

\begin{algorithm}[hbtp]
\caption{Projected acceleration with STL or Price--STL}
\label{alg:score-proj}
\begin{algorithmic}[1]
\Require $x_0\in K_C$, batch size $b$, estimator $G_b\in\{G_b^{\mathrm{STL}},G_b^{\mathrm{Price\text{-}STL}}\}$, horizon length $T$
\State  $v_0\gets x_0$, $y_0\gets x_0$
\For{$k=0,\ldots,T-1$}
\State $g_k\gets G_b(y_k)$
\State $x_{k+1}\gets P_{K_C}(y_k-\eta g_k)$
\State $v_{k+1}\gets P_{K_C}((1-s^{-1})v_k+s^{-1}y_k-(\alpha s)^{-1}g_k)$
\State $y_{k+1}\gets(x_{k+1}+s^{-1}v_{k+1})/(1+s^{-1})$
\EndFor
\State \Return $q_T=\mathcal N(m_T,C_T^2)$
\end{algorithmic}
\end{algorithm}

Both projected sequences remain in $K_C$, and each query point is their convex combination. Hence the inverse factor and both score estimators remain well-defined throughout the run.

\begin{theorem}\label{thm:score-recovery-clean}
Suppose Assumptions~\ref{assump:curvature} and~\ref{assump:hessian-lipschitz} hold, $x_0\in K_C$, and $0<\varepsilon\le\Delta_0$. For $G^{\mathrm{STL}}$, there is a fixed batch size such that Algorithm~\ref{alg:score-proj} attains
\[
\mathbb E[F(q_{T_\varepsilon})-F^\star]\le\varepsilon,\qquad T_\varepsilon=O\left(\sqrt\kappa\log\frac{2\Delta_0}{\varepsilon}\right),
\qquad
N_{\nabla V}=O\left(\left(d\sqrt\kappa+\frac{\lambda_H^2d^3}{\varepsilon}\right)\log\frac{2\Delta_0}{\varepsilon}\right).
\]
\end{theorem}

Table~\ref{tab:stochastic-rates} summarizes the fixed batch sizes and oracle complexities.
See Theorem~\ref{thm:score-recovery-full} for the full second-order guarantees.

\begin{table}[hbtp]
\centering
\footnotesize
\renewcommand{\arraystretch}{1.5}
\setlength{\tabcolsep}{5pt}
\begin{tabular}{@{\hspace{10pt}}lccc@{\hspace{10pt}}}
\toprule
\textbf{Method} & \textbf{Batch size} & \textbf{Gradient queries} & \textbf{HVP queries}\\
\midrule
\mbox{FD + Alg.~\ref{alg:fd-prox}} & $1+\dfrac{\lambda_H^2d^3}{\varepsilon\sqrt\kappa}$
& $d\sqrt\kappa+\dfrac{\lambda_H^2d^3}{\varepsilon}$ & $0$\\[0.5em]
\mbox{STL + Alg.~\ref{alg:score-proj}} & $d+\dfrac{\lambda_H^2d^3}{\varepsilon\sqrt\kappa}$
& $d\sqrt\kappa+\dfrac{\lambda_H^2d^3}{\varepsilon}$ & $0$\\[0.5em]
\mbox{Paired + Alg.~\ref{alg:fd-prox}} & $1+\dfrac{\lambda_H^2d^2}{\varepsilon\sqrt\kappa}$
& $\sqrt\kappa+\dfrac{\lambda_H^2d^2}{\varepsilon}$
& $d\sqrt\kappa+\dfrac{\lambda_H^2d^3}{\varepsilon}$\\[0.5em]
\mbox{Price--STL + Alg.~\ref{alg:score-proj}} & $1+\dfrac{\lambda_H^2d^2}{\varepsilon\sqrt\kappa}$
& $\sqrt\kappa+\dfrac{\lambda_H^2d^2}{\varepsilon}$
& $d\sqrt\kappa+\dfrac{\lambda_H^2d^3}{\varepsilon}$\\
\bottomrule
\end{tabular}
\caption{Fixed-batch guarantees. All four methods (estimator, algorithm) attain expected objective gap at most $\varepsilon$ in $T_\varepsilon=O\left(\sqrt\kappa\log(2\Delta_0/\varepsilon)\right)$ iterations. Entries give orders in $O(\cdot)$ notation; the two oracle-count columns omit the common factor $\log(2\Delta_0/\varepsilon)$.}
\label{tab:stochastic-rates}
\end{table}

\paragraph{Interpretation.}
The two first-order methods use gradient evaluations differently: FD uses $d+1$ calibration gradients per update to reduce variance, while STL avoids calibration but requires a batch size proportional to $d$ (Table~\ref{tab:stochastic-rates}).

When $\lambda_H^2d^2\lesssim\varepsilon\sqrt\kappa$, the variance-dependent terms are dominated by the accelerated terms in all oracle-complexity bounds in Table~\ref{tab:stochastic-rates}. Second-order estimators then need only constant batches, and the first-order ones batches of size $O(d)$. This complements Section~\ref{sec:lb}: curvature bounds alone do not permit uniform acceleration, while controlled Hessian variation provides a guarantee for acceleration.
 
We have not optimized the dimension dependence.
Existing analyses of full-covariance GVI give iteration bounds with explicit factors linear in $d$~\citep{domke2023provable,kim2024linear,kim2026stochastic}.
For Price--STL, a sharper variance analysis under additional structural assumptions also yields a gradient-query sampling term linear in $d$ (Appendix~\ref{app:dimension}).

\section{Experiments}\label{sec:experiments}

In this section, we present experimental results for our accelerated methods. 
Each method is labeled with the estimator, proximal/projected update (prox/proj), and accelerated/nonaccelerated variant (acc./nonacc.) that is used.
We refer to Appendix~\ref{app:experiments} for full experimental details.

\paragraph{Gaussian targets.}
We compare Algorithm~\ref{alg:score-proj} using $G_b^{\mathrm{STL}}$ with its nonaccelerated projected-gradient counterpart, at $b=64$.
We use the standard normal distribution initialization $q_0 = \mathcal{N}(0,I_d)$, and vary the condition number $\kappa$ from $10$ to $3000$. 
Figure~\ref{fig:exp-gaussian} confirms the predicted scaling: the iteration count grows as $\sqrt\kappa$ with momentum and as $\kappa$ without. 
Appendix~\ref{exp:gaussian} also reports the samplewise-exact Paired benchmark and comparisons with all seventeen methods at $\kappa=300$.

\begin{figure}[!htbp]
\centering
\includegraphics[width=\linewidth]{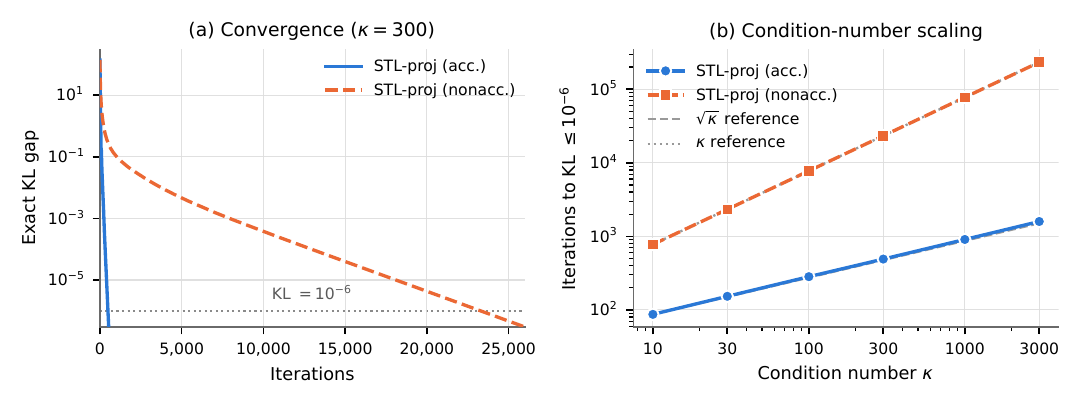}
\caption{\textbf{Stochastic acceleration on Gaussian targets.} STL-proj with $b=64$: (a) mean KL gap at $\kappa=300$ with $20$--$80$ percentile bands over five seeds; (b) updates to mean KL $\le10^{-6}$, with $\sqrt\kappa$ and $\kappa$ reference lines.}
\label{fig:exp-gaussian}
\end{figure}

\paragraph{Non-Gaussian targets.}
We construct graph-probit posteriors~\citep{miller2020graphprobit,bertozzi2018uncertainty} from five binary-classification datasets in OpenML \citep{vanschoren2014openml}. Standardized predictors define a $10$-nearest-neighbor graph with normalized Laplacian $L_G$. We use one latent score per observation and Gaussian prior precision $Q_\tau=I_n+\tau L_G$, with chosen hyperparameter $\tau=49$.
Graph-probit allows conditioning to worsen without increasing
Hessian variation: increasing $\tau$ raises
$\kappa=2+\tau\lambda_{\max}(L_G)$, while $L_H$ remains
independent of $\tau$.
Refer to Appendix~\ref{exp:targets} for further details. 

We compare seventeen first- and second-order methods, including five baselines~\citep{domke2023provable,diao2023forward,luu2025stochastic,kim2026stochastic}, from $\mathcal N(0,I_d)$ for $600$ updates with five seeds. 
FD and paired use proximal Algorithm~\ref{alg:fd-prox} or projected Algorithm~\ref{alg:fd-proj}, and STL and Price-STL use Algorithm~\ref{alg:score-proj}. 
The matching variants share the estimator, stepsize, and derivative cost per update. Appendix~\ref{exp:graph-probit-results} also compares methods by gradient-query cost. 
Objective gaps use deterministic quadrature and a common high-accuracy numerical reference $F_{\mathrm{ref}}\approx F^\star$ (Appendix~\ref{exp:evaluation}).

\vspace{-1em}
\begin{figure}[!t]
\centering
\includegraphics[width=\linewidth]{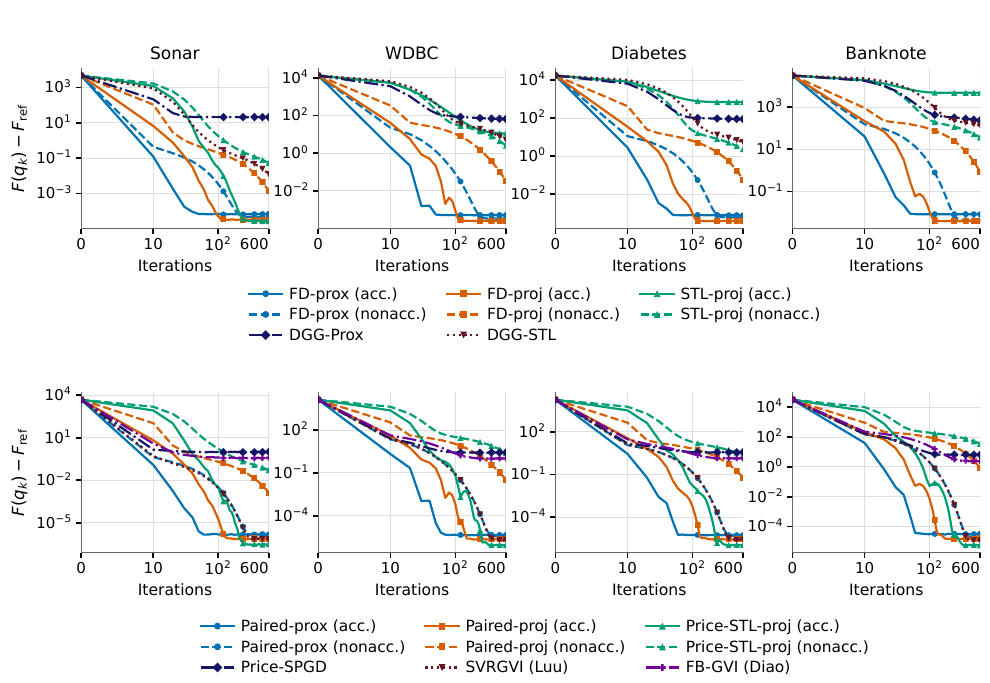}
\caption{\textbf{Convergence on graph-probit posteriors.} The top row compares first-order methods and the bottom row compares second-order methods, all with batch size $64$. Curves are five-seed mean objective gaps from $\mathcal N(0,I_d)$, with $20$--$80$ percentile bands. Solid and dashed lines distinguish accelerated and nonaccelerated variants within each matched pair.
Appendix~\ref{exp:graph-probit-results} reports comparisons across all five datasets.
}
\label{fig:exp-openml}
\end{figure}

\paragraph{Results.}
At $b=64$, every accelerated variant except STL-proj reaches $10^{-5}$ of the initial mean objective gap before its nonaccelerated counterpart on all five datasets (Table~\ref{tab:exp-probit-hitting}).
Consistent with its dimension-dependent batch requirement, 
momentum can hurt STL-proj at small batches.
At $b=256$, it reaches the $1\%$ threshold sooner and attains smaller final gaps than its nonaccelerated counterpart on all five datasets
(Appendix~\ref{exp:graph-probit-results}).

\section{Conclusion} \label{sec:conclusion}

This work takes a first step toward the theoretical foundations of momentum-based stochastic acceleration in Gaussian variational inference. Exact Gaussian expectations permit classical $\sqrt{\kappa}$ acceleration, but, perhaps surprisingly, curvature bounds alone cannot ensure uniform acceleration under pointwise access. This contrasts with recent progress on accelerated sampling for strongly log-concave targets~\citep{altschuler2025ballistic,chen2026exact,lu2026windowed}. 
Our variance-reduced stochastic methods recover accelerated guarantees under sufficiently small Hessian variation, and experiments show empirical gains on real-data posteriors.
Future work includes identifying which structures determine stochastic acceleration, and comparing systematically with geometry-aware methods such as natural gradient descent, which are subject to the same lower bound yet often perform well in practice, suggesting that they exploit different favorable structure.

\clearpage
\section*{Disclosure of AI assistance}
In this work, we used AI tools (GPT-5.6 Sol and GPT-6 Astra) throughout the research.
The research questions, the mathematical claims, and the main ideas were developed by the authors. For the theoretical results in Section~\ref{sec:lb} and Section~\ref{sec:stochastic}, AI tools assisted in exploring possible proof strategies and checking the validity of individual proof steps. They were also used to search for examples presented in Section~\ref{app:dimension},
For experiments, AI tools assisted in drafting code for implementing the algorithms suggested in the paper, and the authors reviewed it. 
AI tools also assisted with manuscript drafting and proofreading. 
The authors reviewed and checked the mathematical arguments and examined the experimental results. We take full responsibility for the final content of this work.

\appendix

\section{Omitted proofs for Section~\ref{sec:intro} and Section~\ref{sec:exact}}\label{app:intro}

We first justify differentiating under the expectation. Recall
\[
E(m,A):=\mathbb E_\xi[V(m+A\xi)],\qquad \xi\sim\mathcal N(0,I_d).
\]
Assumption~\ref{assump:curvature} gives $\|\nabla^2V(x)\|_{\mathrm{op}}\le\beta$, and hence
\[
\|\nabla V(x)\|\le\|\nabla V(0)\|+\beta\|x\|,\qquad
|V(x)|\le |V(0)|+\|\nabla V(0)\|\|x\|+\frac{\beta}{2}\|x\|^2.
\]
On any bounded neighborhood of $(m,A)$, $V(m+A\xi)$ and its first two derivatives with respect to $(m,A)$ are therefore bounded in absolute value by $C(1+\|\xi\|^2)$ for some constant $C$. Since this bound is Gaussian-integrable, dominated convergence justifies differentiating twice under the expectation. We can now prove Lemma~\ref{lem:curvature-inheritance}.

\begin{proof}[Proof of Lemma~\ref{lem:curvature-inheritance}]
Fix $\theta=(m,A)$ and a direction $(u,U)$, and define
\[
g(s):=E(m+su,A+sU)=\mathbb{E}_{\xi}\left[V\bigl(m+A\xi+s(u+U\xi)\bigr)\right].
\]
The second derivative of the integrand is bounded in absolute value by $\beta\|u+U\xi\|^2$, which is integrable. Thus
\[
g''(0)=\mathbb{E}_{\xi}\left[(u+U\xi)^\top\nabla^2V(m+A\xi)(u+U\xi)\right].
\]
Applying Assumption~\ref{assump:curvature} pointwise gives
\[
\alpha\mathbb{E}_{\xi}\|u+U\xi\|^2\le g''(0)\le\beta\mathbb{E}_{\xi}\|u+U\xi\|^2.
\]
Since $\mathbb{E}[\xi]=0$ and $\mathbb{E}[\xi\xi^\top]=I_d$,
\[
\mathbb{E}_{\xi}\|u+U\xi\|^2=\|u\|^2+\|U\|_F^2=\|(u,U)\|^2.
\]
Because $g''(0)=\nabla^2E(\theta)[(u,U),(u,U)]$, the result follows.
\end{proof}

We then show the explicit proximal map of the entropy term used in Section~\ref{sec:exact}.

\begin{proposition}
\label{prop:logdet-prox}
For $\eta>0$, let $\phi_\eta(y):=\frac{y+\sqrt{y^2+4\eta}}{2}.$ If $\mathcal A=\mathbb T^d$, then $\prox_{\eta h}$ leaves the strictly lower-triangular entries unchanged and replaces each diagonal entry $Y_{ii}$ by $\phi_\eta(Y_{ii})$. If $\mathcal A=\mathbb S^d$ and $Y=Q\operatorname{diag}(\lambda_1,\ldots,\lambda_d)Q^\top$, then $ \prox_{\eta h}(Y) = Q\operatorname{diag}\bigl( \phi_\eta(\lambda_1),\ldots,\phi_\eta(\lambda_d) \bigr)Q^\top. $ 
\end{proposition}

\begin{proof}
Suppose first that $\mathcal A=\mathbb T^d$. Since
$-\log\det A=-\sum_i\log A_{ii}$, the proximal objective separates over
the entries. Its strictly lower-triangular entries are minimized at
$A_{ij}=Y_{ij}$, while each diagonal entry solves
\[
\min_{\ell>0}
\left\{
-\log\ell+\frac{(\ell-Y_{ii})^2}{2\eta}
\right\}.
\]
The stationarity condition gives
$\ell^2-Y_{ii}\ell-\eta=0$, whose unique positive solution is
$\ell=\phi_\eta(Y_{ii})$.

Now suppose that $\mathcal A=\mathbb S^d$. The objective is strictly
convex, coercive, and diverges at the boundary of $\mathbb S_{++}^d$,
so it has a unique minimizer $A^+$. Its first-order condition is
\[
-(A^+)^{-1}+\frac{A^+-Y}{\eta}=0,
\]
and hence $Y=A^+-\eta(A^+)^{-1}$. Thus, $Y$ and $A^+$ are simultaneously
diagonalizable. Writing
\[
Y=Q\operatorname{diag}(\lambda_1,\ldots,\lambda_d)Q^\top,
\qquad
A^+=Q\operatorname{diag}(\ell_1,\ldots,\ell_d)Q^\top,
\]
we obtain $\ell_i^2-\lambda_i\ell_i-\eta=0$. Its unique positive solution
is $\ell_i=\phi_\eta(\lambda_i)$, which proves the result.
\end{proof}

The following proof establishes the convergence guarantee of Theorem~\ref{thm:exact-apg}.

\begin{proof}[Proof of Theorem~\ref{thm:exact-apg}]
We first verify that every update is well defined. By Lemma~\ref{lem:curvature-inheritance}, $E$ is $\alpha$-strongly convex and $\beta$-smooth on the entire linear space $\mathbb{R}^d\times\mathcal{A}$. Hence $\nabla E(y_k)$ is well defined even when the extrapolated factor in $y_k$ does not belong to $\mathcal{A}_{++}$. Proposition~\ref{prop:logdet-prox} shows that the subsequent proximal step returns a factor in $\mathcal{A}_{++}$. Starting from $x_0\in\mathbb{R}^d\times\mathcal{A}_{++}$, induction therefore shows that every accepted iterate $x_k$ belongs to $\mathbb{R}^d\times\mathcal{A}_{++}$.

The function $h$ is proper, closed, and convex. Therefore, the standard guarantee for accelerated proximal gradient with step size $1/\beta$ and momentum $\rho=\tfrac{1-\sqrt{\alpha/\beta}}{1+\sqrt{\alpha/\beta}}$
gives \citep{nesterov2013gradient}
\[
F(x_k)-F(x_\star)\le(1-\kappa^{-1/2})^k\left(F(x_0)-F(x_\star)+\frac{\alpha}{2}\|x_0-x_\star\|^2\right).
\]

For either $\mathcal{A}=\mathbb{T}^d$ or $\mathcal{A}=\mathbb{S}^d$, the map $A\mapsto AA^\top$ is a bijection from $\mathcal{A}_{++}$ onto $\mathbb{S}_{++}^d$. Thus minimizing in factor coordinates is equivalent to minimizing over $\mathcal{G}_d$, and
\[
F(x_k)=F(q_k),\qquad F(x_\star)=F^\star.
\]
Substitution into the preceding estimate proves the claimed contraction.

Finally, $1-t\le e^{-t}$ gives
\[
(1-\kappa^{-1/2})^{T_\varepsilon}\Delta_0\le e^{-T_\varepsilon/\sqrt\kappa}\Delta_0\le\varepsilon.
\]
Each update makes exactly one exact energy-gradient oracle call. Hence both the number of updates and the number of oracle calls equal $T_\varepsilon$, which proves the stated complexity bounds.
\end{proof}

\section{Proofs for Section~\ref{sec:lb}}\label{app:lb}

We provide the technical details of the construction and proof outlined in Section~\ref{sec:lb}. 

Throughout this section, all asymptotic notation is with respect to $\kappa\to\infty$. The symbols $c,C>0$ denote constants whose values may change from line to line; the fixed constant $C$ appearing in Theorem~\ref{thm:lb-main} will be chosen at the end of the proof. Each occurrence of ``for all sufficiently large $\kappa$'' means that the corresponding statement holds beyond a finite threshold depending only on constants fixed up to that point. At the end of the proof, $\kappa_0$ is chosen larger than all of the finitely many thresholds introduced.

The hidden direction $v$ is uniformly distributed on
$\mathbb S^{d-1}$ and is independent of the algorithm's
randomness. We write $\mathbb P_v$ for probability over
$v$ with the algorithm's randomness fixed, and
$\mathbb P_{v,\mathrm{alg}}$ for the joint probability.

For sufficiently large $\kappa$, define $R=R(\kappa)>0$ as the unique solution of
\[
\frac{\kappa-1}{2}\mathbb E[(G-R)_+^2]=1,\qquad G\sim\mathcal N(0,1).
\]
Indeed, $R\mapsto\mathbb E[(G-R)_+^2]$ is continuous and strictly decreasing,
\[
\frac{\kappa-1}{2}\mathbb E[(G-0)_+^2]=\frac{\kappa-1}{4}>1
\]
for all sufficiently large $\kappa$, while
\[
\lim_{R\to\infty}\frac{\kappa-1}{2}\mathbb E[(G-R)_+^2]=0.
\]
Hence the required positive solution exists and is unique. Moreover, this definition implies $R\to\infty$ as $\kappa\to\infty$. We use the standard Gaussian tail expansion
\[
\mathbb E[(G-R)_+^2]=\frac{2}{\sqrt{2\pi}R^3}e^{-R^2/2}\left(1+O(R^{-2})\right).
\]
Substituting this into the defining equation for $R$ gives
\[
\frac{\kappa-1}{\sqrt{2\pi}R^3}e^{-R^2/2}\left(1+O(R^{-2})\right)=1.
\]
Taking logarithms yields
\[
R^2=2\log\kappa-6\log R+O(1).
\]
Since $R\to\infty$, this implies
$R^2\sim 2\log\kappa$, and hence
$\log R=\frac12\log\log\kappa+O(1)$.
Substituting back gives
\[
R^2=2\log\kappa-3\log\log\kappa+O(1).
\]
Now, set
\[
d:=\lceil R^4\rceil,\qquad a:=(1+R^{-1/2})\sqrt d,
\]
and let
\[
B:=\overline B_2(0,a):=\{x\in\mathbb R^d:\|x\|\le a\}
\]
be the closed Euclidean ball of radius $a$ centered at the origin. Then
\[
d=\Theta(R^4)=\Theta((\log\kappa)^2),\qquad a-\sqrt d=R^{3/2}+o(R^{3/2}).
\]
We write $\gamma:=\mathcal N(0,I_d)$ for simplicity. For $v\in\mathbb S^{d-1}$, define
\[
H_v:=\{x:v^\top x\le R\},\qquad K_v:=B\cap H_v,
\]
and, for every nonempty closed convex set $K$, let
\[
e_K(x):=\frac12\operatorname{dist}(x,K)^2.
\]
Consider the potentials
\[
V_v^0(x):=\frac12\|x\|^2+(\kappa-1)e_{K_v}(x),\qquad V_B^0(x):=\frac12\|x\|^2+(\kappa-1)e_B(x).
\]
For every nonempty closed convex set $K$,
\[
\nabla e_K(x)=x-\proj_Kx,\qquad \|\nabla e_K(x)-\nabla e_K(y)\|\le\|x-y\|.
\]
Since $e_K$ is convex and has a $1$-Lipschitz gradient, both $V_v^0$ and $V_B^0$ are $1$-strongly convex and have gradient Lipschitz constant at most $\kappa$. The following projection identities describe precisely where the hidden target $V_v^0$ agrees with the reference target $V_B^0$ and will be used repeatedly.

\begin{lemma}\label{lem:lb-projection-identity}
For every $x\in\mathbb R^d$,
\[
\proj_{K_v}x=\proj_Bx\quad\Longleftrightarrow\quad v^\top\proj_Bx\le R.
\]
Moreover,
\[
\operatorname{dist}(x,H_v)\le\operatorname{dist}(x,K_v)\le\operatorname{dist}(x,H_v)+\operatorname{dist}(x,B).
\]
\end{lemma}

\begin{proof}
Let $y=\proj_Bx$. If $v^\top y\le R$, then $y\in K_v$. Since $K_v\subseteq B$,
\[
\operatorname{dist}(x,B)\le\operatorname{dist}(x,K_v)\le\|x-y\|=\operatorname{dist}(x,B).
\]
Thus $y$ attains the minimum distance from $x$ to $K_v$, and uniqueness of the Euclidean projection onto the closed convex set $K_v$ gives $\proj_{K_v}x=y=\proj_Bx$. Conversely, if $\proj_{K_v}x=\proj_Bx$, then $\proj_Bx\in K_v$, and hence $v^\top\proj_Bx\le R$.

The first distance inequality follows immediately from $K_v\subseteq H_v$. For the second, let $y=\proj_{H_v}x$ and $z=\proj_By$. Since $0,y\in H_v$ and $z$ lies on the line segment joining $0$ and $y$, convexity of $H_v$ gives $z\in H_v$. Hence $z\in B\cap H_v=K_v$. Therefore,
\[
\operatorname{dist}(x,K_v)\le\|x-z\|\le\|x-y\|+\|y-z\|=\operatorname{dist}(x,H_v)+\operatorname{dist}(y,B).
\]
Since $0\in H_v$, nonexpansiveness of the projection onto $H_v$ gives
\[
\|y\|=\|\proj_{H_v}x-\proj_{H_v}0\|\le\|x\|.
\]
Because $B$ is the ball of radius $a$ centered at the origin,
\[
\operatorname{dist}(y,B)=(\|y\|-a)_+\le(\|x\|-a)_+=\operatorname{dist}(x,B).
\]
Combining the last two inequalities proves the lemma.
\end{proof}

We next show that the probability mass outside the enclosing ball is negligible.

\begin{lemma}\label{lem:lb-outer-ball-tail}
If $X\sim\mathcal N(m,\Sigma)$, $\Sigma\preceq I_d$, and $\|m\|=o(R^{3/2})$, then
\[
\mathbb E\operatorname{dist}(X,B)^2\le\exp\left[-\frac12R^3+o(R^3)\right].
\]
\end{lemma}

\begin{proof}
Write $X=m+\Sigma^{1/2}G_d$ with $G_d\sim\mathcal N(0,I_d)$, and define $f(g):=\|m+\Sigma^{1/2}g\|$.
Since $\Sigma\preceq I_d$,
\[
|f(g)-f(h)|\le\|\Sigma^{1/2}(g-h)\|\le\|g-h\|,
\]
so $f$ is $1$-Lipschitz. Moreover,
\[
\mathbb E\|X\|\le\|m\|+\mathbb E\|\Sigma^{1/2}G_d\|\le\|m\|+\sqrt{\operatorname{tr}\Sigma}\le\|m\|+\sqrt d.
\]
Hence, using $a=(1+R^{-1/2})\sqrt d$, $d=\lceil R^4\rceil$, and $\|m\|=o(R^{3/2})$,
\[
a-\mathbb E\|X\|\ge R^{-1/2}\sqrt d-\|m\|=R^{3/2}+o(R^{3/2}).
\]
Thus $a-\mathbb E\|X\|>0$ for all sufficiently large $\kappa$. By Gaussian concentration for $1$-Lipschitz functions,
\[
\mathbb P\left(\|X\|-\mathbb E\|X\|\ge t\right)\le e^{-t^2/2}.
\]
Since $\operatorname{dist}(X,B)=(\|X\|-a)_+$, tail integration gives
\[
\mathbb E\operatorname{dist}(X,B)^2=2\int_0^\infty t\,\mathbb P(\|X\|>a+t)\,dt.
\]
Therefore,
\[
\mathbb E\operatorname{dist}(X,B)^2\le2\int_0^\infty t\exp\left[-\frac12(a-\mathbb E\|X\|+t)^2\right]dt.
\]
Using $(u+t)^2\ge u^2+t^2$ for $u,t\ge0$,
\[
\mathbb E\operatorname{dist}(X,B)^2\le2\exp\left[-\frac12(a-\mathbb E\|X\|)^2\right].
\]
The claimed estimate follows from
\[
(a-\mathbb E\|X\|)^2\ge R^3+o(R^3).
\]
\end{proof}

We can now verify the $O(1)$ initial gap.

\begin{proposition}\label{prop:lb-initial-gap}
Uniformly over $v\in\mathbb S^{d-1}$,
\[
F_{V_v^0}(\gamma)-F_{V_v^0}^\star=1+o(1).
\]
\end{proposition}

\begin{proof}
Since $\gamma=\mathcal N(0,I_d)$, for every $q\in\mathcal G_d$,
\[
F_{V_v^0}(q)=\mathrm{KL}(q\,\|\,\gamma)+(\kappa-1)\mathbb E_q[e_{K_v}(X)]-\frac d2\log(2\pi).
\]
Consequently,
\[
F_{V_v^0}(\gamma)-F_{V_v^0}^\star=(\kappa-1)\mathbb E_\gamma[e_{K_v}(X)]-\inf_{p\in\mathcal G_d}\left\{\mathrm{KL}(p\,\|\,\gamma)+(\kappa-1)\mathbb E_p[e_{K_v}(X)]\right\}.
\]

We first evaluate the normalized objective at $\gamma$. Since $\|v\|=1$ and $H_v=\{x:v^\top x\le R\}$,
\[
\operatorname{dist}(x,H_v)=(v^\top x-R)_+.
\]
Lemma~\ref{lem:lb-projection-identity} gives
\[
0\le\operatorname{dist}(x,K_v)^2-\operatorname{dist}(x,H_v)^2\le2\operatorname{dist}(x,H_v)\operatorname{dist}(x,B)+\operatorname{dist}(x,B)^2.
\]
For $X\sim\gamma$, rotational invariance gives $v^\top X\sim\mathcal N(0,1)$, and hence
\[
\mathbb E_\gamma\operatorname{dist}(X,H_v)^2=\mathbb E[(G-R)_+^2]\le1.
\]
Therefore, by Cauchy--Schwarz,
\[
0\le\mathbb E_\gamma[e_{K_v}(X)]-\mathbb E_\gamma[e_{H_v}(X)]\le\left(\mathbb E_\gamma\operatorname{dist}(X,B)^2\right)^{1/2}+\frac12\mathbb E_\gamma\operatorname{dist}(X,B)^2.
\]
Lemma~\ref{lem:lb-outer-ball-tail} gives
\[
\mathbb E_\gamma[e_{K_v}(X)]-\mathbb E_\gamma[e_{H_v}(X)]\le\exp\left[-\frac14R^3+o(R^3)\right].
\]
Since $\log\kappa=O(R^2)=o(R^3)$,
\[
(\kappa-1)\left(\mathbb E_\gamma[e_{K_v}(X)]-\mathbb E_\gamma[e_{H_v}(X)]\right)=O(e^{-cR^3})=o(1)
\]
uniformly over $v$. Using the defining equation for $R$,
\[
(\kappa-1)\mathbb E_\gamma[e_{K_v}(X)]=\frac{\kappa-1}{2}\mathbb E[(G-R)_+^2]+o(1)=1+o(1)
\]
uniformly over $v$.

It remains to show that the infimum of the normalized objective is $o(1)$. Set $\mu:=\frac{\log R}{R}$ and consider the Gaussian $q:=\mathcal N(-\mu v,I_d)$. Its relative entropy with respect to $\gamma$ is
\[
\mathrm{KL}(q\,\|\,\gamma)=\frac{\mu^2}{2}=O\left(\frac{(\log R)^2}{R^2}\right)=o(1).
\]
Under $q$,
\[
v^\top X\overset{d}=G-\mu,\qquad G\sim\mathcal N(0,1),
\]
and hence
\[
(\kappa-1)\mathbb E_q[e_{H_v}(X)]=\frac{\kappa-1}{2}\mathbb E[(G-(R+\mu))_+^2].
\]
Using the defining relation for $R$,
\[
(\kappa-1)\mathbb E_q[e_{H_v}(X)]=\frac{\mathbb E[(G-(R+\mu))_+^2]}{\mathbb E[(G-R)_+^2]}.
\]
The Gaussian tail expansion gives
\[
\frac{\mathbb E[(G-(R+\mu))_+^2]}{\mathbb E[(G-R)_+^2]}=\left(\frac{R}{R+\mu}\right)^3\exp\left[-R\mu-\frac{\mu^2}{2}\right](1+o(1)).
\]
Since $R\mu=\log R$ and $\mu/R\to0$,
\[
(\kappa-1)\mathbb E_q[e_{H_v}(X)]=\frac{1+o(1)}{R}=o(1).
\]
Moreover, $\|-\mu v\|=o(R^{3/2})$, so Lemma~\ref{lem:lb-outer-ball-tail} and the same comparison between $K_v$ and $H_v$ give
\[
(\kappa-1)\left(\mathbb E_q[e_{K_v}(X)]-\mathbb E_q[e_{H_v}(X)]\right)=o(1).
\]
Therefore,
\[
\mathrm{KL}(q\,\|\,\gamma)+(\kappa-1)\mathbb E_q[e_{K_v}(X)]=o(1).
\]
So,
\[
0\le\inf_{p\in\mathcal G_d}\left\{\mathrm{KL}(p\,\|\,\gamma)+(\kappa-1)\mathbb E_p[e_{K_v}(X)]\right\}\le o(1).
\]
\end{proof}
We next smooth the construction to make the targets infinitely differentiable. Choose a nonnegative, even, radial density $\rho\in C_c^\infty(\mathbb R^d)$ supported on the closed unit ball $\overline{B}_2(0,1)$. Set
\[
\tau:=(\kappa-1)^{-2},\qquad \rho_\tau(u):=\tau^{-d}\rho(u/\tau),\qquad e_K^\tau:=e_K*\rho_\tau.
\]
Define
\[
V_v(x):=\frac12\|x\|^2+(\kappa-1)e_{K_v}^\tau(x),\qquad V_B(x):=\frac12\|x\|^2+(\kappa-1)e_B^\tau(x).
\]

\begin{lemma}\label{lem:lb-full-jet}
For every nonempty closed convex set $K$,
\[
0\le e_K^\tau(x)-e_K(x)\le\frac{\tau^2}{2}.
\]
Moreover,
\[
0\preceq\nabla^2e_K^\tau(x)\preceq I_d.
\]
Consequently, $V_v\in C^\infty(\mathbb R^d)$ and
$I_d\preceq\nabla^2V_v(x)\preceq\kappa I_d$. 

If $v^\top\proj_Bx\le R-\tau$, then $V_v$ and $V_B$ have the same derivatives of every finite order at $x$.

Furthermore, for every $q\in\mathcal G_d$,
\[
0\le F_{V_v}(q)-F_{V_v^0}(q)\le\frac{1}{2(\kappa-1)^3},
\]
and
\[
0\le F_{V_v}^\star-F_{V_v^0}^\star\le\frac{1}{2(\kappa-1)^3}.
\]
In particular,
\[
F_{V_v}(\gamma)-F_{V_v}^\star=1+o(1)
\]
uniformly over $v\in\mathbb S^{d-1}$.
\end{lemma}

\begin{proof}
Let $Z\sim\rho_\tau$. Since $\rho_\tau$ is even, $\mathbb EZ=0$. Hence, by convexity of $e_K$ and Jensen's inequality,
\[
e_K^\tau(x)=\mathbb E[e_K(x-Z)]\ge e_K(x-\mathbb EZ)=e_K(x).
\]
Since $e_K$ has a $1$-Lipschitz gradient,
\[
e_K(x-Z)\le e_K(x)-\langle\nabla e_K(x),Z\rangle+\frac12\|Z\|^2.
\]
Taking expectations and using $\mathbb EZ=0$ and $\|Z\|\le\tau$ gives
\[
0\le e_K^\tau(x)-e_K(x)\le\frac{\tau^2}{2}.
\]

Since $\rho_\tau\in C_c^\infty(\mathbb R^d)$, convolution with $\rho_\tau$ gives $e_K^\tau\in C^\infty(\mathbb R^d)$. Moreover, $e_K^\tau$ is convex because it is an average of translates of the convex function $e_K$. Also,
\[
\nabla e_K^\tau(x)=\mathbb E[\nabla e_K(x-Z)],
\]
and therefore
\[
\|\nabla e_K^\tau(x)-\nabla e_K^\tau(y)\|\le\mathbb E\|\nabla e_K(x-Z)-\nabla e_K(y-Z)\|\le\|x-y\|.
\]
Thus $e_K^\tau$ is smooth and convex with a $1$-Lipschitz gradient, so $0\preceq\nabla^2e_K^\tau(x)\preceq I_d.$ The Hessian bounds for $V_v$ follow immediately.

The map $x\mapsto v^\top\proj_Bx$ is $1$-Lipschitz because $\proj_B$ is nonexpansive and $\|v\|=1$. Hence, if
\[
v^\top\proj_Bx\le R-\tau
\]
and $\|u\|\le\tau$, then
\[
v^\top\proj_B(x-u)\le v^\top\proj_Bx+\|u\|\le R.
\]
By Lemma~\ref{lem:lb-projection-identity}, this implies
\[
\proj_{K_v}(x-u)=\proj_B(x-u),
\]
and therefore
\[
e_{K_v}(x-u)=e_B(x-u).
\]
For every multi-index $\alpha$, differentiating the convolution kernel gives
\[
\partial^\alpha(e_{K_v}^\tau-e_B^\tau)(x)
=\int_{\mathbb R^d}[e_{K_v}(x-u)-e_B(x-u)]\,\partial^\alpha\rho_\tau(u)\,du=0,
\]
since $\operatorname{supp}(\partial^\alpha\rho_\tau)\subseteq\overline B_2(0,\tau)$. Thus $V_v$ and $V_B$ have the same values and all finite-order derivatives at $x$, including when $v^\top\Pi_Bx=R-\tau$.

Finally, the first bound and $\tau=(\kappa-1)^{-2}$ give, for every $x\in\mathbb R^d$,
\[
0\le V_v(x)-V_v^0(x)\le\frac{\kappa-1}{2}\tau^2=\frac{1}{2(\kappa-1)^3}.
\]
Taking expectations under any $q\in\mathcal G_d$ gives
\[
0\le F_{V_v}(q)-F_{V_v^0}(q)\le\frac{1}{2(\kappa-1)^3}.
\]
Taking infima over $q\in\mathcal G_d$ similarly yields
\[
0\le F_{V_v}^\star-F_{V_v^0}^\star\le\frac{1}{2(\kappa-1)^3}.
\]
Therefore,
\[
\left|\left(F_{V_v}(\gamma)-F_{V_v}^\star\right)-\left(F_{V_v^0}(\gamma)-F_{V_v^0}^\star\right)\right|\le\frac{1}{2(\kappa-1)^3}=o(1).
\]
Proposition~\ref{prop:lb-initial-gap} now gives
\[
F_{V_v}(\gamma)-F_{V_v}^\star=1+o(1)
\]
uniformly over $v$.
\end{proof}

We next bound the probability that a point query reaches the region where the oracle response depends on $v$.

\begin{lemma}\label{lem:lb-spherical-cap}
There exists a constant $C>0$ such that, for every fixed $z\in\mathbb R^d$ and $0<u<\|z\|$,
\[
\mathbb P_v(v^\top z>u)\le\frac{C\|z\|}{u\sqrt d}\exp\left[-\frac{d-1}{2}\frac{u^2}{\|z\|^2}\right].
\]
Moreover, for all sufficiently large $\kappa$, uniformly over $x\in\mathbb R^d$,
\[
\mathbb P_v\left(v^\top\proj_Bx>R-\tau\right)\le\frac{C}{\kappa}\exp\left[C(\log\kappa)^{3/4}\right].
\]
\end{lemma}

\begin{proof}
By rotational invariance, we may assume without loss of generality that $z=\|z\|e_1$. Writing $v=(v_1,\ldots,v_d)$, we then have $v^\top z=\|z\|v_1$.
To derive its distribution, let
\[
g=(g_1,\ldots,g_d)\sim\mathcal N(0,I_d).
\]
Since the standard Gaussian distribution is rotationally invariant, $v:=\frac{g}{\|g\|}$ is uniformly distributed on $\mathbb S^{d-1}$. In particular,
\[
v_1^2=\frac{g_1^2}{g_1^2+\cdots+g_d^2}.
\]
Set
\[
U:=g_1^2,\qquad W:=g_2^2+\cdots+g_d^2.
\]
Then $U$ and $W$ are independent and satisfy
\[
U\sim\chi_1^2,\qquad W\sim\chi_{d-1}^2.
\]
Equivalently, they are independent Gamma random variables with the same scale parameter and shape parameters $1/2$ and $(d-1)/2$, respectively. Therefore,
\[
\frac{U}{U+W}\sim\operatorname{Beta}\left(\frac12,\frac{d-1}{2}\right).
\]
Thus $T:=v_1^2$ has density
\[
f_T(t)=\frac{1}{B\left(\frac12,\frac{d-1}{2}\right)}t^{-1/2}(1-t)^{(d-3)/2},\qquad 0<t<1,
\]
where $B(a,b):=\Gamma(a)\Gamma(b)/\Gamma(a+b)$ is the beta function. The distribution of $v_1$ is symmetric about zero. Since $T=v_1^2$, a change of variables gives, for $0<x<1$,
\[
f_T(x^2)=\frac{f_{v_1}(x)+f_{v_1}(-x)}{2x}=\frac{f_{v_1}(x)}{x}.
\]
Hence
\[
f_{v_1}(x)=c_d(1-x^2)^{(d-3)/2},\qquad -1<x<1,
\]
where $c_d:=\frac{1}{B\left(\frac12,\frac{d-1}{2}\right)}$. Using $\Gamma(1/2)=\sqrt{\pi}$, we obtain
\[
c_d=\frac{\Gamma(d/2)}{\sqrt{\pi}\,\Gamma((d-1)/2)}.
\]
Finally, using Stirling's formula gives
\[
\frac{\Gamma(x+1/2)}{\Gamma(x)}=x^{1/2}\left(1+O(x^{-1})\right).
\]
Taking $x=(d-1)/2$ yields
\[
c_d=\sqrt{\frac{d-1}{2\pi}}\left(1+O(d^{-1})\right).
\]
so there exists a constant $C>0$ such that
$c_d\le C\sqrt d$ for every $d-1\ge1$. 

Since
\[
\mathbb P_v(v^\top z>u)=c_d\int_{u/\|z\|}^1(1-t^2)^{(d-3)/2}\,dt,
\]
and $t\ge u/\|z\|$ on the domain of integration,
\[
\int_{u/\|z\|}^1(1-t^2)^{(d-3)/2}\,dt
\le\frac{\|z\|}{u}\int_{u/\|z\|}^1t(1-t^2)^{(d-3)/2}\,dt
=\frac{\|z\|}{u(d-1)}\left(1-\frac{u^2}{\|z\|^2}\right)^{(d-1)/2}.
\]
Using $c_d\le C\sqrt d$ and, for $d\ge2$,
\[
\frac{\sqrt d}{d-1}=\frac{1}{\sqrt d}\frac{d}{d-1}\le\frac{2}{\sqrt d},
\]
we obtain
\[
\mathbb P_v(v^\top z>u)\le\frac{C\|z\|}{u\sqrt d}\left(1-\frac{u^2}{\|z\|^2}\right)^{(d-1)/2}.
\]
Finally, since $1-x\le e^{-x}$ for $x\ge0$,
\[
\mathbb P_v(v^\top z>u)\le\frac{C\|z\|}{u\sqrt d}\exp\left[-\frac{d-1}{2}\frac{u^2}{\|z\|^2}\right].
\]
We apply the preceding estimate with $z=\proj_Bx$ and $u=R-\tau$. If
\[
\|\proj_Bx\|\le R-\tau,
\]
then $v^\top\proj_Bx\le\|\proj_Bx\|\le R-\tau$ for every $v\in\mathbb S^{d-1}$, so the probability is zero. Otherwise, since $\proj_Bx\in B$ and $B$ is the Euclidean ball of radius $a$,
\[
\|\proj_Bx\|\le a=(1+R^{-1/2})\sqrt d.
\]
the preceding bound gives
\[
\mathbb P_v\left(v^\top\proj_Bx>R-\tau\right)\le C\exp\left[-\frac{d-1}{d}\frac{(R-\tau)^2}{2(1+R^{-1/2})^2}\right],
\]
after increasing the constant $C$ if necessary. Since $d=\lceil R^4\rceil$,
\[
\frac{d-1}{d}=1+O(R^{-4}).
\]
Moreover, $\tau=(\kappa-1)^{-2}$ is negligible compared with any inverse power of $R$, so
\[
(R-\tau)^2=R^2+o(1).
\]
Using
\[
(1+R^{-1/2})^{-2}=1-2R^{-1/2}+3R^{-1}+O(R^{-3/2}),
\]
we obtain
\[
\frac{d-1}{d}\frac{(R-\tau)^2}{2(1+R^{-1/2})^2}
=\frac{R^2}{2}-R^{3/2}+O(R).
\]
Also,
\[
\frac{R^2}{2}=\log\kappa-\frac32\log\log\kappa+O(1),
\]
so $R^2=O(\log\kappa)$ and hence
\[
R^{3/2}=O((\log\kappa)^{3/4}).
\]
Increasing $C$ once more if necessary therefore gives
\[
\mathbb P_v\left(v^\top\proj_Bx>R-\tau\right)\le\frac{C}{\kappa}\exp\left[C(\log\kappa)^{3/4}\right]
\]
uniformly over $x\in\mathbb R^d$.
\end{proof}

We next compare the algorithm with its trajectory under $V_B$. Until the hidden region is reached, the two trajectories coincide.

\begin{lemma}\label{lem:lb-common-path}
Fix an integer $r\ge1$ and a randomized algorithm making at most $N$ pointwise queries. At each query point, the algorithm observes the value of $V$ and its derivatives up to order $r$. For each fixed realization $\omega$ of the algorithm's internal randomness, let
$x_1^0(\omega),\ldots,x_N^0(\omega)$
denote its queries against $V_B$.  If
\[
v^\top\proj_Bx_t^0\le R-\tau
\]
for every query made by the algorithm, then the oracle responses, query points, and final outputs are identical for $V_v$ and $V_B$. Moreover,
\[
\mathbb P_v\left(\exists t:v^\top\proj_Bx_t^0>R-\tau\right)\le C\frac{N}{\kappa}\exp\left[C(\log\kappa)^{3/4}\right].
\]
\end{lemma}

\begin{proof}
Fix a realization $\omega$ of the algorithm's
internal randomness. The first query is then
the same under $V_v$ and $V_B$. Suppose that the two transcripts agree through query $t-1$. Since the algorithm chooses its next query from the previous transcript, the $t$-th query is also the same and equals $x_t^0$. If
\[
v^\top\proj_Bx_t^0\le R-\tau,
\]
Lemma~\ref{lem:lb-full-jet} implies that $V_v$ and $V_B$ return the same derivatives through order $r$ at $x_t^0$. Hence the transcripts continue to agree. By induction, the two query sequences, oracle responses, and final outputs are identical until a query satisfies
\[
v^\top\proj_Bx_t^0>R-\tau.
\]
The points $x_1^0,\ldots,x_N^0$ are determined by $V_B$ and therefore do not depend on $v$. Applying Lemma~\ref{lem:lb-spherical-cap} to each point and taking a union bound gives
\[
\mathbb P_v\left(\exists t:v^\top\proj_Bx_t^0>R-\tau\right)\le C\frac{N}{\kappa}\exp\left[C(\log\kappa)^{3/4}\right].
\]
Averaging over the algorithm's randomness preserves the same bound.
\end{proof}

We next prove Lemma~\ref{lem:lb-compromise}.

\begin{proof}[Proof of Lemma~\ref{lem:lb-compromise}]
Choose $\varepsilon_0>0$ sufficiently small to satisfy the conditions below, and then take $\kappa$ sufficiently large. All choices will be independent of $q$.

The proof of Proposition~\ref{prop:lb-initial-gap} and Lemma~\ref{lem:lb-full-jet} imply that
\[
\inf_{p\in\mathcal G_d}\left\{\mathrm{KL}(p\,\|\,\gamma)+(\kappa-1)\mathbb E_p[e_{K_v}^\tau(X)]\right\}=o(1)
\]
uniformly over $v$. Hence, if $F_{V_v}(q)-F_{V_v}^\star\le2\varepsilon_0$,
then, for all sufficiently large $\kappa$,
\[
\mathrm{KL}(q\,\|\,\gamma)+(\kappa-1)\mathbb E_q[e_{K_v}^\tau(X)]\le3\varepsilon_0.
\]
Since $e_{K_v}^\tau\ge e_{K_v}\ge e_{H_v}$, we obtain
\[
\mathrm{KL}(q\,\|\,\gamma)+\frac{\kappa-1}{2}\mathbb E_q[(v^\top X-R)_+^2]\le3\varepsilon_0.
\]

Write $q=\mathcal N(m,\Sigma)$, and let $\lambda_1,\ldots,\lambda_d$ be the eigenvalues of $\Sigma$. The Gaussian KL formula yields
\[
\|m\|^2+\sum_{i=1}^d(\lambda_i-\log\lambda_i-1)\le6\varepsilon_0.
\]
Each summand is nonnegative and is bounded away from zero outside $[1/2,2]$. Moreover,
\[
\lambda-\log\lambda-1\ge\frac18(\lambda-1)^2,\qquad \lambda\in[1/2,2].
\]
Thus our choice of sufficiently small $\varepsilon_0$ ensures
\[
\|m\|^2\le6\varepsilon_0,\qquad \frac12I_d\preceq\Sigma\preceq2I_d,
\]
and
\[
\|I_d-\Sigma\|_F^2=\sum_{i=1}^d(\lambda_i-1)^2\le48\varepsilon_0.
\]
Set
\[
\sigma^2:=v^\top\Sigma v,\qquad t:=\frac{R-v^\top m}{\sigma}.
\]
Since $v^\top X\overset{d}=v^\top m+\sigma G, \, G\sim\mathcal N(0,1)$,
we have
\[
\mathbb E_q[(v^\top X-R)_+^2]=\sigma^2\mathbb E[(G-t)_+^2].
\]
Hence the preceding bound gives
\[
\frac{\kappa-1}{2}\sigma^2\mathbb E[(G-t)_+^2]\le3\varepsilon_0.
\]
Using $\frac{\kappa-1}{2}\mathbb E[(G-R)_+^2]=1$,
we obtain
\[
\sigma^2\frac{\mathbb E[(G-t)_+^2]}{\mathbb E[(G-R)_+^2]}\le 3\varepsilon_0.
\]
Since $\sigma^2\in[1/2,2]$ and $|v^\top m|\le\sqrt{6\varepsilon_0}$, we have $t/R\in[1/2,2]$ for all sufficiently large $\kappa$. The two-sided Gaussian tail estimate
\[
c\frac{e^{-u^2/2}}{u^3}\le\mathbb E[(G-u)_+^2]\le C\frac{e^{-u^2/2}}{u^3},\qquad u\ge1,
\]
therefore gives
\[
c\exp\left[-\frac{t^2-R^2}{2}\right]\le\sigma^2\frac{\mathbb E[(G-t)_+^2]}{\mathbb E[(G-R)_+^2]}\le3\varepsilon_0.
\]
Choosing $\varepsilon_0$ sufficiently small ensures $t^2-R^2\ge2$. We have
\[
\sigma^2(t^2-R^2)=(R-v^\top m)^2-\sigma^2R^2.
\]
Expanding the square and using $\|v\|=1$ gives
\[
\sigma^2(t^2-R^2)=R^2v^\top(I_d-\Sigma)v-2Rv^\top m+(v^\top m)^2.
\]
Hence
\[
-2Rv^\top m+R^2v^\top(I_d-\Sigma)v=\sigma^2(t^2-R^2)-(v^\top m)^2.
\]
Using $t^2-R^2\ge2$, $\sigma^2\ge1/2$, and $(v^\top m)^2\le\|m\|^2\le6\varepsilon_0$, we obtain
\[
-2Rv^\top m+R^2v^\top(I_d-\Sigma)v\ge1-6\varepsilon_0 \ge \frac12
\]
by choosing $\varepsilon_0\le1/12$. Thus, if the objective gap is at most $2\varepsilon_0$, then at least one of $v^\top(-m)\ge\frac{1}{8R}$ or $v^\top(I_d-\Sigma)v\ge\frac{1}{4R^2}$ must hold. We show that each event has probability at most $e^{-cR^2}$ over the random direction $v$.

For the first event, if $\|m\|<1/(8R)$, then it is impossible. Otherwise, Lemma~\ref{lem:lb-spherical-cap} with $z=-m$ and $u=1/(8R)$ gives
\[
\mathbb P_v\left(v^\top(-m)\ge\frac{1}{8R}\right)\le\frac{CR\|m\|}{\sqrt d}\exp\left[-c\frac{d}{R^2\|m\|^2}\right].
\]
Since $\|m\|^2\le6\varepsilon_0$ and $d=\lceil R^4\rceil$,
\[
\frac{d}{R^2\|m\|^2}\ge cR^2,
\]
while $\frac{CR\|m\|}{\sqrt d}$ is also bounded for all sufficiently large $\kappa$. Hence
\[
\mathbb P_v\left(v^\top(-m)\ge\frac{1}{8R}\right)\le Ce^{-cR^2}.
\]

For the second event, write $v=\frac{g}{\|g\|}, \, g\sim\mathcal N(0,I_d).$ On the event $\{\|g\|^2\ge d/2\}$,
\[
v^\top(I_d-\Sigma)v\ge\frac{1}{4R^2}
\]
implies
\[
g^\top(I_d-\Sigma)g\ge\frac{d}{8R^2}.
\]
Moreover,
\[
|\operatorname{tr}(I_d-\Sigma)|\le\sqrt d\,\|I_d-\Sigma\|_F\le \sqrt{48\varepsilon_0}\sqrt d.
\]
Since $d=\lceil R^4\rceil$, by choosing $\varepsilon_0$ sufficiently small, we get $|\operatorname{tr}(I_d-\Sigma)|\le\frac{d}{16R^2}$. Therefore,
\[
g^\top(I_d-\Sigma)g-\operatorname{tr}(I_d-\Sigma)\ge\frac{d}{16R^2}.
\]
Now the Hanson--Wright inequality \citep{hanson1971bound} gives
\[
\mathbb P\left(g^\top(I_d-\Sigma)g-\operatorname{tr}(I_d-\Sigma)\ge u\right)\le2\exp\left[-c\min\left\{\frac{u^2}{\|I_d-\Sigma\|_F^2},\frac{u}{\|I_d-\Sigma\|_{\mathrm{op}}}\right\}\right].
\]
Using
\[
\|I_d-\Sigma\|_F\le \sqrt{48\varepsilon_0},\qquad \|I_d-\Sigma\|_{\mathrm{op}}\le \sqrt{48\varepsilon_0},
\]
and taking $u=d/(16R^2)$, together with $d=\lceil R^4\rceil$, gives
\[
\mathbb P\left(g^\top(I_d-\Sigma)g-\operatorname{tr}(I_d-\Sigma)\ge\frac{d}{16R^2}\right)\le2e^{-cR^2}
\]
for a constant $c>0$. Moreover, $\|g\|^2\sim\chi_d^2$, and the standard chi-square lower-tail bound gives
\[
\mathbb P(\|g\|^2<d/2)\le e^{-cd}\le e^{-cR^2}
\]
for all sufficiently large $\kappa$. The last inequality holds since $d\ge R^4$. So, the event
\[
F_{V_v}(q)-F_{V_v}^\star\le2\varepsilon_0
\] 
is therefore contained in the union of these two events. Hence
\[
\mathbb P_v\left(F_{V_v}(q)-F_{V_v}^\star\le2\varepsilon_0\right)\le Ce^{-cR^2}.
\]
Since $C$ is fixed, decreasing $c>0$ and increasing the threshold on $\kappa$ if necessary gives
\[
\mathbb P_v\left(F_{V_v}(q)-F_{V_v}^\star\le2\varepsilon_0\right)\le e^{-cR^2}.
\]
\end{proof}

We now combine Lemmas~\ref{lem:lb-common-path} and~\ref{lem:lb-compromise} to prove the main theorem.

\begin{proof}[Proof of Theorem~\ref{thm:lb-main}]
Fix the constants $\varepsilon_0,c>0$ supplied by Lemma~\ref{lem:lb-compromise}, and let $C_1\ge1$ be a fixed constant for which the probability estimate in Lemma~\ref{lem:lb-common-path} holds. We first establish the required bounds for all sufficiently large $\kappa$ and then choose a single threshold $\kappa_0$.

Fix an integer $r\ge1$ and a randomized order-$r$ algorithm that makes at most $N$ queries. For a fixed algorithm randomness $\omega$, let $x_1^0,\ldots,x_N^0$ denote its queries against $V_B$, and let $q^0(\omega)$ be the corresponding output. By Lemma~\ref{lem:lb-common-path},
\[
\mathbb P_v\left(\exists t:v^\top\proj_Bx_t^0>R-\tau\right)\le C_1\frac{N}{\kappa}\exp\left[C_1(\log\kappa)^{3/4}\right].
\]
Run the algorithm against $V_v$ using the same $\omega$. If none of the reference queries satisfies $v^\top\proj_Bx_t^0>R-\tau$, the two transcripts and outputs coincide, so the actual output is $q=q^0(\omega)$.

For each fixed $\omega$, the reference output $q^0(\omega)$ does not depend on $v$. Lemma~\ref{lem:lb-compromise} therefore gives
\[
\mathbb P_v\left(F_{V_v}(q^0(\omega))-F_{V_v}^\star\le2\varepsilon_0\right)\le e^{-cR^2}.
\]
Consequently, for each fixed $\omega$,
\[
\left\{ F_{V_v}(q)-F_{V_v}^\star\le2\varepsilon_0 \right\}
\subseteq
\left\{\exists t:\, v^\top\Pi_B x_t^0(\omega)>R-\tau \right\} \cup
\left\{ F_{V_v}(q^0(\omega))-F_{V_v}^\star
\le2\varepsilon_0 \right\}.
\]
Taking a union bound and averaging over the
algorithm's randomness yields
\[
\mathbb P_{v,\mathrm{alg}}\left(F_{V_v}(q)-F_{V_v}^\star\le2\varepsilon_0\right)\le C_1\frac{N}{\kappa}\exp\left[C_1(\log\kappa)^{3/4}\right]+e^{-cR^2}.
\]

Now suppose that the algorithm satisfies
\[
\sup_{v\in\mathbb S^{d-1}}\mathbb E\left[F_{V_v}(q)-F_{V_v}^\star\right]\le\varepsilon_0.
\]
Averaging over $v$ gives
\[
\mathbb E_{v,\mathrm{alg}}\left[F_{V_v}(q)-F_{V_v}^\star\right]\le\varepsilon_0.
\]
Since the objective gap is nonnegative, Markov's inequality implies
\[
\mathbb P_{v,\mathrm{alg}}\left(F_{V_v}(q)-F_{V_v}^\star\le2\varepsilon_0\right)\ge\frac12.
\]
Combining the two probability bounds gives
\[
\frac12\le C_1\frac{N}{\kappa}\exp\left[C_1(\log\kappa)^{3/4}\right]+e^{-cR^2}.
\]
Because $R\to\infty$, the last term is at most $1/4$ for all sufficiently large $\kappa$. Hence
\[
C_1\frac{N}{\kappa}\exp\left[C_1(\log\kappa)^{3/4}\right]\ge\frac14,
\]
and therefore
\[
N\ge\frac{\kappa}{4C_1}\exp\left[-C_1(\log\kappa)^{3/4}\right].
\]

To remove the fixed multiplicative factor, choose the constant in the theorem as $C:=C_1+1$. Since for all sufficiently large $\kappa$,
\[
\exp\left[-(\log\kappa)^{3/4}\right]\le\frac{1}{4C_1}.
\]
Consequently,
\[
N\ge\frac{\kappa}{4C_1}\exp\left[-C_1(\log\kappa)^{3/4}\right]\ge\kappa\exp\left[-C(\log\kappa)^{3/4}\right].
\]

It remains to verify the dimension and initial-gap requirements. The construction gives
\[
d=\Theta((\log\kappa)^2),
\]
and Lemma~\ref{lem:lb-full-jet} gives
\[
V_v\in C^\infty(\mathbb R^d),\qquad I_d\preceq\nabla^2V_v(x)\preceq\kappa I_d
\]
for every $v\in\mathbb S^{d-1}$ and $x\in\mathbb R^d$. Thus the entire family belongs to $\mathcal V_d(\kappa)$. The same lemma also gives
\[
F_{V_v}(\gamma)-F_{V_v}^\star=1+o(1)
\]
uniformly over $v\in\mathbb S^{d-1}$. Fix $C_0:=2$. Then, for all sufficiently large $\kappa$,
\[
\sup_{v\in\mathbb S^{d-1}}\left(F_{V_v}(\gamma)-F_{V_v}^\star\right)\le C_0.
\]

We have now fixed $\varepsilon_0$, $C_0$, and $C$. Choose $\kappa_0>1$ larger than the finitely many thresholds needed for the construction and the preceding lemmas, and this proves the theorem.
\end{proof}

\section{Proofs for Section~\ref{sec:stochastic}}\label{app:stochastic}

We first analyze the Cholesky finite-difference construction and its variants, and then turn to the symmetric-root score estimators.

\subsection{Cholesky proximal acceleration}\label{subsec:fd-prox}

We first establish the unbiasedness and variance bound of the finite-difference estimator used in Algorithm~\ref{alg:fd-prox}.

\begin{proposition}\label{prop:fd-estimator-full}
For every $(m,L)\in\mathbb{R}^d\times\mathbb{T}^d$,
\[
\mathbb{E}[G^{\mathrm{FD}}(m,L;\xi)]=\nabla E(m,L).
\]
Under Assumption~\ref{assump:hessian-lipschitz},
\[
\mathbb{E}\|G^{\mathrm{FD}}(m,L;\xi)-\nabla E(m,L)\|^2\le\frac{L_H^2}{2}\left(3(d+5)\|L\|_F^4+\nu^2d(d+3)\|L\|_F^2\right).
\]
If $V$ is quadratic, then $G^{\mathrm{FD}}(m,L;\xi)=\nabla E(m,L)$ for every $\xi$.
\end{proposition}

\begin{proof}
Fix $(m,L)$ and write $D=D_\nu(m)$ and $r=r_\nu(m,L;\xi)$. Since $\mathbb{E}\xi=0$ and $\mathbb{E}[\xi\xi^\top]=I_d$,
\[
\mathbb{E}[\nabla V(m)+r]=\mathbb{E}[\nabla V(m+L\xi)],\qquad DL+\mathbb{E}[r\xi^\top]=\mathbb{E}[\nabla V(m+L\xi)\xi^\top].
\]
Taking the lower-triangular part proves unbiasedness. Let $H=\nabla^2 V(m)$. For each $j$,
\[
D_\nu(m;e_j)-He_j=\int_0^1[\nabla^2V(m+t\nu e_j)-H]e_j\,dt.
\]
Therefore
\[
\|D-H\|_{\mathrm{op}}\le\|D-H\|_F\le\frac{L_H|\nu|\sqrt d}{2}.
\]
Taylor's theorem gives
\[
\|\nabla V(m+L\xi)-\nabla V(m)-HL\xi\|\le\frac{L_H}{2}\|L\xi\|^2,
\]
and hence
\[
\|r\|\le\frac{L_H}{2}\|L\xi\|^2+\frac{L_H|\nu|\sqrt d}{2}\|L\xi\|.
\]
Since $\mathbb{E}G^{\mathrm{FD}}=\nabla E$, the centered estimator is
\[
G^{\mathrm{FD}}-\nabla E=\left(r-\mathbb{E}r,\operatorname{low}\left(r\xi^\top-\mathbb{E}[r\xi^\top]\right)\right).
\]
For every square-integrable $Z$,
\[
\mathbb{E}\|Z-\mathbb{E}Z\|^2\le\mathbb{E}\|Z\|^2.
\]
Moreover, $\operatorname{low}(\cdot)$ is a Frobenius contraction and $\|r\xi^\top\|_F^2=\|r\|^2\|\xi\|^2$. Hence
\[
\mathbb{E}\|G^{\mathrm{FD}}-\nabla E\|^2\le\mathbb{E}[(1+\|\xi\|^2)\|r\|^2].
\]
Using $(a+b)^2\le2a^2+2b^2$,
\[
\|r\|^2\le\frac{L_H^2}{2}\left(\|L\xi\|^4+\nu^2d\|L\xi\|^2\right).
\]

Let $A=L^\top L\succeq0$. Then $\|L\xi\|^2=\xi^\top A\xi$ and $\operatorname{tr}(A)=\|L\|_F^2$. The Gaussian identities
\[
\mathbb{E}[\xi_i\xi_j]=\delta_{ij},\qquad \mathbb{E}[\xi_i\xi_j\xi_k\xi_\ell]=\delta_{ij}\delta_{k\ell}+\delta_{ik}\delta_{j\ell}+\delta_{i\ell}\delta_{jk}
\]
give
\[
\mathbb{E}[\xi^\top A\xi]=\operatorname{tr}(A),\qquad \mathbb{E}[(\xi^\top A\xi)^2]=(\operatorname{tr}A)^2+2\operatorname{tr}(A^2)\le3(\operatorname{tr}A)^2
\]
and
\[
\mathbb{E}[\|\xi\|^2(\xi^\top A\xi)]=(d+2)\operatorname{tr}(A).
\]
Consequently,
\[
\mathbb{E}[(1+\|\xi\|^2)\|L\xi\|^2]=(d+3)\|L\|_F^2.
\]
Writing $\xi=Ru$ with $R=\|\xi\|$ and using $\mathbb{E}[R^6]=(d+4)\mathbb{E}[R^4]$ gives
\[
\mathbb{E}[\|\xi\|^2\|L\xi\|^4]=(d+4)\mathbb{E}\|L\xi\|^4\le3(d+4)\|L\|_F^4.
\]
Thus
\[
\mathbb{E}[(1+\|\xi\|^2)\|L\xi\|^4]\le3(d+5)\|L\|_F^4.
\]
Combining the estimates yields
\[
\mathbb{E}\|G^{\mathrm{FD}}-\nabla E\|^2\le\frac{L_H^2}{2}\left(3(d+5)\|L\|_F^4+\nu^2d(d+3)\|L\|_F^2\right).
\]

If $V$ is quadratic with Hessian $H$, then $D_\nu(m;e_j)=He_j$ for every $j$ and $\nabla V(m+L\xi)=\nabla V(m)+HL\xi$. Thus $D=H$, $r=0$, and $G^{\mathrm{FD}}(m,L;\xi)=\nabla E(m,L)$ for every $\xi$.
\end{proof}

We next derive the constrained proximal map used in the Cholesky update.

\begin{proposition}\label{prop:fd-prox-map}
For $\eta>0$ and $Y\in\mathbb{T}^d$, define $L^{(\gamma)}$ for $\gamma\ge0$ by
\[
L^{(\gamma)}_{ij}:=\frac{Y_{ij}}{1+\gamma}\quad(i>j),\qquad L^{(\gamma)}_{ii}:=\frac{Y_{ii}+\sqrt{Y_{ii}^2+4\eta(1+\gamma)}}{2(1+\gamma)}.
\]
If $\|L^{(0)}\|_F\le\sqrt{d/\alpha}$, then
$\prox_{\eta\bar h}(Y)=L^{(0)}$. Otherwise, there is a unique $\gamma>0$ satisfying $\|L^{(\gamma)}\|_F^2=d/\alpha$ and $\prox_{\eta\bar h}(Y)=L^{(\gamma)}$.
\end{proposition}

\begin{proof}
The proximal subproblem is
\[
\min_{L\in\mathbb{T}_{++}^d}\left\{-\eta\log\det L+\frac12\|L-Y\|_F^2:\|L\|_F^2\le\frac{d}{\alpha}\right\}.
\]
The objective is strictly convex and diverges as any diagonal entry approaches zero, so the minimizer is unique and belongs to $\mathbb T_{++}^d$. Let $\gamma\ge0$ be the multiplier for the norm constraint. The KKT conditions are
\[
(1+\gamma)L_{ij}=Y_{ij}\quad(i>j),\qquad (1+\gamma)L_{ii}-Y_{ii}-\frac{\eta}{L_{ii}}=0.
\]
The diagonal equation has a unique positive root, giving $L=L^{(\gamma)}$. If $\|L^{(0)}\|_F^2\le d/\alpha$, complementary slackness gives $\gamma=0$. Otherwise, the norm constraint is active and $\|L^{(\gamma)}\|_F^2=d/\alpha$. For $i>j$, $|L_{ij}^{(\gamma)}|=|Y_{ij}|/(1+\gamma)$ is nonincreasing in $\gamma$. For each diagonal entry, differentiating the stationarity equation gives
\[
\frac{dL_{ii}^{(\gamma)}}{d\gamma}=-\frac{L_{ii}^{(\gamma)}}{1+\gamma+\eta/(L_{ii}^{(\gamma)})^2}<0.
\]
Hence $\gamma\mapsto\|L^{(\gamma)}\|_F$ is continuous and strictly decreasing. Moreover,
\[
\lim_{\gamma\to\infty}\|L^{(\gamma)}\|_F=0.
\]
Since $\|L^{(0)}\|_F>\sqrt{d/\alpha}$, there is a unique $\gamma>0$ satisfying the constraint.
\end{proof}

We now prove the fixed-batch proximal acceleration result used in the complexity analysis.

\begin{lemma}\label{lem:prox-fixed-batch}
Let $\Phi=f+h$, where $f$ is $\alpha$-strongly convex and $\beta$-smooth and $h$ is proper, closed, and convex. Let $x_\star$ minimize $\Phi$, write $\Phi^\star:=\Phi(x_\star)$ and $\kappa:=\beta/\alpha$, and set
\[
s:=\sqrt{\max\{\kappa,4\}},\qquad \eta:=\frac{1}{\alpha s^2},\qquad \rho:=\frac{1-s^{-1}}{1+s^{-1}}.
\]
Suppose $x_0=y_0$ and
\[
x_{k+1}=\prox_{\eta h}(y_k-\eta g_k),\qquad y_{k+1}=x_{k+1}+\rho(x_{k+1}-x_k),
\]
where, conditionally on the past $\mathcal F_k$,
\[
\mathbb E[g_k\mid\mathcal F_k]=\nabla f(y_k),\qquad \mathbb E[\|g_k-\nabla f(y_k)\|^2\mid\mathcal F_k]\le\sigma^2.
\]
Then, for every integer $T\ge0$,
\[
\mathbb E[\Phi(x_T)-\Phi^\star]\le e^{-T/s}\left(\Phi(x_0)-\Phi^\star+\frac{\alpha}{2}\|x_0-x_\star\|^2\right)+\frac{\sigma^2}{\alpha s}.
\]
\end{lemma}

\begin{proof}
Since $s\ge2$ and $s^2\ge\kappa=\beta/\alpha$, we have $\beta\eta\le1$. Proximal optimality gives
\[
p_k:=\frac{y_k-x_{k+1}}{\eta}-g_k\in\partial h(x_{k+1}),\qquad \widetilde g_k:=g_k+p_k=\frac{y_k-x_{k+1}}{\eta}.
\]
Define
\[
\Lambda_k(z):=f(y_k)+\langle g_k,z-y_k\rangle+\frac{\alpha}{2}\|z-y_k\|^2+h(x_{k+1})+\langle p_k,z-x_{k+1}\rangle.
\]
By strong convexity of $f$, convexity of $h$, and conditional unbiasedness,
\[
\mathbb E[\Lambda_k(z)\mid\mathcal F_k]\le\Phi(z)
\]
for every $\mathcal F_k$-measurable $z$. Let
\[
\bar x_k:=\prox_{\eta h}(y_k-\eta\nabla f(y_k)).
\]
Decomposing $x_{k+1}-y_k=(x_{k+1}-\bar x_k)+(\bar x_k-y_k)$ and using conditional centering gives
\[
\mathbb E[\langle\nabla f(y_k)-g_k,x_{k+1}-y_k\rangle\mid\mathcal F_k]=\mathbb E[\langle\nabla f(y_k)-g_k,x_{k+1}-\bar x_k\rangle\mid\mathcal F_k].
\]
Nonexpansiveness of the proximal map gives
\[
\|x_{k+1}-\bar x_k\|\le\eta\|g_k-\nabla f(y_k)\|,
\]
and therefore
\[
\mathbb E[\langle\nabla f(y_k)-g_k,x_{k+1}-y_k\rangle\mid\mathcal F_k]\le\eta\sigma^2.
\]
By $\beta$-smoothness,
\[
\Phi(x_{k+1})\le f(y_k)+h(x_{k+1})+\langle\nabla f(y_k),x_{k+1}-y_k\rangle+\frac{\beta}{2}\|x_{k+1}-y_k\|^2.
\]
Since
\[
\Lambda_k(y_k)=f(y_k)+h(x_{k+1})+\langle p_k,y_k-x_{k+1}\rangle,
\]
we obtain
\[
\Phi(x_{k+1})\le\Lambda_k(y_k)+\langle\widetilde g_k,x_{k+1}-y_k\rangle+\langle\nabla f(y_k)-g_k,x_{k+1}-y_k\rangle+\frac{\beta}{2}\|x_{k+1}-y_k\|^2.
\]
Using $x_{k+1}-y_k=-\eta\widetilde g_k$ and taking conditional expectations,
\[
\mathbb E[\Phi(x_{k+1})\mid\mathcal F_k]\le\mathbb E[\Lambda_k(y_k)\mid\mathcal F_k]+\left(\frac{\beta\eta^2}{2}-\eta\right)\mathbb E[\|\widetilde g_k\|^2\mid\mathcal F_k]+\eta\sigma^2.
\]

Define
\[
\varphi_0(z):=\Phi(x_0)+\frac{\alpha}{2}\|z-x_0\|^2,\qquad \varphi_{k+1}:=(1-s^{-1})\varphi_k+s^{-1}\Lambda_k.
\]
Let
\[
v_k:=\arg\min_z\varphi_k(z),\qquad \varphi_k^{\min}:=\varphi_k(v_k).
\]
Every $\varphi_k$ is a quadratic with Hessian $\alpha I$, so
\[
\varphi_k(z)=\varphi_k^{\min}+\frac{\alpha}{2}\|z-v_k\|^2.
\]
The stationarity condition for $\varphi_{k+1}$ gives
\[
v_{k+1}=(1-s^{-1})v_k+s^{-1}y_k-\frac{1}{\alpha s}\widetilde g_k.
\]
We next establish
\[
y_k=\frac{x_k+s^{-1}v_k}{1+s^{-1}}.
\]
It holds at $k=0$. Assuming it at iteration $k$ and using $x_{k+1}=y_k-\eta\widetilde g_k$ gives
\[
s^{-1}v_{k+1}=x_{k+1}-(1-s^{-1})x_k,
\]
which is equivalent to the momentum update for $y_{k+1}$. Set
\[
\zeta_0:=0,\qquad \zeta_{k+1}:=(1-s^{-1})\zeta_k+\eta\sigma^2.
\]
We prove by induction that
\[
\mathbb E[\varphi_k^{\min}]\ge\mathbb E[\Phi(x_k)]-\zeta_k.
\]
The claim holds with equality at $k=0$. Since
\[
x_k-y_k=-s^{-1}(v_k-y_k),
\]
the lower-model property at $x_k$ gives
\[
\Lambda_k(x_k)=\Lambda_k(y_k)-s^{-1}\langle\widetilde g_k,v_k-y_k\rangle+\frac{\alpha}{2s^2}\|v_k-y_k\|^2.
\]
Completing the square in $\varphi_{k+1}$ gives
\[
\varphi_{k+1}^{\min}=(1-s^{-1})\varphi_k^{\min}+s^{-1}\Lambda_k(y_k)+(1-s^{-1})s^{-1}\left(\langle\widetilde g_k,v_k-y_k\rangle+\frac{\alpha}{2}\|v_k-y_k\|^2\right)-\frac{1}{2\alpha s^2}\|\widetilde g_k\|^2.
\]
Using the induction hypothesis and the lower-model bound at $x_k$, the inner-product terms cancel and yield
\[
\mathbb E[\varphi_{k+1}^{\min}]\ge\mathbb E[\Lambda_k(y_k)]-(1-s^{-1})\zeta_k-\frac{\eta}{2}\mathbb E\|\widetilde g_k\|^2+\frac{\alpha}{2}(1-s^{-1})s^{-1}(1+s^{-1})\mathbb E\|v_k-y_k\|^2.
\]
Dropping the final nonnegative term,
\[
\mathbb E[\varphi_{k+1}^{\min}]\ge\mathbb E[\Lambda_k(y_k)]-(1-s^{-1})\zeta_k-\frac{\eta}{2}\mathbb E\|\widetilde g_k\|^2.
\]
The one-step upper bound implies
\[
\mathbb E[\Lambda_k(y_k)]\ge\mathbb E[\Phi(x_{k+1})]+\left(\eta-\frac{\beta\eta^2}{2}\right)\mathbb E\|\widetilde g_k\|^2-\eta\sigma^2.
\]
Therefore
\[
\mathbb E[\varphi_{k+1}^{\min}]\ge\mathbb E[\Phi(x_{k+1})]-\zeta_{k+1}+\frac{\eta}{2}(1-\beta\eta)\mathbb E\|\widetilde g_k\|^2.
\]
Since $\beta\eta\le1$, the induction closes. Finally, the lower-model property at $x_\star$ gives
\[
\mathbb E[\varphi_T(x_\star)]-\Phi^\star\le(1-s^{-1})^T\left(\varphi_0(x_\star)-\Phi^\star\right).
\]
Using the bound established above,
\[
\mathbb E[\varphi_T^{\min}]\ge\mathbb E[\Phi(x_T)]-\zeta_T,
\]
together with $\varphi_T^{\min}\le\varphi_T(x_\star)$, we obtain
\[
\mathbb E[\Phi(x_T)-\Phi^\star] \le (1-s^{-1})^T \left( \Phi(x_0)-\Phi^\star+\frac{\alpha}{2}\|x_0-x_\star\|^2 \right) +\zeta_T.
\]
Moreover,
\[
\zeta_T= \eta\sigma^2\sum_{j=0}^{T-1}(1-s^{-1})^j \le \eta s\sigma^2 = \frac{\sigma^2}{\alpha s}.
\]
Using $(1-s^{-1})^T\le e^{-T/s}$ completes the proof.
\end{proof}

We can now combine the estimator bound with the  acceleration lemma.

\begin{theorem}[Full version of Theorem~\ref{thm:fd-prox}]\label{thm:fd-prox-full}
Suppose Assumptions~\ref{assump:curvature} and~\ref{assump:hessian-lipschitz} hold, $x_0\in\operatorname{dom}\bar h$, $0<|\nu|\le\alpha^{-1/2}$, and $0<\varepsilon\le\Delta_0$, where
\[
\Delta_0:=F(q_0)-F^\star+\frac{\alpha}{2}\|x_0-x_\star\|^2.
\]
Let
\[
s:=\sqrt{\max\{\kappa,4\}},\qquad b:=\left\lceil\max\left\{1,\frac{1494\lambda_H^2d^3}{s\varepsilon}\right\}\right\rceil,\qquad T_\varepsilon:=\left\lceil s\log\frac{2\Delta_0}{\varepsilon}\right\rceil.
\]
Choose
\[
\eta:=\frac{1}{\alpha s^2},\qquad \rho:=\frac{s-1}{s+1}.
\]
Then Algorithm~\ref{alg:fd-prox} with $G_b^{\mathrm{FD}}$ satisfies
\[
\mathbb E[F(q_{T_\varepsilon})-F^\star]\le\varepsilon.
\]
The exact first-order oracle count is
\[
N_{\nabla V}=T_\varepsilon(d+1+b),
\]
and consequently
\[
T_\varepsilon=O\left(\sqrt\kappa\log\frac{2\Delta_0}{\varepsilon}\right),\qquad N_{\nabla V}=O\left(\left(d\sqrt\kappa+\frac{\lambda_H^2d^3}{\varepsilon}\right)\log\frac{2\Delta_0}{\varepsilon}\right).
\]
\end{theorem}

\begin{proof}
Write $x_k=(m_k,L_k)$ and $y_k=(m_k^y,L_k^y)$. Every proximal iterate satisfies $\|L_k\|_F\le\sqrt{d/\alpha}$.
Since $L_k^y=(1+\rho)L_k-\rho L_{k-1}$ for $k\ge1$ and $0\le\rho<1$, $\|L_k^y\|_F<3\sqrt{d/\alpha}$,
and hence
\[
\|L_k^y\|_F^2<\frac{9d}{\alpha},\qquad \|L_k^y\|_F^4<\frac{81d^2}{\alpha^2}.
\]
Proposition~\ref{prop:fd-estimator-full} and $|\nu|\le\alpha^{-1/2}$ give
\[
\mathbb E[G^{\mathrm{FD}}(m_k^y,L_k^y;\xi)\mid\mathcal F_k]=\nabla E(m_k^y,L_k^y)
\]
and
\[
\mathbb E[\|G^{\mathrm{FD}}(m_k^y,L_k^y;\xi)-\nabla E(m_k^y,L_k^y)\|^2\mid\mathcal F_k]\le\frac{L_H^2}{\alpha^2}(126d+621)d^2.
\]
Let
\[
\sigma^2:=\frac{L_H^2}{\alpha^2}(126d+621)d^2.
\]
For $G_b^{\mathrm{FD}}$,
\[
\mathbb E[G_b^{\mathrm{FD}}(y_k)\mid\mathcal F_k]=\nabla E(y_k),\qquad \mathbb E[\|G_b^{\mathrm{FD}}(y_k)-\nabla E(y_k)\|^2\mid\mathcal F_k]\le\frac{\sigma^2}{b}.
\]
Using $\lambda_H=L_H/\alpha^{3/2}$,
\[
\frac{\sigma^2}{\alpha}=(126d+621)\lambda_H^2d^2\le747\lambda_H^2d^3.
\]

Define $\Phi(m,L):=E(m,L)+\bar h(L)$. The norm restriction in $\bar h$ does not change the optimizer, and every accepted iterate lies in its domain. Therefore
\[
\Phi(x_\star)=F^\star,\qquad \Phi(x_k)=F(q_k).
\]
Applying Lemma~\ref{lem:prox-fixed-batch} with variance parameter $\sigma^2/b$ gives
\[
\mathbb E[F(q_{T_\varepsilon})-F^\star]\le e^{-T_\varepsilon/s}\Delta_0+\frac{\sigma^2}{\alpha bs}\le e^{-T_\varepsilon/s}\Delta_0+\frac{747\lambda_H^2d^3}{bs}.
\]
By construction,
\[
e^{-T_\varepsilon/s}\Delta_0\le\frac{\varepsilon}{2},\qquad \frac{747\lambda_H^2d^3}{bs}\le\frac{\varepsilon}{2}.
\]
Thus the desired accuracy holds. Each update constructs $D_\nu(m_k^y)$ using $d+1$ gradient evaluations and uses $b$ additional evaluations for the batch. Hence
\[
N_{\nabla V}=T_\varepsilon(d+1+b),
\]
which yields the stated complexity.
\end{proof}

The same accelerated update also admits a second-order estimator with improved dimension dependence.

\begin{proposition}\label{prop:hessian-assisted-final}
Under Assumption~\ref{assump:hessian-lipschitz}, for
\[
G^{\mathrm{paired}}(m,L;\xi,\zeta):=\left(\nabla V(m+L\xi)-\nabla^2V(m+L\zeta)L\xi,\operatorname{low}(\nabla^2V(m+L\zeta)L)\right),
\]
where $\xi,\zeta\sim\mathcal N(0,I_d)$ are independent,
\[
\mathbb{E}[G^{\mathrm{paired}}(m,L;\xi,\zeta)]=\nabla E(m,L),
\]
and
\[
\mathbb{E}\|G^{\mathrm{paired}}(m,L;\xi,\zeta)-\nabla E(m,L)\|^2\le\frac92L_H^2\|L\|_F^4.
\]
If $V$ is quadratic, then $G^{\mathrm{paired}}(m,L;\xi,\zeta)=\nabla E(m,L)$ for every $\xi,\zeta$.
\end{proposition}

\begin{proof}
Let
\[
X:=m+L\xi,\qquad Y:=m+L\zeta,\qquad H_Y:=\nabla^2V(Y).
\]
Since $\xi$ and $\zeta$ are independent and $\mathbb E\xi=0$, we have $\mathbb E[H_YL\xi]=0$. So, the mean component has expectation $\nabla_mE(m,L)$. For the covariance-factor component, Stein's identity gives
\[
\mathbb E[\nabla V(m+L\zeta)\zeta^\top]=\mathbb E[\nabla^2V(m+L\zeta)L],
\]
and therefore
\[
\operatorname{low}\left(\mathbb E[H_YL]\right)=\nabla_LE(m,L).
\]
This proves unbiasedness. Let
\[
H:=\nabla^2V(m),\qquad z:=L\xi,\qquad y:=L\zeta.
\]
Since the mean minimizes mean-square distance,
\[
\mathbb E\|G^{\mathrm{paired}}-\nabla E\|^2\le\mathbb E\left\|G^{\mathrm{paired}}-\left(\nabla V(m),\operatorname{low}(HL)\right)\right\|^2.
\]
For the first component,
\[
\nabla V(m+z)-\nabla V(m)-\nabla^2V(m+y)z=\left(\nabla V(m+z)-\nabla V(m)-Hz\right)+\left(H-\nabla^2V(m+y)\right)z.
\]
Hessian Lipschitzness gives
\[
\|\nabla V(m+z)-\nabla V(m)-Hz\|\le\frac{L_H}{2}\|z\|^2
\]
and
\[
\|(H-\nabla^2V(m+y))z\|\le L_H\|y\|\|z\|.
\]
Thus
\[
\|\nabla V(m+z)-\nabla V(m)-\nabla^2V(m+y)z\|^2\le\frac{L_H^2}{2}\|z\|^4+2L_H^2\|y\|^2\|z\|^2.
\]
Independence gives $\mathbb E[\|y\|^2\|z\|^2]=\|L\|_F^4$,
while $\mathbb E\|L\xi\|^4\le3\|L\|_F^4$. Hence
\[
\mathbb E\|\nabla V(X)-H_YL\xi-\nabla V(m)\|^2\le\frac72L_H^2\|L\|_F^4.
\]
For the covariance-factor component,
\[
\|\operatorname{low}((H_Y-H)L)\|_F\le\|(H_Y-H)L\|_F\le L_H\|y\|\|L\|_F,
\]
and therefore
\[
\mathbb E\|\operatorname{low}(H_YL)-\operatorname{low}(HL)\|_F^2\le L_H^2\|L\|_F^4.
\]
Adding the two components gives the result. If $V$ is quadratic, its Hessian is constant and both errors vanish.
\end{proof}

Combining this variance bound with Lemma~\ref{lem:prox-fixed-batch} gives the second-order complexity.

\begin{corollary}\label{cor:hessian-assisted-fixed-batch}
Suppose Assumptions~\ref{assump:curvature} and~\ref{assump:hessian-lipschitz} hold, $x_0\in\operatorname{dom}\bar h$, and $0<\varepsilon\le\Delta_0$, where
\[
\Delta_0:=F(q_0)-F^\star+\frac{\alpha}{2}\|x_0-x_\star\|^2.
\]
Let
\[
s:=\sqrt{\max\{\kappa,4\}},\qquad b:=\left\lceil\max\left\{1,\frac{729\lambda_H^2d^2}{s\varepsilon}\right\}\right\rceil,\qquad T_\varepsilon:=\left\lceil s\log\frac{2\Delta_0}{\varepsilon}\right\rceil.
\]
Choose
\[
\eta:=\frac{1}{\alpha s^2},\qquad \rho:=\frac{s-1}{s+1}.
\]
Run the update of Algorithm~\ref{alg:fd-prox} with $G_b^{\mathrm{paired}}$ in place of $G_b^{\mathrm{FD}}$. Then
\[
\mathbb E[F(q_{T_\varepsilon})-F^\star]\le\varepsilon.
\]
The exact oracle counts are
\[
N_{\nabla V}=bT_\varepsilon,\qquad N_{\mathrm{HVP}}=dbT_\varepsilon.
\]
Consequently,
\[
N_{\nabla V}=O\left(\left(\sqrt\kappa+\frac{\lambda_H^2d^2}{\varepsilon}\right)\log\frac{2\Delta_0}{\varepsilon}\right),
\]
and
\[
N_{\mathrm{HVP}}=O\left(\left(d\sqrt\kappa+\frac{\lambda_H^2d^3}{\varepsilon}\right)\log\frac{2\Delta_0}{\varepsilon}\right).
\]
\end{corollary}

\begin{proof}
As in the proof of Theorem~\ref{thm:fd-prox-full}, $\|L_k^y\|_F<3\sqrt{d/\alpha}$, so $\|L_k^y\|_F^4<{81d^2}/{\alpha^2}$. Proposition~\ref{prop:hessian-assisted-final} gives
\[
\mathbb E[G_b^{\mathrm{paired}}(y_k)\mid\mathcal F_k]=\nabla E(y_k)
\]
and
\[
\mathbb E[\|G_b^{\mathrm{paired}}(y_k)-\nabla E(y_k)\|^2\mid\mathcal F_k]\le\frac{\sigma_2^2}{b},\qquad \sigma_2^2:=\frac{729}{2}\frac{L_H^2d^2}{\alpha^2}.
\]
Thus
\[
\frac{\sigma_2^2}{\alpha}=\frac{729}{2}\lambda_H^2d^2.
\]
Applying Lemma~\ref{lem:prox-fixed-batch},
\[
\mathbb E[F(q_{T_\varepsilon})-F^\star]\le e^{-T_\varepsilon/s}\Delta_0+\frac{729\lambda_H^2d^2}{2bs}.
\]
The prescribed choices make both terms at most $\varepsilon/2$. Each realization of $G^{\mathrm{paired}}$ uses one gradient evaluation and $d$ HVPs, giving the stated oracle counts.
\end{proof}

\subsection{Cholesky projected acceleration}\label{subsec:fd-projected}

The finite-difference estimator can alternatively be combined with an exact entropy derivative and Euclidean projection. Define
\[
K_L:=\left\{(m,L):m\in\mathbb R^d,\ L\in\mathbb T^d,\ L_{ii}\ge\frac{1}{2\sqrt\beta},\ \|L\|_F\le\sqrt{\frac d\alpha}\right\}.
\]
To extend the entropy gradient outside $K_L$, define
\begin{equation}
\psi(t):=\begin{cases}\log(2\sqrt\beta)-2\sqrt\beta\left(t-\dfrac{1}{2\sqrt\beta}\right)+2\beta\left(t-\dfrac{1}{2\sqrt\beta}\right)^2,&t<\dfrac{1}{2\sqrt\beta},\\-\log t,&t\ge\dfrac{1}{2\sqrt\beta}.\end{cases}
\label{eq:entropy-extension-psi}
\end{equation}
and
\[
\widetilde h(L):=\sum_{i=1}^d\psi(L_{ii}),\qquad \widetilde F(m,L):=E(m,L)+\widetilde h(L).
\]

The next lemma establishes the properties needed for projected acceleration.

\begin{lemma}\label{lem:chol-geometry-final}
The minimizer $x_\star=(m_\star,L_\star)$ belongs to $K_L$. Moreover, $\widetilde F$ is globally $\alpha$-strongly convex and $5\beta$-smooth on $\mathbb R^d\times\mathbb T^d$, and
\[
\widetilde F=F,\qquad \nabla\widetilde F=\nabla F\qquad\text{on }K_L.
\]
\end{lemma}

\begin{proof}
Since $\Sigma_\star=L_\star L_\star^\top\preceq\alpha^{-1}I$, $\|L_\star\|_F^2=\operatorname{tr}(\Sigma_\star)\le d/\alpha$. Let $\Sigma_{\star,[i]}$ denote the leading $i\times i$ principal submatrix of $\Sigma_\star$. The Cholesky factorization gives
\[
L_{\star,ii}^2=\frac{1}{(\Sigma_{\star,[i]}^{-1})_{ii}}.
\]
Since $\Sigma_{\star,[i]}\succeq\beta^{-1}I$, we have $\Sigma_{\star,[i]}^{-1}\preceq\beta I$, and hence $L_{\star,ii}\ge\frac1{\sqrt\beta}>\frac1{2\sqrt\beta}$.
Therefore $x_\star\in K_L$.

The function $\psi$ is continuously differentiable and convex, with
\[
\psi''(t)=\begin{cases}4\beta,&t<\dfrac{1}{2\sqrt\beta},\\[0.4em]\dfrac1{t^2},&t>\dfrac{1}{2\sqrt\beta}.\end{cases}
\]
Thus $0\le\psi''(t)\le4\beta$, so $\widetilde h$ is convex and $4\beta$-smooth. By Lemma~\ref{lem:curvature-inheritance}, $E$ is $\alpha$-strongly convex and $\beta$-smooth. Hence $\widetilde F$ is globally $\alpha$-strongly convex and $5\beta$-smooth.

For $(m,L)\in K_L$, all diagonal entries satisfy $L_{ii}\ge1/(2\sqrt\beta)$, so
\[
\widetilde h(L)=-\sum_{i=1}^d\log L_{ii}=-\log\det L.
\]
The two branches of $\psi$ have the same first derivative at $1/(2\sqrt\beta)$, so the gradients also agree.
\end{proof}

The projection itself also admits a simple one-dimensional characterization.

\begin{proposition}\label{prop:fd-projection}
For $(m,Y)\in\mathbb R^d\times\mathbb T^d$, the Euclidean projection onto $K_L$ is
\[
P_{K_L}(m,Y)=(m,L^{(\gamma)}),
\]
where
\[
L^{(\gamma)}_{ij}:=\frac{Y_{ij}}{1+\gamma}\quad(i>j),\qquad L^{(\gamma)}_{ii}:=\max\left\{\frac1{2\sqrt\beta},\frac{Y_{ii}}{1+\gamma}\right\}.
\]
If $\|L^{(0)}\|_F\le\sqrt{d/\alpha}$, then $\gamma=0$. Otherwise, $\gamma>0$ is the unique solution of
$\|L^{(\gamma)}\|_F^2=d/\alpha$.
\end{proposition}

\begin{proof}
The mean coordinate is unconstrained. For the factor coordinate, consider
\[
\min_{L\in\mathbb T^d}\left\{\frac12\|L-Y\|_F^2:L_{ii}\ge\frac1{2\sqrt\beta},\ \|L\|_F^2\le\frac d\alpha\right\}.
\]
Let $\gamma\ge0$ be the multiplier for the norm constraint and $\mu_i\ge0$ the multipliers for the diagonal lower bounds. The KKT conditions give
\[
(1+\gamma)L_{ij}=Y_{ij}\quad(i>j),\qquad (1+\gamma)L_{ii}=Y_{ii}+\mu_i,
\]
with
\[
\mu_i\left(L_{ii}-\frac1{2\sqrt\beta}\right)=0.
\]
Thus
\[
L_{ii}=\max\left\{\frac1{2\sqrt\beta},\frac{Y_{ii}}{1+\gamma}\right\}.
\]
If $\|L^{(0)}\|_F\le\sqrt{d/\alpha}$, the norm constraint is inactive. Otherwise it is active. The map $\gamma\mapsto\|L^{(\gamma)}\|_F$ is continuous and nonincreasing and is strictly decreasing while its value exceeds $\sqrt{d/\alpha}$. Moreover,
\[
\lim_{\gamma\to\infty}\|L^{(\gamma)}\|_F=\frac{\sqrt d}{2\sqrt\beta}<\sqrt{\frac d\alpha}.
\]
Hence there is a unique positive solution of the active constraint.
\end{proof}

These ingredients lead to Algorithm~\ref{alg:fd-proj}, which is a projected counterpart of the proximal method.

\begin{algorithm}[hbtp]
\caption{Cholesky projected acceleration}
\label{alg:fd-proj}
\begin{algorithmic}[1]
\Require $x_0\in K_L$, batch size $b$, estimator $G_b\in\{G_b^{\mathrm{FD}},G_b^{\mathrm{paired}}\}$, FD scale $\nu$, horizon length $T$
\State $y_0\gets x_0$
\For{$k=0,\ldots,T-1$}
\State $\widehat g_k\gets G_b(y_k)$
\State write $y_k=(m_k^y,L_k^y)$ and set $g_k\gets\widehat g_k+(0,\nabla\widetilde h(L_k^y))$
\State $x_{k+1}\gets P_{K_L}(y_k-\eta g_k)$
\State $y_{k+1}\gets x_{k+1}+\rho(x_{k+1}-x_k)$
\EndFor
\State \Return $q_T=\mathcal N(m_T,L_TL_T^\top)$
\end{algorithmic}
\end{algorithm}

The extrapolated point $y_k$ need not be feasible, but both $E$ and $\widetilde h$ are defined on the entire ambient space. The projection ensures that every accepted iterate belongs to $K_L$.

The same estimator bounds therefore give the following projected guarantee.

\begin{theorem}\label{thm:fd-proj-full}
Suppose Assumptions~\ref{assump:curvature} and~\ref{assump:hessian-lipschitz} hold, $x_0\in K_L$, $0<|\nu|\le\alpha^{-1/2}$, and $0<\varepsilon\le\Delta_0$, where
\[
\Delta_0:=F(q_0)-F^\star+\frac{\alpha}{2}\|x_0-x_\star\|^2.
\]
Let
\[
s:=\sqrt{\max\{5\kappa,4\}},\qquad b:=\left\lceil\max\left\{1,\frac{1494\lambda_H^2d^3}{s\varepsilon}\right\}\right\rceil,\qquad T_\varepsilon:=\left\lceil s\log\frac{2\Delta_0}{\varepsilon}\right\rceil.
\]
Choose
\[
\eta:=\frac{1}{\alpha s^2},
\qquad
\rho:=\frac{s-1}{s+1}.
\]
Then Algorithm~\ref{alg:fd-proj} with $G_b^{\mathrm{FD}}$ satisfies
\[
\mathbb E[F(q_{T_\varepsilon})-F^\star]\le\varepsilon.
\]
The exact first-order oracle count is
\[
N_{\nabla V}=T_\varepsilon(d+1+b),
\]
and consequently
\[
N_{\nabla V}=O\left(\left(d\sqrt\kappa+\frac{\lambda_H^2d^3}{\varepsilon}\right)\log\frac{2\Delta_0}{\varepsilon}\right).
\]
\end{theorem}

\begin{proof}
Every accepted iterate belongs to $K_L$, and the momentum relation gives $\|L_k^y\|_F<3\sqrt{\frac d\alpha}$. Proposition~\ref{prop:fd-estimator-full} therefore gives
\[
\mathbb E[G_b^{\mathrm{FD}}(y_k)\mid\mathcal F_k]=\nabla E(y_k),\qquad \mathbb E[\|G_b^{\mathrm{FD}}(y_k)-\nabla E(y_k)\|^2\mid\mathcal F_k]\le\frac{\sigma^2}{b},
\]
where ${\sigma^2}/{\alpha}\le747\lambda_H^2d^3$. Let
\[
g_k:=G_b^{\mathrm{FD}}(y_k)+(0,\nabla\widetilde h(L_k^y)).
\]
Since the entropy gradient is deterministic,
\[
\mathbb E[g_k\mid\mathcal F_k]=\nabla\widetilde F(y_k)
\]
with the same stochastic error. By Lemma~\ref{lem:chol-geometry-final}, $\widetilde F$ is $\alpha$-strongly convex and $5\beta$-smooth. Moreover,
\[
P_{K_L}=\prox_{\eta\iota_{K_L}}.
\]
Thus Lemma~\ref{lem:prox-fixed-batch} applies with $f=\widetilde F$, $h=\iota_{K_L}$, condition number $5\kappa$, and variance parameter $\sigma^2/b$. Hence
\[
\mathbb E[F(q_{T_\varepsilon})-F^\star]\le e^{-T_\varepsilon/s}\Delta_0+\frac{747\lambda_H^2d^3}{bs}.
\]
The prescribed $T_\varepsilon$ and $b$ make the two terms at most $\varepsilon/2$. The oracle count is the same as in the proximal finite-difference method.
\end{proof}

The second order estimator yields the analogous result.

\begin{corollary}\label{cor:fd-proj-second-order}
Suppose Assumptions~\ref{assump:curvature} and~\ref{assump:hessian-lipschitz} hold, $x_0\in K_L$, and $0<\varepsilon\le\Delta_0$, where
\[
\Delta_0:=F(q_0)-F^\star+\frac{\alpha}{2}\|x_0-x_\star\|^2.
\]
Let
\[
s:=\sqrt{\max\{5\kappa,4\}},\qquad b:=\left\lceil\max\left\{1,\frac{729\lambda_H^2d^2}{s\varepsilon}\right\}\right\rceil,\qquad T_\varepsilon:=\left\lceil s\log\frac{2\Delta_0}{\varepsilon}\right\rceil.
\]
Run Algorithm~\ref{alg:fd-proj} with $G_b^{\mathrm{paired}}$ in place of $G_b^{\mathrm{FD}}$. Then
\[
\mathbb E[F(q_{T_\varepsilon})-F^\star]\le\varepsilon.
\]
The exact oracle counts are
\[
N_{\nabla V}=bT_\varepsilon,\qquad N_{\mathrm{HVP}}=dbT_\varepsilon.
\]
Consequently,
\[
N_{\nabla V}=O\left(\left(\sqrt\kappa+\frac{\lambda_H^2d^2}{\varepsilon}\right)\log\frac{2\Delta_0}{\varepsilon}\right),
\]
and
\[
N_{\mathrm{HVP}}=O\left(\left(d\sqrt\kappa+\frac{\lambda_H^2d^3}{\varepsilon}\right)\log\frac{2\Delta_0}{\varepsilon}\right).
\]
\end{corollary}

\begin{proof}
As above, $\|L_k^y\|_F<3\sqrt{d/\alpha}$. Proposition~\ref{prop:hessian-assisted-final} gives
\[
\mathbb E[G_b^{\mathrm{paired}}(y_k)\mid\mathcal F_k]=\nabla E(y_k)
\]
and
\[
\mathbb E[\|G_b^{\mathrm{paired}}(y_k)-\nabla E(y_k)\|^2\mid\mathcal F_k]\le\frac{\sigma^2}{b},\qquad \sigma^2:=\frac{729}{2}\frac{L_H^2d^2}{\alpha^2}.
\]
Adding the deterministic entropy gradient preserves this stochastic error. Lemma~\ref{lem:prox-fixed-batch} gives
\[
\mathbb E[F(q_{T_\varepsilon})-F^\star]\le e^{-T_\varepsilon/s}\Delta_0+\frac{729\lambda_H^2d^2}{2bs}.
\]
The parameter choices make both terms at most $\varepsilon/2$. Each realization uses one gradient evaluation and $d$ HVPs.
\end{proof}

\subsection{Symmetric-root score acceleration}\label{subsec:score}

We begin by recording the gradient identities underlying STL and Price--STL. Let
\[
X:=m+C\xi,\qquad \xi\sim\mathcal N(0,I_d),
\]
with $C\in\mathbb S_{++}^d$. Since
\[
F(m,C)=\mathbb E[V(m+C\xi)]-\log\det C,
\]
differentiation under the expectation gives
\[
\nabla_mF(m,C)=\mathbb E[\nabla V(X)],\qquad \nabla_CF(m,C)=\operatorname{sym}\mathbb E[\nabla V(X)\xi^\top]-C^{-1}.
\]
Moreover,
\[
\nabla_x\log q_{m,C}(X)=-C^{-1}\xi.
\]
Thus
\[
\mathbb E[G^{\mathrm{STL}}(m,C;\xi)]=\nabla F(m,C).
\]
Stein's identity gives
\[
\mathbb E[\nabla V(X)\xi^\top]=\mathbb E[\nabla^2V(X)C],
\]
and therefore
\[
\mathbb E[G^{\mathrm{Price\text{-}STL}}(m,C;\xi)]=\nabla F(m,C).
\]

The following matrix inequality is useful for controlling the score-estimator variance in symmetric-root coordinates.

\begin{lemma}\label{lem:skew-final}
Let $A\in\mathbb S^d$, $C\in\mathbb S_{++}^d$, and $B:=AC-C^{-1}$. Then
\[
\|B\|_F^2\le2\|\operatorname{sym}(B)\|_F^2.
\]
\end{lemma}

\begin{proof}
Since $A$ and $C$ are symmetric,
\[
CB=CAC-I=B^\top C.
\]
Diagonalize $C=Q\operatorname{diag}(c_1,\ldots,c_d)Q^\top$ and write $\widetilde B:=Q^\top BQ$. Then
\[
c_i\widetilde B_{ij}=c_j\widetilde B_{ji}.
\]
Let $S:=\operatorname{sym}(\widetilde B)$ and $K:=\operatorname{skew}(\widetilde B)$ where
$\operatorname{skew}(M):=\frac{M-M^\top}{2}$.
Then
\[
K_{ij}=\frac{c_j-c_i}{c_i+c_j}S_{ij}.
\]
Since $c_i,c_j>0$,
\[
\|K\|_F^2\le\|S\|_F^2.
\]
By orthogonal invariance and Frobenius orthogonality,
\[
\|B\|_F^2=\|S\|_F^2+\|K\|_F^2\le2\|S\|_F^2=2\|\operatorname{sym}(B)\|_F^2.
\]
\end{proof}

We can now establish the explicit relative and additive moment bounds.

\begin{lemma}\label{lem:score-growth-full}
If $V$ is quadratic, then for every $(m,C)\in\mathbb R^d\times\mathbb S_{++}^d$,
\[
\mathbb E\|G_b^{\mathrm{STL}}(m,C)\|^2\le\left(1+\frac{d+4}{b}\right)\|\nabla F(m,C)\|^2,
\]
and
\[
\mathbb E\|G_b^{\mathrm{Price\text{-}STL}}(m,C)\|^2\le\left(1+\frac{2}{b}\right)\|\nabla F(m,C)\|^2.
\]
Under Assumptions~\ref{assump:curvature} and~\ref{assump:hessian-lipschitz}, for every $(m,C)\in K_C$,
\[
\mathbb E\|G_b^{\mathrm{STL}}(m,C)\|^2\le\left(1+\frac{4d+16}{b}\right)\|\nabla F(m,C)\|^2+\frac{128L_H^2d(d+2)(d+5)}{\alpha^2b},
\]
and
\[
\mathbb E\|G_b^{\mathrm{Price\text{-}STL}}(m,C)\|^2\le\left(1+\frac{8}{b}\right)\|\nabla F(m,C)\|^2+\frac{16L_H^2}{\alpha^2b}\left(10d^2+\frac{d(d+2)}{2}\right).
\]
\end{lemma}

\begin{proof}
Fix $(m,C)$ and write
\[
X:=m+C\xi,\qquad \xi\sim\mathcal N(0,I_d),\qquad g:=\nabla F(m,C).
\]
For a square-integrable random vector or matrix $Z$, write
\[
\operatorname{Var}(Z):=\mathbb E\|Z-\mathbb EZ\|^2.
\]
If $Z_1,\ldots,Z_b$ are independent copies of $Z$, then
\[
\mathbb E\left\|\frac1b\sum_{i=1}^bZ_i\right\|^2=\|\mathbb EZ\|^2+\frac1b\operatorname{Var}(Z).
\]
Thus it suffices to control the single-sample variances. Define
\[
u:=\nabla V(m),\qquad H:=\nabla^2V(m),\qquad B:=HC-C^{-1},\qquad S:=\operatorname{sym}(B).
\]
Suppose first that $V$ is quadratic. Then
\[
r(\xi):=\nabla V(X)-C^{-1}\xi=u+B\xi,
\]
and
\[
G^{\mathrm{STL}}=\left(u+B\xi,\operatorname{sym}(r(\xi)\xi^\top)\right),\qquad G^{\mathrm{Price\text{-}STL}}=(u+B\xi,S),
\]
with $g=(u,S)$. For STL,
\[
\mathbb E\|r(\xi)-u\|^2=\|B\|_F^2.
\]
For every $p,z\in\mathbb R^d$,
\[
\|\operatorname{sym}(pz^\top)\|_F^2=\frac12\|p\|^2\|z\|^2+\frac12(p^\top z)^2.
\]
Using vanishing odd Gaussian moments,
\[
\mathbb E[\|r(\xi)\|^2\|\xi\|^2]=d\|u\|^2+(d+2)\|B\|_F^2.
\]
Moreover,
\[
\mathbb E[(r(\xi)^\top\xi)^2]=\|u\|^2+2\|S\|_F^2+(\operatorname{tr}S)^2.
\]
Since $\mathbb E[\operatorname{sym}(r(\xi)\xi^\top)]=S$,
\[
\operatorname{Var}(G^{\mathrm{STL}})=\frac{d+1}{2}\|u\|^2+\frac{d+4}{2}\|B\|_F^2+\frac12(\operatorname{tr}S)^2.
\]

Diagonalize
\[
C=Q\operatorname{diag}(c_1,\ldots,c_d)Q^\top,\qquad \widetilde B:=Q^\top BQ,\qquad \widetilde S:=Q^\top SQ.
\]
The identity $CB=B^\top C$ gives
\[
c_i\widetilde B_{ij}=c_j\widetilde B_{ji},\qquad \widetilde B_{ij}=\frac{2c_j}{c_i+c_j}\widetilde S_{ij}.
\]
Hence
\[
\|B\|_F^2\le\sum_i\widetilde S_{ii}^2+2\sum_{i\ne j}\widetilde S_{ij}^2,
\]
while
\[
(\operatorname{tr}S)^2\le d\sum_i\widetilde S_{ii}^2.
\]
Substituting into the variance formula,
\[
\operatorname{Var}(G^{\mathrm{STL}})\le(d+4)(\|u\|^2+\|S\|_F^2)=(d+4)\|g\|^2.
\]
For Price--STL, the covariance-factor component is deterministic, so Lemma~\ref{lem:skew-final} gives
\[
\operatorname{Var}(G^{\mathrm{Price\text{-}STL}})=\|B\|_F^2\le2\|S\|_F^2\le2\|g\|^2.
\]
This proves the quadratic bounds after batch averaging. Now suppose Assumptions~\ref{assump:curvature} and~\ref{assump:hessian-lipschitz} hold and $(m,C)\in K_C$. Then
\[
\|C\|_{\mathrm{op}}\le\frac2{\sqrt\alpha},\qquad \|C\|_F\le2\sqrt{\frac d\alpha}.
\]
Define
\[
w(\xi):=\nabla V(X)-u-HC\xi,\qquad D(\xi):=\operatorname{sym}((\nabla^2V(X)-H)C),\qquad M:=\frac{2L_H}{\alpha}.
\]
Taylor's theorem and Hessian Lipschitzness give
\[
\|w(\xi)\|\le\frac{L_H}{2}\|C\xi\|^2\le M\|\xi\|^2
\]
and
\[
\|D(\xi)\|_F\le\|\nabla^2V(X)-H\|_{\mathrm{op}}\|C\|_F\le2M\sqrt d\,\|\xi\|.
\]
Therefore
\[
\mathbb E\|w(\xi)\|^2\le M^2d(d+2),\qquad \mathbb E\|D(\xi)\|_F^2\le4M^2d^2.
\]
Stein's identity gives
\[
\operatorname{sym}\mathbb E[w(\xi)\xi^\top]=\mathbb ED(\xi).
\]
Define $\delta:=(\mathbb Ew,\mathbb ED).$ Then $g=(u,S)+\delta$. Moreover,
\[
\|\mathbb Ew\|\le Md,\qquad \|\mathbb ED\|_F\le2Md,
\]
so $\|\delta\|^2\le5M^2d^2$. For STL, decompose
\[
G^{\mathrm{STL}}=\mathcal A+\mathcal E,
\]
where $\mathcal A(\xi):=\left(u+B\xi,\operatorname{sym}((u+B\xi)\xi^\top)\right)$ and $\mathcal E(\xi):=\left(w(\xi),\operatorname{sym}(w(\xi)\xi^\top)\right)$. Then
\[
\mathbb E\mathcal E=\delta,\qquad \mathbb E\mathcal A=g-\delta.
\]
The quadratic calculation yields
\[
\operatorname{Var}(\mathcal A)\le(d+4)\|g-\delta\|^2.
\]
Also,
\[
\mathbb E\|\mathcal E\|^2\le\mathbb E[(1+\|\xi\|^2)\|w(\xi)\|^2]\le M^2d(d+2)(d+5)=:\sigma_H^2.
\]
Centering the two components and using $\|a+b\|^2\le2\|a\|^2+2\|b\|^2$,
\[
\operatorname{Var}(G^{\mathrm{STL}})\le2(d+4)\|g-\delta\|^2+2\sigma_H^2.
\]
Thus
\[
\operatorname{Var}(G^{\mathrm{STL}})\le4(d+4)\|g\|^2+4(d+4)\|\delta\|^2+2\sigma_H^2.
\]
The bound on $\delta$ gives
\[
(d+4)\|\delta\|^2\le\frac{15}{2}\sigma_H^2.
\]
Hence
\[
\operatorname{Var}(G^{\mathrm{STL}})\le(4d+16)\|g\|^2+32\sigma_H^2
\]
and therefore
\[
\operatorname{Var}(G^{\mathrm{STL}})\le(4d+16)\|g\|^2+\frac{128L_H^2d(d+2)(d+5)}{\alpha^2}.
\]

For Price--STL,
\[
G^{\mathrm{Price\text{-}STL}}=(u+B\xi+w(\xi),S+D(\xi)),
\]
so
\[
G^{\mathrm{Price\text{-}STL}}-g=(B\xi,0)+(w-\mathbb Ew,D-\mathbb ED).
\]
Therefore
\[
\operatorname{Var}(G^{\mathrm{Price\text{-}STL}})\le2\|B\|_F^2+2\mathbb E\|w\|^2+2\mathbb E\|D\|_F^2.
\]
Let $g_C:=\nabla_CF(m,C)$. Since $g_C=S+\mathbb ED$, Lemma~\ref{lem:skew-final} and Jensen's inequality give
\[
\|B\|_F^2\le2\|S\|_F^2\le4\|g_C\|_F^2+4\|\mathbb ED\|_F^2\le4\|g\|^2+4\mathbb E\|D\|_F^2.
\]
Hence
\[
\operatorname{Var}(G^{\mathrm{Price\text{-}STL}})\le8\|g\|^2+2\mathbb E\|w\|^2+10\mathbb E\|D\|_F^2,
\]
which yields
\[
\operatorname{Var}(G^{\mathrm{Price\text{-}STL}})\le8\|g\|^2+\frac{16L_H^2}{\alpha^2}\left(10d^2+\frac{d(d+2)}{2}\right).
\]
Batch averaging completes the proof.
\end{proof}

We next construct the smooth spectral extension needed to apply projected acceleration.

Using the scalar function $\psi$ from \eqref{eq:entropy-extension-psi}, define
\[
\widetilde h(C):=\sum_{i=1}^d\psi(\lambda_i(C)),\qquad \widetilde F(m,C):=E(m,C)+\widetilde h(C).
\]

\begin{lemma}\label{lem:symmetric-geometry-final}
Under Assumption~\ref{assump:curvature}, the minimizer $x_\star=(m_\star,C_\star)$ belongs to the interior of $K_C$. Moreover, $\widetilde F$ is globally $\alpha$-strongly convex and $5\beta$-smooth on $\mathbb R^d\times\mathbb S^d$, satisfies
\[
\widetilde F=F,\qquad \nabla\widetilde F=\nabla F\qquad\text{on }K_C,
\]
and obeys
\[
\widetilde F(P_{K_C}z)\le\widetilde F(z)
\]
for every $z\in\mathbb R^d\times\mathbb S^d$.
\end{lemma}

\begin{proof}
From $\beta^{-1}I\preceq\Sigma_\star\preceq\alpha^{-1}I$
and $C_\star=\Sigma_\star^{1/2}$,
\[
\frac1{\sqrt\beta}I\preceq C_\star\preceq\frac1{\sqrt\alpha}I,
\]
so $x_\star$ belongs to the interior of $K_C$.

Since $0\le\psi''\le4\beta$, the spectral function $\widetilde h$ is convex and $4\beta$-smooth on $\mathbb S^d$. By Lemma~\ref{lem:curvature-inheritance}, $E$ is $\alpha$-strongly convex and $\beta$-smooth. Hence $\widetilde F$ is globally $\alpha$-strongly convex and $5\beta$-smooth. For every $(m,C)\in K_C$,
\[
\widetilde h(C)=-\sum_{i=1}^d\log\lambda_i(C)=-\log\det C.
\]
The two branches of $\psi$ have matching first derivatives, so the gradients also agree.

It remains to prove projection monotonicity. Write
\[
C=Q\operatorname{diag}(c_1,\ldots,c_d)Q^\top
\]
and let $\widetilde c_i$ be $c_i$ clipped to $[1/(2\sqrt\beta),2/\sqrt\alpha]$. Then
\[
P_{K_C}(m,C)=\left(m,Q\operatorname{diag}(\widetilde c_1,\ldots,\widetilde c_d)Q^\top\right).
\]
Let $q_i$ be the $i$th column of $Q$ and define
\[
X_c:=m+\sum_{j=1}^dc_jq_j\xi_j.
\]
Gaussian integration by parts gives
\[
\frac{\partial}{\partial c_i}\mathbb E[V(X_c)]=c_i\,\mathbb E[q_i^\top\nabla^2V(X_c)q_i].
\]
Write $a_i(c):=\mathbb E[q_i^\top\nabla^2V(X_c)q_i]$,
so that $\alpha\le a_i(c)\le\beta$. The derivative of $\widetilde F$ in the $c_i$ coordinate is $c_i a_i(c)+\psi'(c_i)$. If $c_i\ge2/\sqrt\alpha$, then $\psi'(c_i)=-1/c_i$, and
\[
c_i a_i(c)+\psi'(c_i)\ge\alpha c_i-\frac1{c_i}\ge2\sqrt\alpha-\frac{\sqrt\alpha}{2}>0.
\]
Thus clipping $c_i$ downward cannot increase $\widetilde F$. If $c_i\le1/(2\sqrt\beta)$, then
\[
\psi'(c_i)=4\beta c_i-4\sqrt\beta,
\]
and hence
\[
c_i a_i(c)+\psi'(c_i)=c_i(a_i(c)+4\beta)-4\sqrt\beta\le\frac{5\beta}{2\sqrt\beta}-4\sqrt\beta<0.
\]
Thus clipping $c_i$ upward cannot increase $\widetilde F$. Coordinates already in the spectral interval are unchanged, proving $\widetilde F(P_{K_C}(m,C))\le\widetilde F(m,C)$.
\end{proof}

The remaining optimization argument can be done by the following projected acceleration lemma, which extends the unconstrained strong-growth acceleration argument of \citet{vaswani2019fast} to the projected setting.

\begin{lemma}\label{lem:score-proj-fixed-batch}
Let $K$ be a nonempty closed convex subset of a finite-dimensional Euclidean space. Let $\Phi$ be globally $\alpha$-strongly convex and $\beta$-smooth, with minimizer $x_\star\in K$, and write $\Phi^\star:=\Phi(x_\star)$ and $\kappa:=\beta/\alpha$. Suppose
\[
\Phi(P_Kz)\le\Phi(z)
\]
for every $z$ in the ambient space. Set
\[
s:=\sqrt{\max\{3\kappa,4\}},\qquad \eta:=\frac1{\alpha s^2}.
\]
Starting from $x_0=v_0\in K$, consider the updates
\[
y_k:=\frac{x_k+s^{-1}v_k}{1+s^{-1}},
\]
\[
x_{k+1}:=P_K(y_k-\eta g_k),
\]
and
\[
v_{k+1}:=P_K\left((1-s^{-1})v_k+s^{-1}y_k-\frac1{\alpha s}g_k\right).
\]
Suppose, conditionally on the past $\mathcal F_k$,
\[
\mathbb E[g_k\mid\mathcal F_k]=\nabla\Phi(y_k)
\]
and
\[
\mathbb E[\|g_k-\nabla\Phi(y_k)\|^2\mid\mathcal F_k]\le\frac12\|\nabla\Phi(y_k)\|^2+\sigma^2.
\]
Then, for every integer $T\ge0$,
\[
\mathbb E[\Phi(x_T)-\Phi^\star]\le e^{-T/s}\left(\Phi(x_0)-\Phi^\star+\frac{\alpha}{2}\|x_0-x_\star\|^2\right)+\frac{\sigma^2}{\alpha s}.
\]
\end{lemma}

\begin{proof}
Since $s\ge2$ and $s^2\ge3\kappa=3\beta/\alpha$, we have $\beta\eta\le1/3$. The projection steps ensure that $x_k,v_k\in K$, and convexity of $K$ gives $y_k\in K$. Let $\mathcal F_k$ denote the sigma-field generated by the randomness before the fresh batch at iteration $k$, so that $x_k$, $v_k$, and $y_k$ are $\mathcal F_k$-measurable.

Conditional unbiasedness and the variance assumption imply
\[
\mathbb E[\|g_k\|^2\mid\mathcal F_k]=\|\nabla\Phi(y_k)\|^2+\mathbb E[\|g_k-\nabla\Phi(y_k)\|^2\mid\mathcal F_k]\le\frac32\|\nabla\Phi(y_k)\|^2+\sigma^2.
\]

Projection optimality for the $v$-update gives
\[
p_k:=\alpha s\left((1-s^{-1})v_k+s^{-1}y_k-v_{k+1}\right)-g_k
\]
with
\[
\langle p_k,z-v_{k+1}\rangle\le0\qquad\text{for every }z\in K.
\]
Define
\[
\Lambda_k(z):=\Phi(y_k)+\langle g_k,z-y_k\rangle+\frac{\alpha}{2}\|z-y_k\|^2+\langle p_k,z-v_{k+1}\rangle.
\]
Thus strong convexity and conditional unbiasedness give
\[
\mathbb E[\Lambda_k(z)\mid\mathcal F_k]\le\Phi(y_k)+\langle\nabla\Phi(y_k),z-y_k\rangle+\frac{\alpha}{2}\|z-y_k\|^2\le\Phi(z)
\]
for every $\mathcal F_k$-measurable $z\in K$.

Projection monotonicity and $\beta$-smoothness give
\[
\Phi(x_{k+1})\le\Phi(y_k-\eta g_k)\le\Phi(y_k)-\eta\langle\nabla\Phi(y_k),g_k\rangle+\frac{\beta\eta^2}{2}\|g_k\|^2.
\]
Taking conditional expectations,
\[
\mathbb E[\Phi(x_{k+1})\mid\mathcal F_k]\le\Phi(y_k)-\eta\|\nabla\Phi(y_k)\|^2+\frac{\beta\eta^2}{2}\mathbb E[\|g_k\|^2\mid\mathcal F_k].
\]

Define
\[
\varphi_0(z):=\Phi(x_0)+\frac{\alpha}{2}\|z-x_0\|^2,\qquad \varphi_{k+1}:=(1-s^{-1})\varphi_k+s^{-1}\Lambda_k.
\]
Every $\varphi_k$ is quadratic with Hessian $\alpha I$. Its unique minimizer is $v_k$. Indeed, this holds at $k=0$, and
\[
\nabla\varphi_{k+1}(v_{k+1})=(1-s^{-1})\alpha(v_{k+1}-v_k)+s^{-1}\left(g_k+p_k+\alpha(v_{k+1}-y_k)\right)=0
\]
by the definition of $p_k$. Writing
\[
\varphi_k^{\min}:=\varphi_k(v_k),
\]
we have
\[
\varphi_k(z)=\varphi_k^{\min}+\frac{\alpha}{2}\|z-v_k\|^2.
\]
The definition of $y_k$ gives
\[
x_k-y_k=-s^{-1}(v_k-y_k).
\]

Set
\[
\zeta_0:=0,\qquad \zeta_{k+1}:=(1-s^{-1})\zeta_k+\eta\sigma^2.
\]
We prove by induction that
\[
\mathbb E[\varphi_k^{\min}]\ge\mathbb E[\Phi(x_k)]-\zeta_k.
\]
The claim holds with equality at $k=0$.

Evaluating $\varphi_{k+1}$ at $v_{k+1}$ gives
\[
\begin{aligned}
\varphi_{k+1}^{\min}
&=(1-s^{-1})\varphi_k^{\min}+s^{-1}\Phi(y_k)
+(1-s^{-1})\frac{\alpha}{2}\|v_{k+1}-v_k\|^2\\
&\quad+s^{-1}\langle g_k,v_{k+1}-y_k\rangle
+s^{-1}\frac{\alpha}{2}\|v_{k+1}-y_k\|^2.
\end{aligned}
\]
Completing the square in the terms involving $v_{k+1}$ gives the exact identity
\begin{align*}
&(1-s^{-1})\frac{\alpha}{2}\|v_{k+1}-v_k\|^2+s^{-1}\langle g_k,v_{k+1}-y_k\rangle+s^{-1}\frac{\alpha}{2}\|v_{k+1}-y_k\|^2\\
&=(1-s^{-1})s^{-1}\left(\langle g_k,v_k-y_k\rangle+\frac{\alpha}{2}\|v_k-y_k\|^2\right)\\
&-\frac{1}{2\alpha s^2}\|g_k\|^2+\frac{\alpha}{2}\left\|v_{k+1}-\left((1-s^{-1})v_k+s^{-1}y_k-\frac1{\alpha s}g_k\right)\right\|^2.
\end{align*}
Dropping the final nonnegative term,
\[
\varphi_{k+1}^{\min}\ge(1-s^{-1})\varphi_k^{\min}+s^{-1}\Phi(y_k)+(1-s^{-1})s^{-1}\left(\langle g_k,v_k-y_k\rangle+\frac{\alpha}{2}\|v_k-y_k\|^2\right)-\frac{\eta}{2}\|g_k\|^2.
\]

By the induction hypothesis, $\mathbb E[\varphi_k^{\min}]\ge\mathbb E[\Phi(x_k)]-\zeta_k$. Strong convexity gives
\[
\Phi(x_k)\ge\Phi(y_k)+\langle\nabla\Phi(y_k),x_k-y_k\rangle+\frac{\alpha}{2}\|x_k-y_k\|^2.
\]
Using $x_k-y_k=-s^{-1}(v_k-y_k)$ and conditional unbiasedness,
\[
\mathbb E[\Phi(x_k)]\ge\mathbb E[\Phi(y_k)]-s^{-1}\mathbb E\langle g_k,v_k-y_k\rangle+\frac{\alpha}{2s^2}\mathbb E\|v_k-y_k\|^2.
\]
Substituting into the estimate for $\varphi_{k+1}^{\min}$ cancels the inner-product terms and yields
\[
\mathbb E[\varphi_{k+1}^{\min}]\ge\mathbb E[\Phi(y_k)]-(1-s^{-1})\zeta_k-\frac{\eta}{2}\mathbb E\|g_k\|^2+\frac{\alpha}{2}(1-s^{-1})s^{-1}(1+s^{-1})\mathbb E\|v_k-y_k\|^2.
\]
Dropping the final nonnegative term,
\[
\mathbb E[\varphi_{k+1}^{\min}]\ge\mathbb E[\Phi(y_k)]-(1-s^{-1})\zeta_k-\frac{\eta}{2}\mathbb E\|g_k\|^2.
\]

The one-step upper bound implies
\[
\mathbb E[\Phi(y_k)]\ge\mathbb E[\Phi(x_{k+1})]+\eta\mathbb E\|\nabla\Phi(y_k)\|^2-\frac{\beta\eta^2}{2}\mathbb E\|g_k\|^2.
\]
Therefore
\[
\mathbb E[\varphi_{k+1}^{\min}]\ge\mathbb E[\Phi(x_{k+1})]-(1-s^{-1})\zeta_k+\eta\mathbb E\|\nabla\Phi(y_k)\|^2-\frac{\eta}{2}(1+\beta\eta)\mathbb E\|g_k\|^2.
\]
Since $\beta\eta\le1/3$,
\[
\eta\mathbb E\|\nabla\Phi(y_k)\|^2-\frac{\eta}{2}(1+\beta\eta)\mathbb E\|g_k\|^2\ge-\eta\sigma^2.
\]
Hence
\[
\mathbb E[\varphi_{k+1}^{\min}]\ge\mathbb E[\Phi(x_{k+1})]-\zeta_{k+1},
\]
closing the induction. Finally, since $x_\star\in K$,
\[
\mathbb E[\Lambda_k(x_\star)\mid\mathcal F_k]\le\Phi^\star.
\]
Iterating the estimate-sequence recursion gives
\[
\mathbb E[\varphi_T(x_\star)]-\Phi^\star\le(1-s^{-1})^T\left(\varphi_0(x_\star)-\Phi^\star\right).
\]
The induction invariant and $\varphi_T^{\min}\le\varphi_T(x_\star)$ imply
\[
\mathbb E[\Phi(x_T)-\Phi^\star]\le(1-s^{-1})^T\left(\Phi(x_0)-\Phi^\star+\frac{\alpha}{2}\|x_0-x_\star\|^2\right)+\zeta_T.
\]
Moreover,
\[
\zeta_T=\eta\sigma^2\sum_{j=0}^{T-1}(1-s^{-1})^j\le\eta s\sigma^2=\frac{\sigma^2}{\alpha s}.
\]
Using $(1-s^{-1})^T\le e^{-T/s}$ completes the proof.
\end{proof}

Combining the moment bounds, the smooth extension, and the projected acceleration lemma gives the full score-based guarantee.

\begin{theorem}[Full version of Theorem~\ref{thm:score-recovery-clean}]\label{thm:score-recovery-full}
Suppose Assumptions~\ref{assump:curvature} and~\ref{assump:hessian-lipschitz} hold, $x_0\in K_C$, and $0<\varepsilon\le\Delta_0$, where
\[
\Delta_0:=F(q_0)-F^\star+\frac{\alpha}{2}\|x_0-x_\star\|^2.
\]
Let
\[
s:=\sqrt{15\kappa},\qquad T_\varepsilon:=\left\lceil s\log\frac{2\Delta_0}{\varepsilon}\right\rceil.
\]
Choose
\[
\eta:=\frac{1}{\alpha s^2}=\frac{1}{15\beta}.
\]
For STL, choose
\[
b_{\mathrm{STL}}:=\left\lceil\max\left\{8d+32,\frac{256\lambda_H^2d(d+2)(d+5)}{s\varepsilon}\right\}\right\rceil.
\]
For Price--STL, choose
\[
b_{\mathrm{Price}}:=\left\lceil\max\left\{16,\frac{16\lambda_H^2d(21d+2)}{s\varepsilon}\right\}\right\rceil.
\]
Then Algorithm~\ref{alg:score-proj}, with either estimator and its corresponding fixed batch size, satisfies
\[
\mathbb E[F(q_{T_\varepsilon})-F^\star]\le\varepsilon.
\]
The exact oracle counts are
\[
N_{\nabla V}^{\mathrm{STL}}=b_{\mathrm{STL}}T_\varepsilon,
\]
and
\[
N_{\nabla V}^{\mathrm{Price}}=b_{\mathrm{Price}}T_\varepsilon,\qquad N_{\mathrm{HVP}}^{\mathrm{Price}}=db_{\mathrm{Price}}T_\varepsilon.
\]
Consequently,
\[
T_\varepsilon=O\left(\sqrt\kappa\log\frac{2\Delta_0}{\varepsilon}\right),
\]
\[
N_{\nabla V}^{\mathrm{STL}}=O\left(\left(d\sqrt\kappa+\frac{\lambda_H^2d^3}{\varepsilon}\right)\log\frac{2\Delta_0}{\varepsilon}\right),
\]
\[
N_{\nabla V}^{\mathrm{Price}}=O\left(\left(\sqrt\kappa+\frac{\lambda_H^2d^2}{\varepsilon}\right)\log\frac{2\Delta_0}{\varepsilon}\right),
\]
and
\[
N_{\mathrm{HVP}}^{\mathrm{Price}}=O\left(\left(d\sqrt\kappa+\frac{\lambda_H^2d^3}{\varepsilon}\right)\log\frac{2\Delta_0}{\varepsilon}\right).
\]
\end{theorem}

\begin{proof}
The two projection steps ensure that $x_k,v_k\in K_C$, and convexity gives $y_k\in K_C$. By Lemma~\ref{lem:symmetric-geometry-final}, $\widetilde F$ is globally $\alpha$-strongly convex and $5\beta$-smooth, agrees with $F$ and its gradient on $K_C$, and satisfies
\[
\widetilde F(P_{K_C}z)\le\widetilde F(z).
\]
Since $x_\star$ lies in the interior of $K_C$ and $\nabla F(x_\star)=0$, it is also the global minimizer of $\widetilde F$.

For the selected estimator, define
\[
\sigma^2:=\begin{cases}\dfrac{128L_H^2d(d+2)(d+5)}{\alpha^2},&\text{STL},\\[0.4em]\dfrac{8L_H^2d(21d+2)}{\alpha^2},&\text{Price--STL}.\end{cases}
\]
Lemma~\ref{lem:score-growth-full} and fresh batch sampling give
\[
\mathbb E[g_k\mid\mathcal F_k]=\nabla\widetilde F(y_k).
\]
For STL,
\[
\mathbb E[\|g_k-\nabla\widetilde F(y_k)\|^2\mid\mathcal F_k]\le\frac{4d+16}{b}\|\nabla\widetilde F(y_k)\|^2+\frac{\sigma^2}{b},
\]
while for Price--STL,
\[
\mathbb E[\|g_k-\nabla\widetilde F(y_k)\|^2\mid\mathcal F_k]\le\frac8b\|\nabla\widetilde F(y_k)\|^2+\frac{\sigma^2}{b}.
\]
The prescribed batch sizes satisfy $b_{\mathrm{STL}}\ge8d+32$ and $b_{\mathrm{Price}}\ge16$, so in either case
\[
\mathbb E[\|g_k-\nabla\widetilde F(y_k)\|^2\mid\mathcal F_k]\le\frac12\|\nabla\widetilde F(y_k)\|^2+\frac{\sigma^2}{b}.
\]

Lemma~\ref{lem:score-proj-fixed-batch} therefore applies with $\Phi=\widetilde F$, $K=K_C$, strong-convexity constant $\alpha$, smoothness constant $5\beta$, and additive variance parameter $\sigma^2/b$. Its parameter becomes
\[
s=\sqrt{\max\{15\kappa,4\}}=\sqrt{15\kappa}.
\]
Hence
\[
\mathbb E[F(q_T)-F^\star]\le e^{-T/s}\Delta_0+\frac{\sigma^2}{\alpha bs}.
\]
At $T=T_\varepsilon$, $e^{-T_\varepsilon/s}\Delta_0\le\frac{\varepsilon}{2}$. Using $\lambda_H^2=L_H^2/\alpha^3$, each prescribed batch size also satisfies
\[
b\ge\frac{2\sigma^2}{\alpha s\varepsilon},
\]
and therefore
\[
\frac{\sigma^2}{\alpha bs}\le\frac{\varepsilon}{2}.
\]
This proves the accuracy bound.

Each STL sample uses one target-gradient evaluation. Each Price--STL sample uses one target-gradient evaluation and $d$ HVPs to form $\nabla^2V(X)C$. The exact oracle counts and their asymptotic forms follow directly from the batch sizes and $T_\varepsilon$.
\end{proof}

For quadratic targets, $L_H=0$, so the additive variance term vanishes. The preceding proof gives geometric accelerated convergence with fixed batch sizes of orders $d$ and $1$ for STL and Price--STL, respectively.

\begin{remark}
The fixed-batch choices in Theorems~\ref{thm:fd-prox} and~\ref{thm:score-recovery-clean} are selected for the target accuracy $\varepsilon$ and kept fixed throughout the run. Consequently, the variance-dependent oracle terms carry an additional factor $\log(2\Delta_0/\varepsilon)$. Multistage or increasing-batch stochastic acceleration can avoid this extra logarithmic factor \citep{ghadimi2013optimal}.
\end{remark}

\subsection{Dimension dependence and structured Hessian variation}
\label{app:dimension}

The additive variance bounds underlying Table~\ref{tab:stochastic-rates} scale as $L_H^2d^3/\alpha^2$ for FD and STL, and as $L_H^2d^2/\alpha^2$ for their second-order variants. These powers arise from uniform norm and Gaussian-moment bounds. Here we give a sharper analysis of the same Price--STL estimator using two alternative measures of Hessian variation. Write $H(x):=\nabla^2V(x)$ and define the global curvature oscillation
\begin{equation*}
\omega_H:=\inf_{A\in\mathbb S^d}\sup_{x\in\mathbb R^d}\|H(x)-A\|_{\mathrm{op}}.
\end{equation*}
Under Assumption~\ref{assump:curvature}, $\omega_H\le(\beta-\alpha)/2$. 
Alternatively, suppose a constant $L_{H,F}$ satisfies
\begin{equation}
\|H(x)-H(y)\|_F\le L_{H,F}\|x-y\|,\qquad x,y\in\mathbb R^d.
\label{eq:dimension-frobenius}
\end{equation}
These quantities measure the amplitude and the rate of Hessian variation, respectively. 
For example, $V''(x)=a+\delta\sin(kx)$, with $a>\delta>0$ and $k>0$, has $\omega_H=\delta$ and $L_{H,F}=\delta k$ while its curvature bounds stay fixed. 

\begin{proposition}\label{prop:price-structured-variance}
Under Assumption~\ref{assump:curvature}, the single-sample Price--STL estimator satisfies, for every $(m,C)\in K_C$,
\begin{equation*}
\operatorname{Var}(G^{\mathrm{Price\text{-}STL}})\le2\|\nabla F(m,C)\|^2+\sigma_{\mathrm{str}}^2,
\qquad \sigma_{\mathrm{str}}^2:=\frac{8\omega_H^2d}{\alpha}.
\end{equation*}
If \eqref{eq:dimension-frobenius} also holds, one may instead take
\[
\sigma_{\mathrm{str}}^2:=\min\left\{\frac{8\omega_H^2d}{\alpha},\frac{32L_{H,F}^2d}{\alpha^2}\right\}.
\]
\end{proposition}

\begin{proof}
Let $X=m+C\xi$, $\bar H=\mathbb EH(X)$, $u=\mathbb E\nabla V(X)$, and $B=\bar HC-C^{-1}$. Stein's identity gives
\[
e:=\nabla V(X)-u-\bar HC\xi,\qquad \mathbb Ee=0,\qquad \mathbb E[e\xi^\top]=0.
\]
Consequently, the mean component has variance $\|B\|_F^2+\mathbb E\|e\|^2$. Gaussian Poincar\'e gives
\[
\mathbb E\|e\|^2\le D_C:=\mathbb E\|(H(X)-\bar H)C\|_F^2.
\]
The covariance-factor component $\operatorname{sym}(H(X)C)-C^{-1}$ also has variance at most $D_C$. Since $CB=B^\top C$ and $\nabla_CF=\operatorname{sym}(B)$, Lemma~\ref{lem:skew-final} yields
\[
\operatorname{Var}(G^{\mathrm{Price\text{-}STL}})\le2\|\nabla_CF\|_F^2+2D_C.
\]
For any constant symmetric $A$, centering minimizes the squared deviation, so
\[
D_C\le\mathbb E\|(H(X)-A)C\|_F^2
\le\sup_x\|H(x)-A\|_{\mathrm{op}}^2\|C\|_F^2.
\]
Taking the infimum gives $D_C\le\omega_H^2\|C\|_F^2$. Under \eqref{eq:dimension-frobenius}, an independent copy $Y$ of $X$ gives the alternative bound
\[
D_C=\tfrac12\mathbb E\|(H(X)-H(Y))C\|_F^2
\le\tfrac12L_{H,F}^2\|C\|_{\mathrm{op}}^2\mathbb E\|X-Y\|^2
=L_{H,F}^2\|C\|_{\mathrm{op}}^2\|C\|_F^2.
\]
Finally, $\|C\|_{\mathrm{op}}^2\le4/\alpha$ and $\|C\|_F^2\le4d/\alpha$ on $K_C$ give the stated constants.
\end{proof}

\paragraph{Fixed-batch consequence.}
Let $x_0\in K_C$, $\Delta_0=F(q_0)-F^\star+\frac\alpha2\|x_0-x_\star\|^2$, and $0<\varepsilon\le\Delta_0$. Keep Algorithm~\ref{alg:score-proj} and set
\[
s=\sqrt{15\kappa},\quad \eta=\frac1{15\beta},\quad
b=\left\lceil\max\left\{16,\frac{2\sigma_{\mathrm{str}}^2}{\alpha s\varepsilon}\right\}\right\rceil,\quad
T_\varepsilon=\left\lceil s\log\frac{2\Delta_0}{\varepsilon}\right\rceil.
\]
The batch variance is at most $2\|\nabla F\|^2/b+\sigma_{\mathrm{str}}^2/b$. Lemmas~\ref{lem:symmetric-geometry-final} and~\ref{lem:score-proj-fixed-batch} therefore give
\[
\mathbb E[F(q_{T_\varepsilon})-F^\star]
\le e^{-T_\varepsilon/s}\Delta_0+\frac{\sigma_{\mathrm{str}}^2}{\alpha bs}\le\varepsilon.
\]
Thus the curvature-oscillation bound, without Hessian Lipschitzness, yields
\begin{equation}
N_{\nabla V}=O\!\left(\left[\sqrt\kappa+\frac{\omega_H^2d}{\alpha^2\varepsilon}\right]\log\frac{2\Delta_0}{\varepsilon}\right),
\qquad N_{\mathrm{HVP}}=dN_{\nabla V}.
\label{eq:price-structured-counts}
\end{equation}
If \eqref{eq:dimension-frobenius} holds, $\omega_H^2/\alpha^2$ in this bound can be replaced by $\min\{\omega_H^2/\alpha^2,L_{H,F}^2/\alpha^3\}$. The sampling term is linear in $d$ for gradient queries, while HVP accounting retains the factor $d$.

\paragraph{Graph-probit example.}
For the targets in Appendix~\ref{exp:targets},
\[
H(f)=Q_\tau+\operatorname{diag}\bigl(\ell''(\widetilde y_i f_i)\bigr),
\qquad 0<\ell''<1.
\]
Choosing $A=Q_\tau+\frac12I$ gives $\omega_H\le\frac12$. Moreover,
\[
\|H(f)-H(g)\|_F^2
=\sum_i|\ell''(\widetilde y_i f_i)-\ell''(\widetilde y_i g_i)|^2
\le C_\Phi^2\|f-g\|^2,
\]
where \(C_\Phi:=\sup_t|\ell'''(t)|<\infty\), so $L_{H,F}\le C_\Phi$. Both constants are independent of $d$ and $\tau$.
Since $\alpha=1$, the batch choice above gives
\[
N_{\nabla V}=O\!\left(\left[\sqrt\kappa+\frac d\varepsilon\right]
\log\frac{2\Delta_0}{\varepsilon}\right),
\qquad N_{\mathrm{HVP}}=dN_{\nabla V}
\]
for Price--STL. Thus the gradient-query sampling term is linear in $d$ and independent of $\tau$.

\section{Experimental details}
\label{app:experiments}

\subsection{Gaussian condition-number experiment}
\label{exp:gaussian}

For $d=8$ and $\kappa\in\{10,30,100,300,1000,3000\}$, we use
\[
\mu=\frac{\mathbf1_d}{\sqrt d},\qquad
H_\kappa=\frac12\operatorname{diag}\!\left(\kappa^{j/(d-1)}\right)_{j=0}^{d-1},\qquad
\pi_\kappa=\mathcal N(\mu,H_\kappa^{-1}).
\]
Thus $\alpha=1/2$, $\beta=\kappa/2$, and $q_\star=\pi_\kappa$. We optimize both mean and covariance from $q_0=\mathcal N(0,I_d)$. 

At every update, we evaluate the exact gap
\[
F(q)-F^\star
=\frac12\left((m-\mu)^\top H_\kappa(m-\mu)
+\operatorname{tr}(H_\kappa\Sigma)-\log\det(H_\kappa\Sigma)-d\right).
\]

\paragraph{Stochastic STL scaling.}
Figure~\ref{fig:exp-gaussian} uses Algorithm~\ref{alg:score-proj} with $G_b^{\mathrm{STL}}$, $b=64$, $\eta=1/(15\beta)$, and $s=\sqrt{15\kappa}$. Its nonaccelerated counterpart uses the same estimator, step, and domain $K_C$, with the projected-gradient update specified in Appendix~\ref{exp:methods}. We use the five seeds in Appendix~\ref{exp:evaluation}, a budget of $120\kappa$ updates, and no restarts or early stopping. The scaling panel records the first update at which the five-seed mean gap is at most $10^{-6}$; every run reaches this threshold. STL remains stochastic away from the optimum. On these quadratic targets, Lemma~\ref{lem:score-growth-full} bounds its variance by $((d+4)/b)\|\nabla F\|^2$, with $(d+4)/b=3/16<1/2$, so Lemma~\ref{lem:score-proj-fixed-batch} applies without an additive variance term.

\paragraph{Paired benchmark.}
We use the paired estimator in Algorithm~\ref{alg:fd-prox}, with $\eta=2/\kappa$ and $\rho=(\sqrt\kappa-1)/(\sqrt\kappa+1)$, and its nonaccelerated variant with $\rho=0$. The constant Hessian makes each paired sample exact:
\[
G^{\mathrm{paired}}(m,L;\xi,\zeta)
=\bigl(H_\kappa(m-\mu),\operatorname{low}(H_\kappa L)\bigr).
\]
We therefore use $b=1$ for Figure~\ref{fig:exp-gaussian-paired}. The factor constraint $\|L\|_F\le\sqrt{d/\alpha}=\sqrt{2d}$ contains the initial and optimal factors and was inactive throughout. There are no restarts. We record the first update with gap at most $10^{-6}$, with a budget of $\lceil20\kappa\rceil$ updates; all runs reach this threshold.

\begin{figure}[hbtp]
\centering
\includegraphics[width=\linewidth]{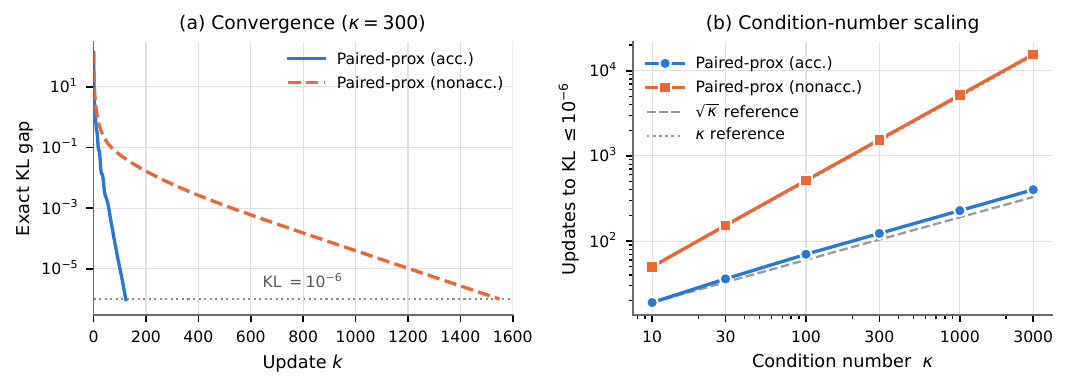}
\caption{\textbf{Paired benchmark on Gaussian targets.} Accelerated and nonaccelerated Paired-prox with a samplewise-exact estimator: (a) KL gap at $\kappa=300$; (b) updates to KL $\le10^{-6}$, with $\sqrt\kappa$ and $\kappa$ reference lines.}
\label{fig:exp-gaussian-paired}
\end{figure}

\paragraph{Additional methods at $\kappa=300$.}
Figure~\ref{fig:exp-gaussian-first} compares all seventeen methods in Table~\ref{tab:exp-methods} at $b=1,8,64$, with $6000$ updates and no early stopping. Steps and domains use $\alpha=1/2$ and $\beta=150$, with FD offset $\nu=10^{-2}/\sqrt\alpha$. Gaps are evaluated exactly at every update. The eight FD and paired variants are samplewise exact on this target and use one seed per batch; the other nine methods use the five seeds specified in Appendix~\ref{exp:evaluation}.

\begin{figure}[hbtp]
\centering
\includegraphics[width=\linewidth]{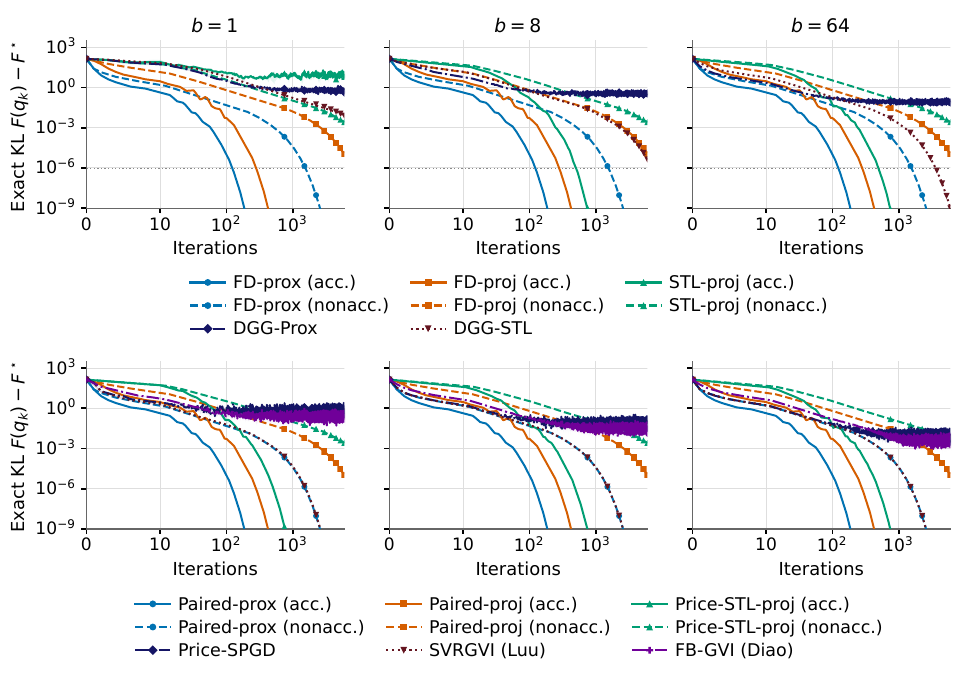}
\caption{\textbf{Gaussian comparisons at $\kappa=300$.} Rows show first- and second-order methods; columns show $b=1,8,64$. Curves are exact KL gaps, averaged over five seeds with $20$--$80$ percentile bands for stochastic methods. FD and paired methods use one seed. Solid and dashed lines distinguish accelerated and nonaccelerated variants. The dotted line marks $10^{-6}$; gaps below $10^{-9}$ leave the displayed range.}
\label{fig:exp-gaussian-first}
\label{fig:exp-gaussian-second}
\end{figure}

\subsection{OpenML datasets and graph-probit posteriors}
\label{exp:targets}

We use version~1 of the five OpenML datasets in Table~\ref{tab:exp-datasets}, retaining every observation and label. Predictors are centered and divided by their population standard deviation after removing zero-variance columns.

Numerical zeros in Diabetes are retained without imputation. Classes are encoded as $y_i\in\{0,1\}$, with $\widetilde y_i=2y_i-1$. Let $X_0\in\mathbb R^{n\times p}$ denote the standardized predictors.

\begin{table}[hbtp]
\centering\small
\begin{tabular}{lrrr}
\toprule
Dataset & OpenML ID & $d=n$ & $\kappa$\\
\midrule
Sonar & 40 & 208 & 69.47\\
Ionosphere & 59 & 351 & 75.64\\
WDBC & 1510 & 569 & 69.68\\
Diabetes & 37 & 768 & 70.20\\
Banknote & 1462 & 1372 & 69.43\\
\bottomrule
\end{tabular}
\caption{\textbf{Graph-probit dimensions and conditioning.} Each observation has one latent score. All targets use $\alpha=1$, so $\kappa=\beta$.}
\label{tab:exp-datasets}
\end{table}

\paragraph{Graph-probit classification.}
We construct an unweighted graph from $X_0$ using Euclidean ten-nearest neighbors, without self edges and with ties resolved by row index. An edge is included if either endpoint selects the other. For symmetric adjacency matrix $W$ and degree matrix $D_G$, define
\[
L_G=I_n-D_G^{-1/2}WD_G^{-1/2},\qquad Q_\tau=I_n+\tau L_G.
\]
The posterior has one latent score per observation and unit probit noise:
\[
V_{\rm probit}(f)=\frac12f^\top Q_\tau f+
\sum_{i=1}^n\ell(\widetilde y_i f_i),\qquad \ell(t)=-\log\Phi(t).
\]
This shifted graph-Laplacian prior follows the model family of \citet{miller2020graphprobit}; see also \citet{bertozzi2018uncertainty}. We fix $\tau=49$, giving $\kappa\le100$.

Writing $r=\phi/\Phi$, where $\phi$ is the standard normal density, gives
\begin{align}
\ell'(t)&=-r(t),& \ell''(t)&=r(t)(t+r(t)),\nonumber\\
\ell'''(t)&=r(t)\bigl[1-(t+r(t))(t+2r(t))\bigr].
\label{eq:exp-probit-derivatives}
\end{align}
For $Z\sim\mathcal N(0,1)$, the truncated-normal identities $\mathbb E[Z\mid Z\le t]=-r(t)<t$ and $\operatorname{Var}(Z\mid Z\le t)=1-\ell''(t)>0$ imply $0<\ell''<1$. Hence
\begin{align*}
\nabla V_{\rm probit}(f)&=Q_\tau f-\bigl(\widetilde y_i r(\widetilde y_i f_i)\bigr)_{i=1}^n,\\
\nabla^2V_{\rm probit}(f)&=Q_\tau+\operatorname{diag}\bigl(\ell''(\widetilde y_i f_i)\bigr),
\end{align*}
and $I_n\preceq\nabla^2V_{\rm probit}\preceq(2+\tau\lambda_{\max}(L_G))I_n$. We use $\alpha=1$ and $\beta=\kappa=2+\tau\lambda_{\max}(L_G)$.

The function $\ell'''$ is continuous and tends to zero at both tails. At $+\infty$ this follows directly from $r=\phi/\Phi$; at $-\infty$, the Gaussian tail expansion gives
\[
\frac{\Phi(-x)}{\phi(x)}=x^{-1}-x^{-3}+3x^{-5}-15x^{-7}+O(x^{-9}),\quad
r(-x)=x+x^{-1}-2x^{-3}+10x^{-5}+O(x^{-7}),
\]
so $\ell'''(-x)=-2x^{-3}+O(x^{-5})$. Thus $C_\Phi:=\sup_t|\ell'''(t)|<\infty$, and
\[
\|\nabla^2V_{\rm probit}(f)-\nabla^2V_{\rm probit}(g)\|_{\rm op}
=\max_i|\ell''(\widetilde y_i f_i)-\ell''(\widetilde y_i g_i)|
\le C_\Phi\|f-g\|.
\]
Varying a single coordinate shows that $L_H=C_\Phi$, independently of $\tau$, the graph, and $n$.

\subsection{Methods, steps, and oracle accounting}
\label{exp:methods}

Table~\ref{tab:exp-methods} lists six accelerated/nonaccelerated pairs and five baselines. DGG-Prox and DGG-STL are from \citet{domke2023provable}; the other baselines are Price-SPGD~\citep{kim2026stochastic}, SVRGVI~\citep{luu2025stochastic}, and FB-GVI~\citep{diao2023forward}. All optimize the same full-covariance Gaussian objective. Write $s_p=\sqrt{\max\{\kappa,4\}}$, $s_e=\sqrt{\max\{5\kappa,4\}}$, and $r_b=1+(d+2)/b$.

\begin{table}[hbtp]
\centering\small
\setlength{\tabcolsep}{3pt}
\renewcommand{\arraystretch}{1.14}
\begin{tabular}{@{}p{0.29\linewidth}p{0.065\linewidth}p{0.065\linewidth}p{0.31\linewidth}p{0.18\linewidth}@{}}
\hline
Method & Oracle & Coord. & Entropy treatment / domain & Step $\eta$ \\
\hline
FD-prox (Alg.~\ref{alg:fd-prox}) & 1st & $L$ & Entropy and factor-ball prox & $1/(\alpha s_p^2)$ \\
FD-proj (Alg.~\ref{alg:fd-proj}) & 1st & $L$ & Entropy extension, $K_L$ & $1/(\alpha s_e^2)$ \\
STL-proj (Alg.~\ref{alg:score-proj}) & 1st & $C$ & Full gradient, $K_C$ & $1/(15\beta)$ \\
Paired-prox (Alg.~\ref{alg:fd-prox}) & 2nd & $L$ & Entropy and factor-ball prox & $1/(\alpha s_p^2)$ \\
Paired-proj (Alg.~\ref{alg:fd-proj}) & 2nd & $L$ & Entropy extension, $K_L$ & $1/(\alpha s_e^2)$ \\
Price-STL-proj (Alg.~\ref{alg:score-proj}) & 2nd & $C$ & Full gradient, $K_C$ & $1/(15\beta)$ \\
\hline
DGG-Prox & 1st & $L$ & Entropy prox, $L_{ii}>0$ & $1/(\beta r_b)$ \\
DGG-STL & 1st & $C$ & Full gradient, $C\succeq\beta^{-1/2}I$ & $1/(2\beta r_b)$ \\
Price-SPGD & 2nd & $L$ & Entropy prox, $L_{ii}>0$ & $1/\beta$ \\
SVRGVI ($c=1$) & 2nd & $\Sigma$ & BW forward--backward step & $1/\beta$ \\
FB-GVI & 2nd & $\Sigma$ & BW forward--backward step & $1/(2\beta)$ \\
\hline
\end{tabular}
\caption{\textbf{Methods and fixed steps.} Each of the first six rows includes accelerated and nonaccelerated variants. $L$ is a positive-diagonal Cholesky factor and $C$ is symmetric positive definite. Oracle order refers to target derivatives.}
\label{tab:exp-methods}
\end{table}

The proximal and Cholesky projected variants use $\rho=(1-s^{-1})/(1+s^{-1})$, with $s=s_p$ and $s=s_e$, respectively; STL and Price-STL use $s=\sqrt{15\kappa}$. Nonaccelerated variants use the same estimator, step, and feasible set with no momentum. 
There are no restarts or step-size searches.
Price-STL uses a score-corrected mean estimator, whereas Price-SPGD uses the raw mean gradient. 
SVRGVI uses $c=1$. Covariance-based methods average the $b$ derivative estimates before one forward--backward update.

Per update, FD uses $d+1+b$ gradients, including $d+1$ calibration queries. Paired uses $b$ gradients and $db$ HVPs at independent gradient and Hessian locations. Price-STL, Price-SPGD, SVRGVI, and FB-GVI have the same counts but share derivative locations. STL, DGG-Prox, and DGG-STL use $b$ gradients and no HVPs. 

\subsection{OpenML execution and objective evaluation}
\label{exp:evaluation}

All methods start from $q_0=\mathcal N(0,I_d)$ and run for $600$ updates, with checkpoints $0,10,\ldots,600$ and five seeds $20260831+997j$, $j=0,\ldots,4$. We use $b=1,8,64,256$. Batches consist of Gaussian Monte Carlo samples, each using the full likelihood. For STL and Price-STL, these practical choices do not enforce the sufficient batch bounds in Theorem~\ref{thm:score-recovery-full}.

\paragraph{Objective and numerical optimum.}
Each likelihood term depends on a one-dimensional Gaussian marginal. We evaluate the prior and entropy analytically and the likelihood by Gauss--Hermite quadrature. Quadrature is used only for evaluation and numerical references. The plotted gaps subtract a common reference $F_{\mathrm{ref}}\approx F^\star$, obtained by solving the stationary equations in $(m,h)$ with $\Sigma=(Q_\tau+\operatorname{diag}h)^{-1}$.

References require a Cholesky-coordinate gradient norm at most $10^{-5}$, a covariance stationary-equation residual at most $10^{-8}$ in relative norm, and agreement between quadrature orders $192$ and $384$ within $10^{-8}$ for the objective and $10^{-5}$ for the gradient. Checkpoints use successive order pairs $(192,384)$, $(384,768)$, and $(768,1536)$ until objective agreement is within $10^{-8}$; each gap subtracts the reference at the same order. These are numerical consistency checks, not rigorous integration-error bounds.

\paragraph{Curves and hitting times.}
Graph-probit figures show five-seed arithmetic means with empirical $20$--$80$ percentile bands; solid and dashed lines distinguish accelerated and nonaccelerated variants. All five evaluations are required at each plotted checkpoint. For mean gap $\overline g_k$ and common initial gap $g_0$, the hitting time at relative threshold $\delta\in(0,1)$ is the first recorded checkpoint with $\overline g_k\le\delta g_0$. It is measured on the ten-update grid and right-censored if the threshold is not reached by update $600$; a finite crossing wins over a missing crossing.

\subsection{Graph-probit comparisons and batch sensitivity}
\label{exp:graph-probit-results}

At $b=64$, accelerated FD-proj, Paired-proj, and Price-STL-proj reach the $1\%$ threshold earlier on every dataset (Table~\ref{tab:exp-probit-hitting}). The proximal pairs tie at the first checkpoint, $k=10$, at this threshold. 

Table~\ref{tab:exp-probit-hitting} reports all six pairs at four thresholds.

\begin{table}[hbtp]
\centering
\small
\begin{tabular}{lcccc}
\toprule
Matched pair & $10^{-2}$ & $10^{-3}$ & $10^{-4}$ & $10^{-5}$ \\
\midrule
FD-prox & 0/0/5/0 & 3/0/2/0 & 4/0/1/0 & 5/0/0/0 \\
FD-proj & 5/0/0/0 & 4/0/1/0 & 5/0/0/0 & 5/0/0/0 \\
STL-proj & 2/3/0/0 & 2/2/0/1 & 2/0/0/3 & 1/0/0/4 \\
Paired-prox & 0/0/5/0 & 3/0/2/0 & 4/0/1/0 & 5/0/0/0 \\
Paired-proj & 5/0/0/0 & 4/0/1/0 & 5/0/0/0 & 5/0/0/0 \\
Price-STL-proj & 5/0/0/0 & 5/0/0/0 & 5/0/0/0 & 5/0/0/0 \\
\bottomrule
\end{tabular}
\caption{\textbf{Graph-probit hitting-time comparisons at $b=64$.} Thresholds are fractions of the initial mean gap. Entries count datasets with an earlier accelerated crossing / earlier nonaccelerated crossing / equal recorded crossing / no crossing by either variant. Matched-pair rankings are identical under derivative-query budgets.}
\label{tab:exp-probit-hitting}
\end{table}

\paragraph{Query costs.}
Figure~\ref{fig:exp-probit-query} replots the $b=64$ curves against cumulative gradient queries. FD includes its $d+1$ calibration gradients per update, shifting its curves relative to STL and the gradient-only baselines. In the second-order comparisons, all methods use $b$ gradients and $db$ HVPs per update. 
Thus the within-pair momentum comparisons hold under both iteration and derivative-query budgets, while costs across estimators differ.

\begin{figure}[hbtp]
\centering
\includegraphics[width=\linewidth]{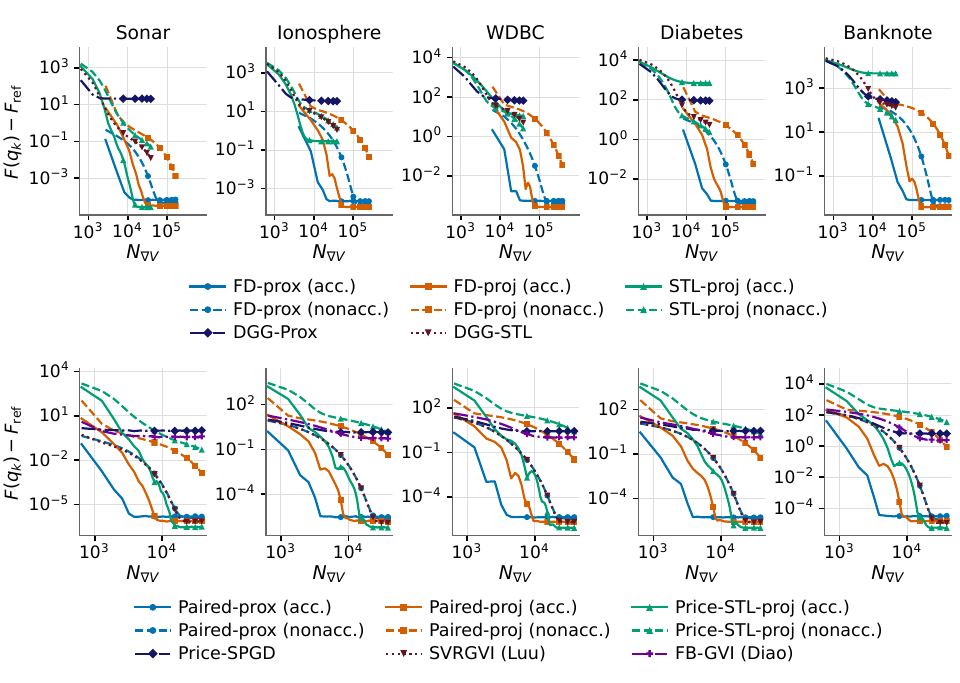}
\caption{\textbf{Graph-probit convergence by gradient-query cost at $b=64$.} Rows show first- and second-order methods on all five datasets. Means and bands follow Figure~\ref{fig:exp-openml}; the logarithmic horizontal axis uses cumulative gradient counts at $k=10,20,\ldots,600$, including calibration and setup queries. First-order methods use no HVPs. For second-order methods, $N_{\mathrm{HVP}}=d[N_{\nabla V}-N_{\nabla V}(0)]$, where the setup count $N_{\nabla V}(0)$ is $1$ for Paired-prox and $0$ otherwise.}
\label{fig:exp-probit-query}
\end{figure}

Increasing the batch size substantially improves accelerated STL-proj (Table~\ref{tab:exp-probit-stl-batch}). At $b=256$, it reaches the $1\%$ threshold in $30$ updates on the first four datasets and $40$ on Banknote, versus $50$ and $60$ for the nonaccelerated method, and has a smaller final gap on every dataset. Figure~\ref{fig:exp-probit-batch256} shows all seventeen methods at this batch size. Accelerated Price-STL-proj already achieves final mean gaps below $6\times10^{-6}$ on all datasets at $b=64$.

\begin{table}[H]
\centering\small
\setlength{\tabcolsep}{4pt}
\begin{tabular}{lrrrrrr}
\toprule
& \multicolumn{2}{c}{$b=8$} & \multicolumn{2}{c}{$b=64$} & \multicolumn{2}{c}{$b=256$}\\
\cmidrule(lr){2-3}\cmidrule(lr){4-5}\cmidrule(lr){6-7}
Dataset & Acc. & Nonacc. & Acc. & Nonacc. & Acc. & Nonacc.\\
\midrule
Sonar & $1.01\times10^{3}$ & $5.66\times10^{-2}$ & $2.70\times10^{-5}$ & $5.24\times10^{-2}$ & $2.84\times10^{-6}$ & $5.24\times10^{-2}$\\
Ionosphere & $3.00\times10^{3}$ & $9.27\times10^{1}$ & $2.84\times10^{-1}$ & $1.46\times10^{0}$ & $1.09\times10^{-5}$ & $1.46\times10^{0}$\\
WDBC & $9.61\times10^{3}$ & $1.28\times10^{3}$ & $1.09\times10^{1}$ & $2.56\times10^{0}$ & $3.34\times10^{-5}$ & $2.55\times10^{0}$\\
Diabetes & $1.90\times10^{4}$ & $9.95\times10^{3}$ & $7.03\times10^{2}$ & $2.33\times10^{0}$ & $7.35\times10^{-5}$ & $2.32\times10^{0}$\\
Banknote & $4.93\times10^{4}$ & $3.71\times10^{4}$ & $4.78\times10^{3}$ & $3.61\times10^{1}$ & $1.09\times10^{0}$ & $3.57\times10^{1}$\\
\bottomrule
\end{tabular}
\caption{\textbf{STL batch sensitivity.} Five-seed mean gaps $F(q_{600})-F_{\mathrm{ref}}$, with identical stepsizes and initialization across batches.}
\label{tab:exp-probit-stl-batch}
\end{table}

\begin{figure}[hbtp]
\centering
\includegraphics[width=\linewidth]{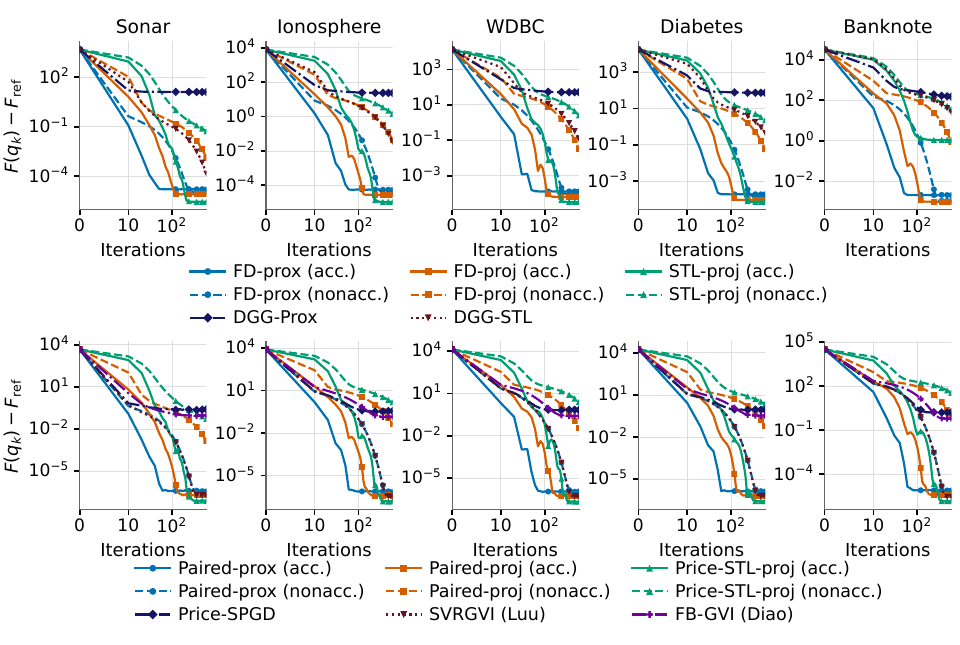}
\caption{\textbf{Graph-probit comparisons at $b=256$.} Rows show first- and second-order methods on all five datasets. Plotting conventions follow Appendix~\ref{exp:evaluation}; other settings match Figure~\ref{fig:exp-openml}.}
\label{fig:exp-probit-batch256}
\end{figure}

\clearpage
\subsection{Logistic-regression comparisons}
\label{exp:openml-results}

To complement the graph-probit results, we also report experiments with Bayesian logistic-regression posteriors.

\paragraph{Setup and evaluation.}
We use the same five datasets and preprocessing as Appendix~\ref{exp:targets}. With $X=[X_0,\mathbf 1_n]$ and $d=p+1$, a standard Gaussian prior on all coefficients, including the intercept, gives
\[
V_{\rm logit}(w)=\frac12\|w\|^2+
\sum_{i=1}^n\bigl[\log(1+e^{x_i^\top w})-y_i x_i^\top w\bigr],
\qquad \alpha=1,\quad \beta=\kappa=1+\tfrac14\|X\|_{\rm op}^2.
\]
The dimensions are $61,34,31,9,5$ and the condition-number bounds are $635.81,774.27,1890.31,403.12,748.71$ for Sonar, Ionosphere, WDBC, Diabetes, and Banknote, respectively. 
We compare all seventeen methods at $b=1,8,64$, using the steps in Appendix~\ref{exp:methods} and the initialization, seeds, update budget, and plotting conventions in Appendix~\ref{exp:evaluation}.
Logistic expectations use Gauss--Hermite quadrature when the marginal standard deviation is below $0.75$. Otherwise, we integrate the positive part analytically and the remaining softplus correction by Gauss--Legendre quadrature on $[0,40]$. 
We compute $F_{\rm ref}$ by L-BFGS-B over the mean, log-diagonal, and strict lower-triangular entries of $L$, requiring a Cholesky-coordinate gradient norm at most $10^{-5}$ and the objective/gradient quadrature-consistency tolerances in Appendix~\ref{exp:evaluation}.

\paragraph{Results.}

Across five datasets and $b=1,8,64$, acceleration reaches the $1\%$ threshold earlier in $65$ of $90$ matched comparisons (Table~\ref{tab:exp-logistic-hitting}). 
Figures~\ref{fig:exp-logistic-first-all} and~\ref{fig:exp-logistic-second-all} show every method, including the five baselines.

\begin{table}[hbtp]
\centering\small
\begin{tabular}{lrrrr}
\toprule
Matched pair & Acc. earlier & Nonacc. earlier & Same checkpoint & Neither reached\\
\midrule
FD-prox & 7 & 7 & 1 & 0\\
FD-proj & 10 & 1 & 2 & 2\\
STL-proj & 12 & 1 & 0 & 2\\
Paired-prox & 9 & 3 & 3 & 0\\
Paired-proj & 12 & 0 & 3 & 0\\
Price-STL-proj & 15 & 0 & 0 & 0\\
\bottomrule
\end{tabular}
\caption{\textbf{Logistic-regression hitting-time comparisons.} Each row covers five datasets and three batches, using a hitting threshold of $1\%$ of the initial mean gap.}
\label{tab:exp-logistic-hitting}
\end{table}

\clearpage
\begin{figure}[hbtp]
\centering
\includegraphics[width=\linewidth]{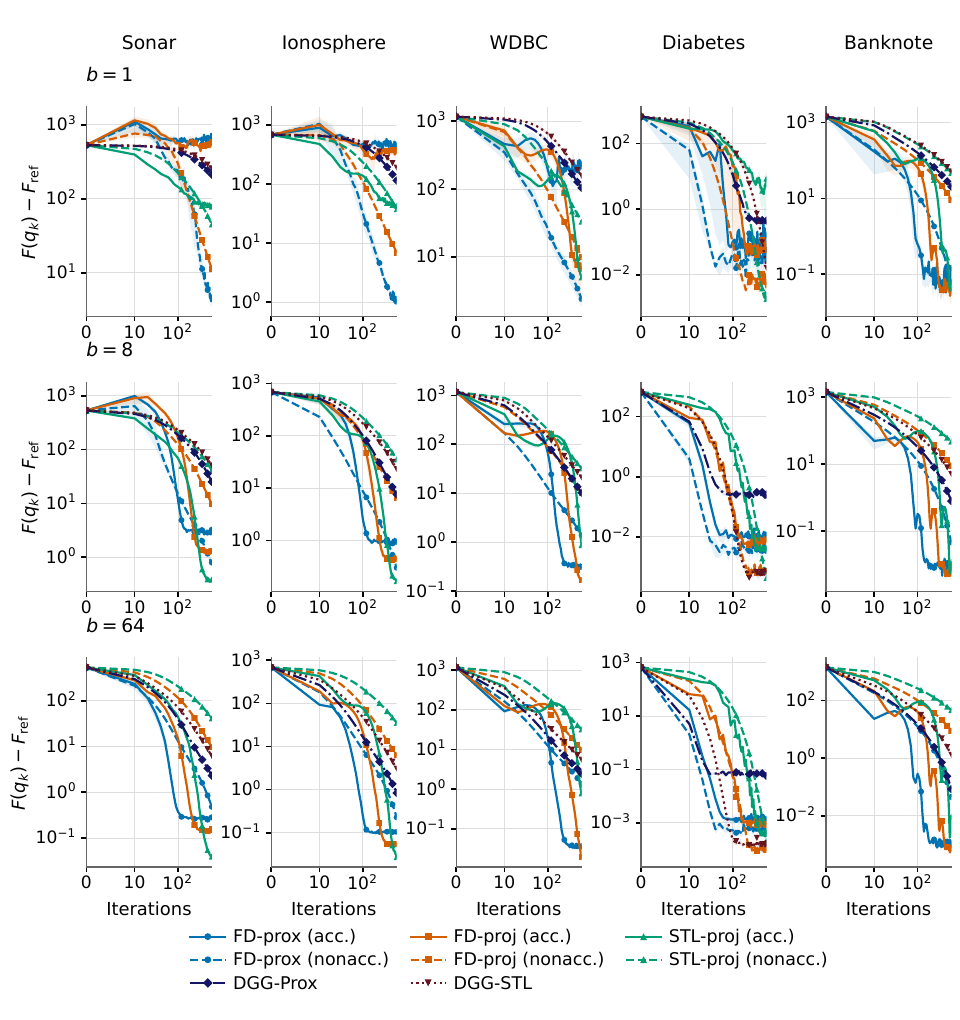}
\caption{\textbf{First-order methods on logistic-regression posteriors.} Rows show $b=1,8,64$; columns show the five datasets. Plotting conventions follow Appendix~\ref{exp:evaluation}.}
\label{fig:exp-logistic-first-all}
\end{figure}
\clearpage
\begin{figure}[H]
\centering
\includegraphics[width=\linewidth]{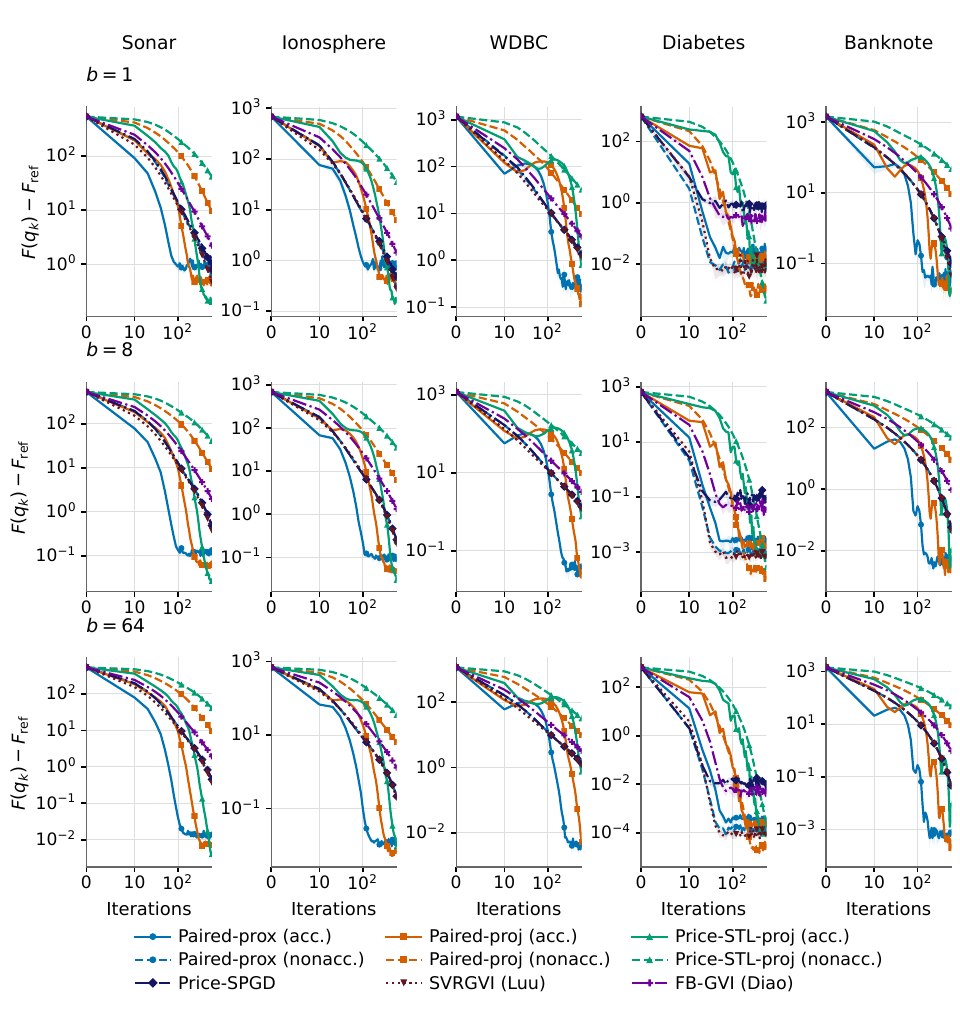}
\caption{\textbf{Second-order methods on logistic-regression posteriors.} Rows show $b=1,8,64$; columns show the five datasets. Plotting conventions follow Appendix~\ref{exp:evaluation}.}
\label{fig:exp-logistic-second-all}
\end{figure}

\bibliography{ref}
\bibliographystyle{unsrtnat}

\end{document}